\documentclass[magisterska,en]{pracamgr}
\usepackage[utf8]{inputenc}
\usepackage[
    natbib=true,
    style=numeric,
    sorting=none,
    maxbibnames=99
]{biblatex}
\usepackage{graphicx}
\usepackage{amsmath}
\usepackage{amssymb}
\usepackage{amsthm}
\usepackage{xspace}
\usepackage{subcaption}
\usepackage{algorithm}
\usepackage[noend]{algpseudocode}
\usepackage{graphbox}
\usepackage{hyperref}
\usepackage{multirow}
\usepackage{pythonhighlight}

\bibliography{bibliography}

\newtheorem{theorem}{Theorem}

\newcommand*{\eg}{\textit{e}.\textit{g}.\@\xspace}
\newcommand*{\ie}{\textit{i}.\textit{e}.\@\xspace}
\newcommand*{\Eg}{\textit{E}.\textit{g}.\@\xspace}

\newcommand*{\etal}{\textit{el~al}.\@\xspace}
\makeatletter
\newcommand*{\etc}{%
    \@ifnextchar{.}%
        {etc}%
        {etc.\@\xspace}%
}
\makeatother

\newcommand{\funarg}[1]{\mathopen{}\left(#1\right)\mathclose{}}
\newcommand{\parens}[1]{\left(#1\right)}
\newcommand\newtag[2]{#1\def\@currentlabel{#1}\label{#2}}

\newcommand{\SWp}[0]{\mathcal{SW}^+}
\newcommand{\NSWp}[0]{\mathcal{NSW}^+}
\newcommand{\FSWp}[1]{\mathcal{FSW}^+_{#1}}
\newcommand{\FNSWp}[1]{\mathcal{FNSW}^+_{#1}}
\newcommand{\FWp}[1]{\mathcal{FW}^+_{#1}}
\newcommand{\FNWp}[1]{\mathcal{FNW}^+_{#1}}
\newcommand{\W}[0]{\mathcal{W}}
\newcommand{\Wp}[0]{\mathcal{W}^+}
\newcommand{\NWp}[0]{\mathcal{NW}^+}
\newcommand{\generator}[1]{\mathcal{G}_{#1}}

\newcommand{\algparbox}[2]{\parbox[t]{\dimexpr#1\linewidth-\algorithmicindent}{#2}}

\DeclareMathOperator{\mix}{mix}
\DeclareMathOperator{\Var}{Var}
\DeclareMathOperator{\EX}{\mathbb{E}}

\algrenewcommand\algorithmicrequire{\textbf{Data:}}

\autor{Maciej Sypetkowski}{386094}

\title{Spatial Latent Representations \\
in Generative Adversarial Networks \\
for Image Generation}
\titlepl{Przestrzenne Ukryte Reprezentacje w Generatywnych Sieciach Przeciwstawnych \\ do Generowania Obrazów}

\kierunek{Computer Science}

\opiekun{
Piotr Biliński, PhD \\
Faculty of Mathematics, Informatics and Mechanics \\
University of Warsaw
}

\date{December 2021}

\dziedzina{ 
11.4 Artificial Intelligence
}

\klasyfikacja{
I.2.10 Vision and Scene Understanding
}
\keywords{generative adversarial networks, StyleGAN, latent representations, GAN inversion, image generation}

\begin{document}
\maketitle

\begin{abstract}
Generative Adversarial Networks (GANs) are currently state-of-the-art methods in image
generation tasks. They generate new images by transforming a latent space into an image data
distribution. In the vast majority of GAN architectures, the latent space is defined as a set of
vectors of given dimensionality. Such representations are not easily interpretable and
do not capture spatial information of image content directly.
In this work,
we define a family of spatial latent spaces for StyleGAN2,
capable of capturing more details and representing images that are out-of-sample in terms of the number and arrangement of object parts, such as an image of multiple faces or a face with more than two eyes.
We propose a method for encoding images into our spaces, together with an attribute model
capable of performing attribute editing in these spaces. We show that our spaces are effective for image manipulation purposes
and encode semantic information well.
Our approach can be used on pre-trained generator models,
and attribute edition can be done using pre-generated direction vectors making the barrier to entry for experimentation and use extremely low.

We propose a regularization method for optimizing latent representations,
which equalizes distributions of parts of latent spaces, making representations
much closer to generated ones.
We use it for encoding images into spatial spaces to obtain significant improvement
in quality while keeping semantics and ability to use our attribute model for edition purposes.
In total, using our methods gives encoding quality boost even as high as 30\% in terms of LPIPS score comparing to standard methods, while keeping semantics.

Additionally, we propose a StyleGAN2 training procedure on our spatial latent spaces,
together with a custom spatial latent representation distribution to make spatially closer elements in the representation more dependent on each other than farther elements. Such approach improves the FID score by 29\% on SpaceNet, and is able to generate consistent images of arbitrary sizes on spatially homogeneous datasets, like satellite imagery.
\end{abstract}

\tableofcontents

\chapter{Introduction}

State-of-the-art GAN models are able to generate photorealistic images, which often are extremely difficult to distinguish from real images for some datasets \cite{Karras2019stylegan2, DBLP:journals/corr/abs-1809-11096}.
Even though GANs can produce very diverse images that seemingly capture input distribution well, GAN inversion (\ie projecting real images into the latent space)
is still a challenging problem. There are two major common issues with obtaining good latent representations:
\begin{itemize}
    \item the image obtained from the latent representation differs from the original image. It may happen when a representation has too few parameters or a suboptimal representation was found, \eg did not converge, fell into a bad local minimum, poor performing loss function was used. 
    \item the representation is not meaningful. It may happen when a representation has too many parameters or appropriate regularization methods are not used.
        In such case, the obtained image from the representation is very visually similar to the original one,
        however it's not possible to use edition methods to obtain a targeted effect,
        like changing the age of a person on the image or obtaining a meaningful interpolation between images.
\end{itemize}
One approach to solving these problems is to define an appropriate latent space that has enough, but not too many degrees of freedom, \ie a latent space is capable of representing images of interest and is not prone to overfitting.
Besides changing parameters of a latent space to achieve that, it's possible to add regularization techniques during GAN inversion \cite{Karras2019stylegan2, DBLP:journals/corr/ZhuKSE16, DBLP:journals/corr/abs-1911-11544}.

Most works related to GAN inversion and image manipulation, use latent spaces defined as a set of one dimensional vectors of numbers (\eg $\mathcal{Z}$, $\mathcal{W}$, $\Wp$ spaces in StyleGAN) \cite{DBLP:journals/corr/ZhuKSE16, DBLP:journals/corr/abs-1802-05701, DBLP:journals/corr/abs-1904-03189}.
Such definitions don't capture spatial information directly, therefore:
\begin{itemize}
    \item translated or differently unaligned images (\eg zoomed or rotated) are not representable in such spaces.
    In particular, the process of encoding images into latent space is not invariant/equivariant to translation and projected representations lose reconstruction quality greatly
    \item there's no mechanism for generalization to out-of-sample images in terms of the number and arrangement of object parts.
    For example in the context of face generation: asymmetric faces, multiple people on one image, diprosopus, \etc
\end{itemize}

In this work, we address these issues.
In Chapter~\ref{chap:def}, we define fully spatial latent spaces for StyleGAN2 and describe their properties.
In Chapter~\ref{chap:mix}, we look at mixing capabilities of latent representations in these spaces.
In Chapter~\ref{chap:edit}, we propose a technique, which is capable of performing attribute edition on our representations.
In Chapter~\ref{chap:proj}, we define a method for inverting images into our spatial latent spaces and show that encoded representations are meaningful.
Our approaches don't require retraining any GAN model, and attribute editing can be done using pre-generated attribute direction vectors (\eg using latent directions from \cite{DBLP:journals/corr/abs-1907-10786, DBLP:journals/corr/abs-2002-03754, DBLP:journals/corr/abs-2001-10238}),
which drastically reduces computing resources needed.

In Chapter~\ref{chap:equaliz}, we further improve our methods by analyzing distribution of projected and generated latent representations and proposing a regularization method for projection into our latent spaces.

Additionally, in Chapter~\ref{chap:train} we argue that standard non-spatial one-dimensional input latent representations for GANs (\eg the $\mathcal{Z}$ space for StyleGAN) may be suboptimal for training on certain types of image data across various domains, \eg cellular images, satellite images or cosmological images. We suspect that spatial representations may be more suitable.
\begin{itemize}
    \item Such data, is invariant/equivariant to translation and rotation,
        \ie the image can be freely translated and rotated depending on the position and rotation of the scope.
    \item More distant objects on the image are more independent from each other and some common style is kept.
        For example, for cell images, one cell has usually one nucleus, so close objects are dependent, however distant cells don't affect each other, but some common style is shared, like cell type or disease, depending on the dataset.
        Similarly for satellite imagery (locally: terrain type; non-locally: architectural style) and cosmological images (locally: stronger forces between closer objects, gravitational lensing; non-locally: zoom).
\end{itemize}

We propose a spatial latent space for training and a training procedure that helps with these problems, and obtain much better scores than the baseline.


\section{Contributions}
\begin{itemize}
    
    \item We define new latent spaces for StyleGAN2, which capture all spatial information and don't require retraining the model
    \item We show that our latent spaces can be used effectively for image editing purposes, including spatial mixing and attribute editing
    \item We propose a method for projecting images directly into our latent space and show that this method is equivariant to translation and gives much greater flexibility, \eg supporting multiple faces on one image
    \item We identify a potential general issue about projected and generated latent representation discrepancies and propose a regularization method for tackling that.
    \item We propose a StyleGAN2 training procedure in a spatial latent space,
    which samples latent representations from a custom distribution to make spatially closer elements in the representation more dependent on each other than farther elements. Such approach improves the FID score by 29\% on SpaceNet, and is able to generate consistent images of arbitrary sizes.
\end{itemize}

\chapter{Related Works}

\section{Generative Adversarial Networks}
Generative adversarial networks (GANs) are powerful generative models that were first introduced in \cite{DBLP:journals/corr/GoodfellowPMXWOCB14}.
They described a framework for training a model capable of generating artificial samples from a given data distribution.
GANs use two neural network models -- a generator and a discriminator, which are trained simultaneously using adversarial learning,
\ie the objective of a discriminator is to distinguish real samples coming from the input data distribution, from the ones generated by a generator.
The objective of a generator is to trick a discriminator into incorrect classification of generated samples.
It can be described as a minimax game between generator $G$ and discriminator $D$:
\begin{equation}
\min_G \max_D V\funarg{D, G},
\end{equation}
\begin{equation}
V\funarg{D, G} = \EX_{x \sim p_{data}}[\log D\funarg{x}]+\EX_{z \sim p_z}[\log\funarg{1 - D\funarg{G\funarg{z}}}].
\end{equation}


Since that time a lot of improvements and extensions were published and GANs become very popular in various computer vision problems, \eg
image-to-image translation \cite{DBLP:journals/corr/ZhuPIE17, DBLP:journals/corr/abs-1711-09020,DBLP:journals/corr/abs-1912-01865, DBLP:journals/corr/abs-1711-11585, DBLP:journals/corr/abs-1903-07291},
superresolution \cite{DBLP:journals/corr/LedigTHCATTWS16, DBLP:journals/corr/abs-1809-00219},
inpainting \cite{DBLP:journals/corr/PathakKDDE16},
video generation \cite{DBLP:journals/corr/VondrickPT16, DBLP:journals/corr/abs-1807-11152}.
A lot of works focus on improving GANs in terms of
loss functions, training stability, regularizations, using higher resolutions, \etc \cite{DBLP:journals/corr/RadfordMC15, DBLP:conf/iccv/MaoLXLWS17, DBLP:journals/corr/Wu0ZH17, DBLP:journals/corr/ArjovskyCB17, DBLP:journals/corr/GulrajaniAADC17, DBLP:journals/corr/ZhangXLZHWM16, DBLP:journals/corr/abs-1710-10916, DBLP:journals/corr/abs-1710-10196, DBLP:journals/corr/abs-1802-05957}.

StyleGAN family models \cite{DBLP:journals/corr/abs-1812-04948, Karras2019stylegan2, DBLP:journals/corr/abs-2006-06676}
are considered to be among models capable of generating highest-fidelity images, especially
face images.
StyleGAN \cite{DBLP:journals/corr/abs-1812-04948} generator consist of mapping network and synthesis network.
Mapping network non-lineary transforms a latent code $z \in \mathcal{Z}$, usually $z \sim \mathcal{N}\funarg{0, I}$
into $w \in \mathcal{W}$, which is commonly called style vector.
Synthesis network consists of multiple synthesis layers.
Input to a first layer is a learned $512 \times 4 \times 4$ tensor.
Every two layers, the activation tensor is up-scaled.
In every layer, a style vector is linearly transformed and used as coefficients for adaptive instance normalization (AdaIN) \cite{DBLP:journals/corr/DumoulinSK16, DBLP:journals/corr/HuangB17}.
For stochastic variation, activation tensors are perturbed using random noise with learnable scale.

StyleGAN2 \cite{Karras2019stylegan2} improves the generator architecture by exchanging AdaIN with convolution weight demodulation.
The second notable architecture change is building output RGB map from the first layer using skip connections,
\ie the RGB map at a given layer is a sum of the RGB map from the previous layer and predicted addend from the feature map.
It can be thought as boosting.

StyleGAN2-ADA \cite{DBLP:journals/corr/abs-2006-06676} takes the architecture and training procedure from StyleGAN2 and proposes an adaptive discriminator augmentation mechanism, which allows to much better generalization and performance on datasets with lower number of samples. The mechanism augments the input for the discriminator using an adaptive augmentation strength determined by a custom heuristic measure of overfitting.


\section{Image to Latent Representation Projection}

There are two commonly used methods for mapping an image into a latent representation.
\subsection{Latent Representation Encoder}
This method consists of training a latent representation encoder that takes an image and encodes it directly into a corresponding latent variable.
These method are very fast during inference -- they require only one forward pass through the encoder model,
however they usually tend not to generalize outside the training dataset well.
One of the most known such models are variational auto-encoders \cite{DBLP:journals/corr/KingmaW13} (VAEs), which are consist of two parts -- encoder and decoder.
Encoder takes a sample and maps it into a lower dimensional latent variable, whereas decoder takes the latent variable and maps it back to the input image.
VAE-GANs \cite{DBLP:journals/corr/LarsenSW15} combines VAE and GAN architectures into one, allowing GAN to benefit from a built-in module for projection into the latent space.
Similar approach was taken in \cite{DBLP:journals/corr/abs-2004-00049} where an additional loss for latent variable similarity of reconstructed image is added.
Such approach allows to ensure that latent representations produced by the encoder lie in the native latent space of the generator.
Authors show that their representations can be used successfully for interpolation between images and face attribute manipulation.
In \cite{DBLP:journals/corr/abs-2104-14754}, the image is encoded into a latent space of a custom StyleGAN-based architecture, which is transformed into style maps for the synthesis network.

\subsection{Latent Representation Optimization}

The other popular method consists of taking an initial latent representation and using optimization methods (\eg backpropagation) to obtain the best latent representation
according to some loss function between the generated and the input image \cite{DBLP:journals/corr/abs-1802-05701, DBLP:journals/corr/ZhuKSE16}.
Zhu \etal \cite{DBLP:journals/corr/ZhuKSE16} additionally combines this method with the image encoder for a given GAN model,
and use it for initialization of the searched latent representation and an additional regularization to ensure the searched representation to be close to the initial one. 
To improve the quality of projected images, some works \cite{DBLP:journals/corr/abs-1904-03189, DBLP:journals/corr/DosovitskiyB16} add an additional perceptual loss.
Karras \etal \cite{Karras2019stylegan2} and Abdal \etal \cite{DBLP:journals/corr/abs-1904-03189} extend the latent representation from a single vector into multiple (style) vectors coming from the StyleGAN's mapping network ($\Wp$ space).
Karras \etal \cite{Karras2019stylegan2} additionally shows a way to optimize stochastic noise inputs by regularizing them to not carry a coherent signal.
Image2StyleGAN++ \cite{DBLP:journals/corr/abs-1911-11544} proposes noise space optimization to restore the high frequency features in an image, however obtained representations are unsuitable for attribute edition purposes.
Huh \etal \cite{DBLP:journals/corr/abs-2005-01703} shows a way to counteract problems of object-center bias, and introduces an additional step
for finding a transformation to center and crop the object of interest.

\section{GAN-based image editing}

Generally there are two approaches for editing existing images using generative models.
One is to use the image as an input to the network and transform it to achieve a targeted effect
\cite{DBLP:journals/corr/ZhuPIE17, DBLP:journals/corr/abs-1711-09020, DBLP:journals/corr/abs-1912-01865, DBLP:journals/corr/abs-1911-12861, DBLP:journals/corr/abs-2104-14754}.
Such approaches may require more supervision, like pairs of input-output images which differs in only some attributes.

Another, is first projecting the image into the latent variable,
and then editing the latent variable and generating the image.
Shen \etal \cite{DBLP:journals/corr/abs-1907-10786}
shows a way to find attribute direction vectors in the latent space to change certain face attributes of an image (\eg smile, age, gender, eyeglasses)
by finding attribute separating hyperplanes.
Goetschalckx \etal \cite{DBLP:journals/corr/abs-1906-10112} focuses on learning latent directions for visual image properties that are hard to define in words like aesthetics, memorability, emotional valence, \etc.
Plumerault \etal \cite{DBLP:journals/corr/abs-2001-10238} optimizes latent vectors along a latent direction by maximizing the estimated value for an attribute.
Jahanian \etal \cite{DBLP:journals/corr/abs-1907-07171} uses data augmentation to learn latent space trajectories for transformations like shift, rotation, scale, brightness.
Voynov \etal \cite{DBLP:journals/corr/abs-2002-03754} proposes an unsupervised approach for
the discovery of semantically meaningful directions in the GAN latent space \eg background blur, background removal, luminance, zoom, rotation.
Structural Noise \cite{DBLP:journals/corr/abs-2004-12411} edits
the input tensor in a GAN, allowing to spatially edit parts of an image.
StyleMapGAN \cite{DBLP:journals/corr/abs-2104-14754} proposes a new architecture based on StyleGAN by extending StyleGAN's mapping network to produce style maps instead of style vectors. They show that it is possible to transfer style between fragments of images.

\section{Datasets}
In this section, we describe datasets that we use in this work.

\subsection{Flickr-Faces-HQ}
Flickr-Faces-HQ (FFHQ) \cite{DBLP:journals/corr/abs-1812-04948} was introduced together with StyleGAN.
This dataset consists of 70\,000 high-quality images at $1024 \times 1024$ resolution of human faces.
It contains a lot of age, ethnicity, image background and use of accessories variation. Images were crawled from Flickr.
Images are aligned and cropped consistently.

\subsection{LSUN}
The Large-scale Scene Understanding (LSUN) \cite{DBLP:journals/corr/YuZSSX15} is a dataset containing 10 scene categories
such as dining room, bedroom, chicken, outdoor church, and so on. Every category has a considerable number of images ranging from around 120\,000 to 3\,000\,000. The size of images is $256 \times 256$.

\subsection{SpaceNet}
SpaceNet \cite{spacenet} is a dataset of satellite imagery.
Currently it contains around 67\,000 square km of very high-resolution imagery, >11M building footprints, and 20\,000 km of road labels.
There are eleven different areas of interests, featuring different cities and regions, \eg Rio De Janeiro, Vegas, Paris, Shanghai.
SpaceNet provides us with 3 different kinds of image representation: panchromatic, RGB, and 8-band channels.
The dataset is associated with SpaceNet Challenges focusing on different problems, \eg building detection, road network detection, multi-sensor all-weather mapping, multi-temporal urban development.

\section{Metrics and Losses}
In this section, we describe non-trivial metrics and losses that we use in this work.

\subsection{Fréchet Inception Distance}
The Fréchet Inception Distance (FID) \cite{DBLP:journals/corr/HeuselRUNKH17} is a method for comparing similarity of two image distributions to each other. It is commonly used between real images and generated images by a GAN as a measure of GAN performance.

The FID passes images from both distributions through the pre-trained inception network \cite{DBLP:journals/corr/SzegedyVISW15} obtaining embeddings for every image, \eg 2048 dimensional embedding vectors for Inception V3. For both distribution, the estimated mean embedding vector and the covariance matrix is calculated. The final distance is defined as following:
\begin{equation}
d^2\funarg{(m_1, C_1), (m_2, C_2)} = || m_1 - m_2 ||^2_2 + \mathrm{Tr}\funarg{C_1 + C_2 - 2(C_1 C_2)^{1/2}},
\end{equation}
where $(m_1, C_1)$ is an estimated mean embedding and covariance matrix from one distribution and $(m_2, C_2)$ from the other.

\subsection{Learned Perceptual Image Patch Similarity}
Learned Perceptual Image Patch Similarity (LPIPS) \cite{DBLP:journals/corr/abs-1801-03924} is used as a metric or a loss function of visual similarity between two images.
It passes both images through a pre-trained convolutional backbone, extracting intermediate activation feature maps.
After calculating the squared distance between corresponding activation maps, that difference is passed into
a trained point-wise convolution layer for tuning of channel importance and then is averaged. The final value of LPIPS is a mean of tuned distances of all intermediate activations.

\subsection{Noise Regularization}
Noise regularization method from StyleGAN2 \cite{Karras2019stylegan2} was introduced for GAN inversion problem in order to optimize latent representation together with noise maps ensuring that they carry a coherent signal.
This regularization method tries to preserve some properties of $\mathcal{N}\funarg{0, I}$ distribution.
It consists of two steps.
\begin{itemize}
    \item The first step is formulated as a regularization loss and tries to preserve independence between elements.
        For every noise maps, it is done by minimizing a product of following noise map elements at multiple resolution scales obtained by averaging $2 \times 2$ pixel neighborhoods up to $8 \times 8$ size. See Algorithm~\ref{alg:noise-reg-loss} for pseudocode.
    \item The second step is done after optimization step and focus on enforcing zero mean and unit standard deviation.
        It is done by noise map standardization, \ie from every noise map its mean is subtracted and divided by its standard deviation.
\end{itemize}

\begin{algorithm}
    \caption{Noise Regularization Loss}
    \label{alg:noise-reg-loss}
    \begin{algorithmic}
        \State $\mathrm{loss} \gets 0$
        \For{every noise map $n$ in the representation}
        \While{spatial size of noise map $n$ is greater or equal $8 \times 8$}
        \State $\mathrm{loss} \gets \mathrm{loss} + \mathrm{mean}\funarg{n \odot \mathrm{roll}\funarg{n, \mathrm{shift}=1,\mathrm{axis}=\mathrm{``width"}}}^2$
        \State $\mathrm{loss} \gets \mathrm{loss} + \mathrm{mean}\funarg{n \odot \mathrm{roll}\funarg{n, \mathrm{shift}=1,\mathrm{axis}=\mathrm{``height"}}}^2$
        \State $n \gets \mathrm{\text{average\_pool}}\funarg{n, \mathrm{kernel\_size}=2, \mathrm{stride}=2}$
        \EndWhile
        \EndFor
        \State \Return $\mathrm{loss}$
    \end{algorithmic}
\end{algorithm}

\chapter{Spatial Latent Space Definition}
\label{chap:def}

In this work, we use StyleGAN2 \cite{Karras2019stylegan2} generator architecture.
Originally StyleGAN2 takes an latent vector $z \in \mathcal{Z}$ and noise maps $n \in \mathcal{N}$ and generates an image.
To make latent representation spatially aware, we trim some part of the architecture and push latent representations further. We propose two ideas:
\begin{itemize}
    \item Trimming a certain number of building blocks in synthesis network. This exposes intermediate feature and RGB maps, which we add to the latent representation.
        Because some style vectors and noise maps are only used in removed layers, we can also remove them from the latent space.
        Note that both feature and RGB maps are spatial and features at a given coordinate affect only surrounding of this coordinate in the generated image.
    \item Trimming mapping network and expanding spatially style vectors. Removing mapping network exposes a style vector ($\mathcal{W}$ space) or a list of style vectors ($\mathcal{W}^+$ space) depending if we account for style vector duplication. Style vectors are not spatial alone, however we expand them on spatial dimensions in a way that spatial dimensions match these of a feature map in a given synthesis network block.
    Style vectors are used for demodulation of convolution weights, and in spatial version we apply them independently at every spatial coordinate.
    For implementation details see \ref{style-layer-impl}.
\end{itemize}
We show components used for our proposed spatial latent spaces in Figure~\ref{fig:latent-space-definition} and definitions in Table~\ref{def:spaces}.

\begin{figure}[t]
    \centering
    \includegraphics[width=\textwidth, align=c]{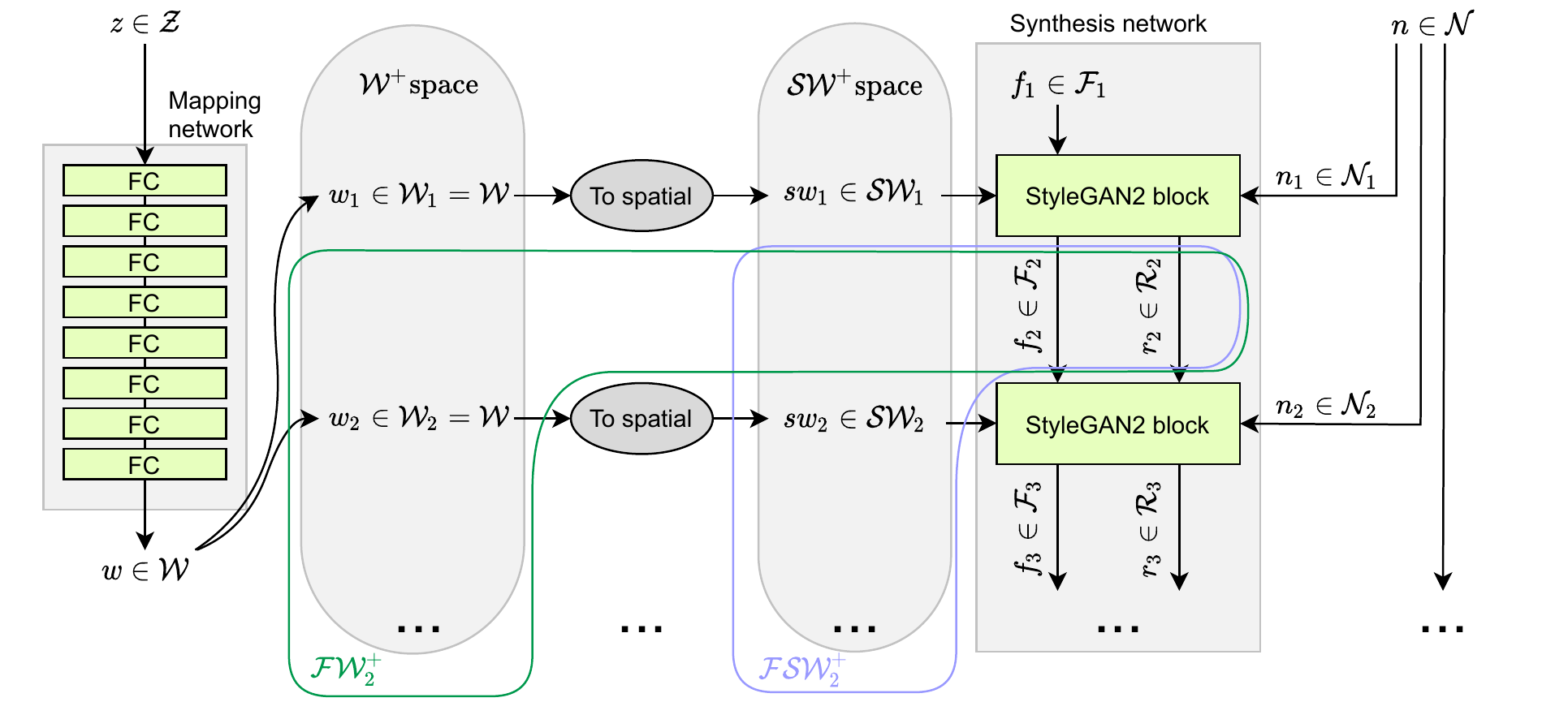}
    \caption{
        Components used for latent space definitions.
        A latent $z \in \mathcal{Z} = \mathbb{R}^{512}$ is mapped into $w \in \mathcal{W} = \mathbb{R}^{512}$ using mapping network, which is replicated for every layer into $w^+ \in \mathcal{W}^+ = \mathbb{R}^{18 \times 512}$ in the same way as in StyleGAN2.
        Style vectors $w_i \in \mathcal{W}_i$ are expanded spatially into $\mathcal{SW}_i$ space in a way that spatial dimensions match these in StyleGAN2 block.
        Note that $\mathcal{SW}_i$ have different spatial size depending on $i$.
        $f_i \in \mathcal{F}_i$ represents a feature map and $r_i \in \mathcal{R}_i$ represents an RGB map before $i$-th block.
        Note that $F_i$ and $R_i$ may have different spatial dimensions for the same $i$, and in particular $\mathcal{R}_1$ doesn't exist.
        In StyleGAN2, $f_1$ is a learnable $\mathbb{R}^{512 \times 4 \times 4}$ tensor.
        $\mathcal{N} = \prod_j \mathcal{N}_j$ is a space of noise maps.
        We construct spatial latent spaces from all $\mathcal{SW}_i$ maps ($\SWp$ space),
        a combination of $\mathcal{F}_i$, $\mathcal{R}_i$ and $\mathcal{W}_j$, $j \geq i$ ($\FWp{i}$ space),
        and a combination of $\mathcal{F}_i$, $\mathcal{R}_i$ and $\mathcal{SW}_j$, $j \geq i$ ($\FSWp{i}$ space).
        In addition we define similarly spaces with noise maps, \ie $\NSWp$, $\FNWp{i}$ and $\FNSWp{i}$.
    }
    \label{fig:latent-space-definition}
\end{figure}

\begin{table}
    \setlength{\tabcolsep}{0.9ex}
    \renewcommand\arraystretch{2}
    \centering
    \begin{tabular}{lp{0.5\linewidth}cr}
        Name & Description & Formula & \\\hline\hline
        $\Wp            $ & A space containing style vectors. Note that this space is already used in some works (\eg \cite{Karras2019stylegan2, DBLP:journals/corr/abs-1904-03189}), however terminology is inconsistent when it comes to containing noise maps as noise maps are not always treated as a proper part of a latent space. In this work, $\Wp$ do not contain noise maps. & $\prod_{j} \mathcal{W}_j $ & (3.1) \\\hline
        $\SWp           $ & A space containing spatial style maps. Generated images will be aligned as in the training set, however styles can be different for different spatial coordinates. & $\prod_{j} \mathcal{SW}_j                                                                    $ & (3.2) \\\hline
        $\FWp{i}        $ & A space containing $i$-th feature and RGB map, and non-spatial style vectors. Generated images will have the same style spatially, however the space is expected to be capable of encoding images that are unaligned (\eg translated spatially) or have different number or arrangement of object parts. & $\mathcal{F}_i \times \mathcal{R}_i \times \prod_{j \geq i} \W_j                             $ & (3.3) \\\hline
        $\FSWp{i}       $ & A space containing $i$-th feature and RGB map, and spatial style maps. It combines $\SWp$ and $\FWp{i}$ ideologically. & $\mathcal{F}_i \times \mathcal{R}_i \times \prod_{j \geq i} \mathcal{SW}_j                   $ & (3.4) \\\hline
        $\NWp           $ & A space containing non-spatial style vectors and noise maps. The $\Wp$ space with additional noise maps.   & $\Wp \times \prod_{j} \mathcal{N}_j                                                          $ & (3.5) \\\hline
        $\NSWp          $ & A space containing spatial style maps and noise maps. The $\SWp$ space with additional noise maps.         & $\SWp \times \prod_{j} \mathcal{N}_j                                                         $ & (3.6) \\\hline
        $\FNWp{i}       $ & A space containing $i$-th feature and RGB map, non-spatial style vectors, and noise maps. The $\FWp{i}$ space with additional noise maps. & $\FWp{i} \times \prod_{j \geq i} \mathcal{N}_j                                               $ & (3.7) \\\hline
        $\FNSWp{i}      $ & A space containing $i$-th feature and RGB map, spatial style maps, and noise maps. The $\FSWp{i}$ space with additional noise maps. & $\FSWp{i} \times \prod_{j \geq i} \mathcal{N}_j                                              $ & (3.8) \\\hline
        $\mathcal{FNZ}_i$ & A space containing $i$-th feature and RGB map, latent vector from the $\mathcal{Z}$ space, and noise maps. & $\mathcal{F}_i \times \mathcal{R}_i \times \mathcal{Z} \times \prod_{j \geq i} \mathcal{N}_j $ & (3.9) \\\hline
    \end{tabular}
    \caption{
        Latent spaces for StyleGAN2.
    }
    \label{def:spaces}
\end{table}

We define the $\FSWp{i}$ latent space as the Cartesian product of $F_i$, $R_i$, $\mathcal{SW}_j$, $j \geq i$,
and similarly the $\FWp{i}$ latent space, from $F_i$, $R_i$, and a fragment of $\Wp$ for the block $i$ and later.
Spaces $\Wp$, $\SWp$, $\FWp{i}$, $\FSWp{i}$ don't contain noise maps.
We define $\NWp$, $\NSWp$, $\FNWp{i}$, $\FNSWp{i}$ the same way but with appropriate noise maps,
\ie $\FNWp{i}$ and $\FNSWp{i}$ contain only noise maps for block $i$ and later.

Intuitively, $\mathcal{F}_i$ and $\mathcal{R}_i$ fragments should contain information about positions of certain objects \eg eyes, noses, mouths and background,
whereas $\Wp$ and $\SWp$ keep styles for the image \eg colors, open/close mouth/eyes, gender.
$\SWp$ can represent different styles depending on spatial position.

Given model parameters, we can convert some latent representations into others without changing the generated image
by forward passing the representation through the network up to a certain point.
\Eg the $\SWp$ representation can be converted into $\FSWp{1}$ by adding to the representation constant input tensor as a feature map,
$\FSWp{1}$ can be converted into $\FSWp{i}$ by passing the feature and RGB maps through blocks from 1 to $i - 1$,
and similarly $\FSWp{i}$ can be converted into $\FSWp{j}$, where $i < j$,
and analogically for representations containing noise maps.

\chapter{Latent Space Mixing}
\label{chap:mix}

In this chapter, we focus on possibilities of latent representation mixing.
Informally, it can be defined as: given two latent representations $v_1$ and $v_2$, and some mixing parameter $m$,
obtain a third latent variable $v_3$ being a mix of $v_1$ and $v_2$.
For instance a common way of linear latent vector interpolation look as the following:
\begin{equation}
    \mathop{\text{linear-interpolation}}\funarg{v_1, v_2, m} = v_1 + m \cdot \parens{v_2 - v_1},
\end{equation}
where $v_1, v_2 \in \mathbb{R}^d$, $d \in \mathbb{N}_+$ is latent space dimensionality, $m \in [0, 1]$ is a mixing coefficient.

\section{Approach}
In order to leverage the spatial structure of latent spaces, we allow mixing depending on spatial coordinates in order to combine two latent representations with each other.
We use \textit{masked mixing}:
\begin{equation}
\mix\funarg{v_1, v_2, m} = v_1 + m \odot \parens{v_2 - v_1},
\label{eq:mix}
\end{equation}
where $v_1, v_2 \in \mathbb{R}^{C \times H \times W}$ are 3D tensors with spatial shapes $(H, W)$ and with $C$ channels,
$m : [0, 1]^{H \times W}$ is a mask transforming spatial coordinates into linear interpolation coefficients.


For our spatial mixing, we use this formula in our experiments on all spatial elements of latent spaces
\ie feature maps, RGB maps, style maps, except for noise maps,
which are additionally divided by $\sqrt{2 m^2 - 2 m + 1}$ in order to preserve standard deviation,
assuming independence of noise map sampling. 
Additionally when noise maps are sampled from $\mathcal{N}\funarg{0, I}$,
such transformation won't change the distribution at all.
Proof of these facts can be found in \ref{preserve-std}.

Note that for some of our experiments, we spatially translate representations prior to mixing.
That allows to obtain images that have different number and arrangement of objects.
We can also obtain images of a different size by having only partial overlap during mixing.

\section{Experiments}

In all our experiments with face images we use the official pre-trained StyleGAN2 model checkpoint on the FFHQ \cite{DBLP:journals/corr/abs-1812-04948} dataset.
We try a few different variations of mixing, but the general idea is the following:
Sample variables from $\Wp$ latent space, transform them to a spatial latent space,
mix them together, and generate the image.

Figures~\ref{fig:mix-SWp} and \ref{fig:interp-SWp} show that we can mix $\NSWp$ latent representations spatially and obtain consistent spatial transitions.
Figure \ref{fig:church-mix} mix latent representations in $\FNSWp{}$ space, obtaining a higher resolution feature and RGB map, which generates
a higher resolution image.

\begin{figure}
    \setlength{\tabcolsep}{0pt}
    \begin{tabular}{ccccc}
        Source 1 & Source 2 & Raw mix & \parbox[t]{0.2\textwidth}{\centering Mix in $\NSWp$} & \parbox[t]{0.2\textwidth}{\centering Mix in $\NSWp$ with smoothing} \\
        \includegraphics[width=0.2\textwidth]{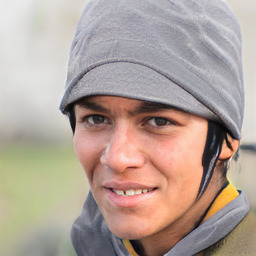} &
        \includegraphics[width=0.2\textwidth]{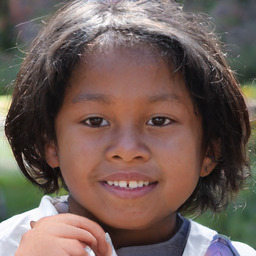} &
        \includegraphics[width=0.2\textwidth]{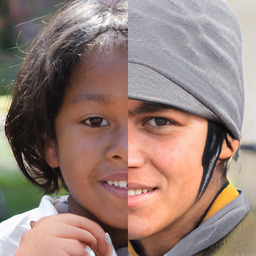} &
        \includegraphics[width=0.2\textwidth]{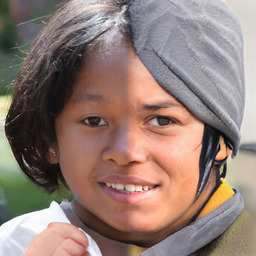} &
        \includegraphics[width=0.2\textwidth]{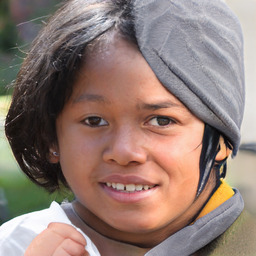} \\
        \includegraphics[width=0.2\textwidth]{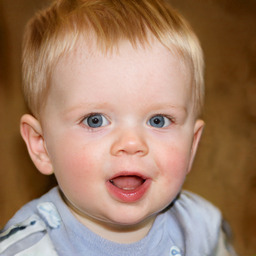} &
        \includegraphics[width=0.2\textwidth]{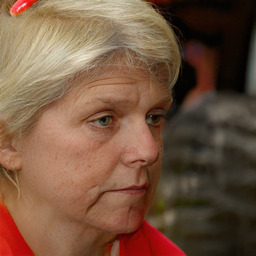} &
        \includegraphics[width=0.2\textwidth]{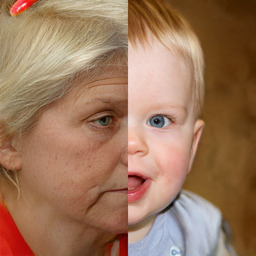} &
        \includegraphics[width=0.2\textwidth]{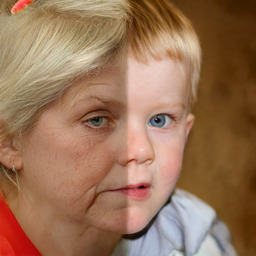} &
        \includegraphics[width=0.2\textwidth]{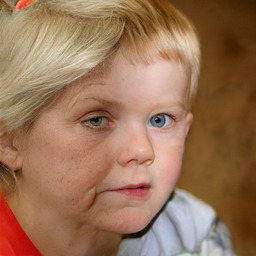} \\
        
        & & & 
        \includegraphics[width=0.2\textwidth,height=0.2\textwidth,decodearray={0.15 0.85}]
            {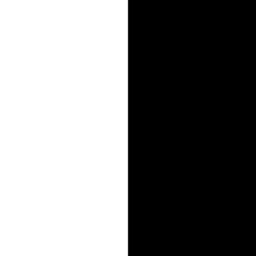} &
        \includegraphics[width=0.2\textwidth,height=0.2\textwidth,decodearray={0.15 0.85}]
            {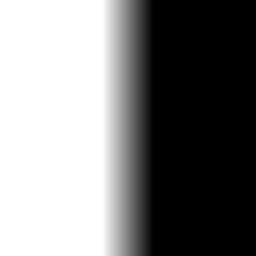} \\
    \end{tabular}
    \caption{
        For both rows, we sample two latent variables from $\NWp$, transform into $\NSWp$
        and mix them using Equation~\ref{eq:mix}. 3rd column shows direct mixing in the pixel space.
        4th and 5th column show mixing in $\NSWp$ using a mask from the bottom row.
        Mixing latent variables in $\NSWp$ is able to align certain parts like mouths, noses, \protect\etc,
        and produce a realistic transition between images.
        Because fine styles like color scheme are close to the end of the network,
        to improve mixing, we make the mixing mask more smooth.
    }
    \label{fig:mix-SWp}
\end{figure}

\begin{figure}
    \setlength{\tabcolsep}{0pt}
    \begin{tabular}{cccccccc}
        Source 1 & \multicolumn{6}{c}{$\xrightarrow{\makebox[0.7\linewidth]{}}$} & Source 2 \\
        \includegraphics[width=0.125\textwidth]{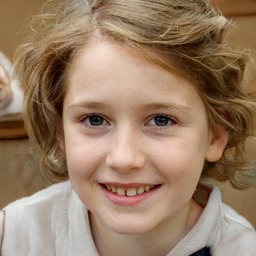} &
        \includegraphics[width=0.125\textwidth]{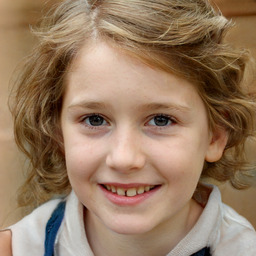} &
        \includegraphics[width=0.125\textwidth]{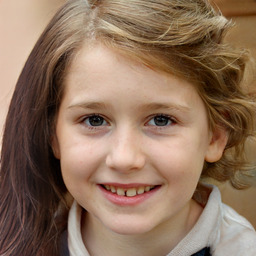} &
        \includegraphics[width=0.125\textwidth]{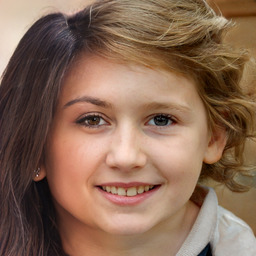} &
        \includegraphics[width=0.125\textwidth]{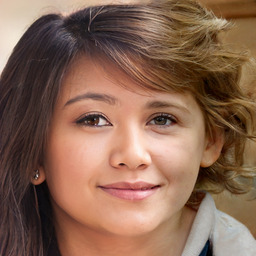} &
        \includegraphics[width=0.125\textwidth]{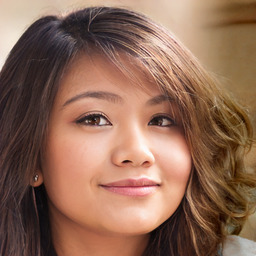} &
        \includegraphics[width=0.125\textwidth]{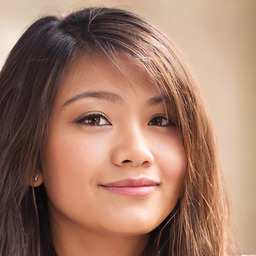} &
        \includegraphics[width=0.125\textwidth]{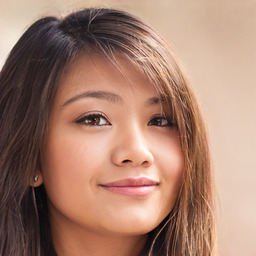} \\
        \includegraphics[width=0.125\textwidth]{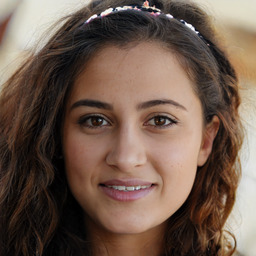} &
        \includegraphics[width=0.125\textwidth]{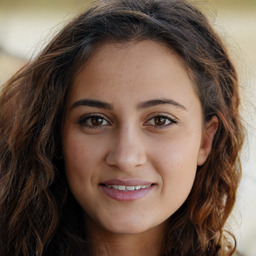} &
        \includegraphics[width=0.125\textwidth]{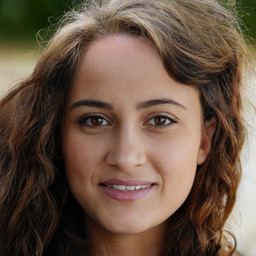} &
        \includegraphics[width=0.125\textwidth]{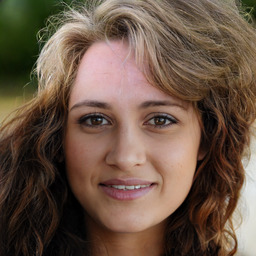} &
        \includegraphics[width=0.125\textwidth]{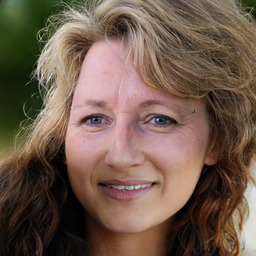} &
        \includegraphics[width=0.125\textwidth]{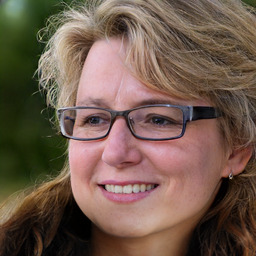} &
        \includegraphics[width=0.125\textwidth]{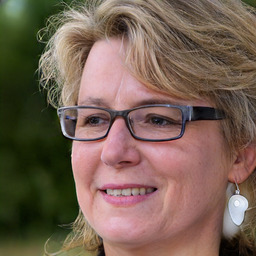} &
        \includegraphics[width=0.125\textwidth]{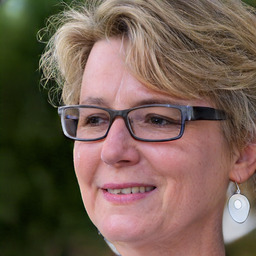} \\
        \includegraphics[width=0.125\textwidth]{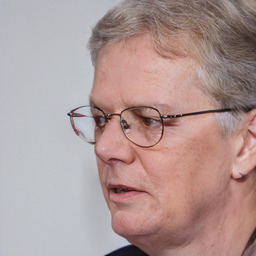} &
        \includegraphics[width=0.125\textwidth]{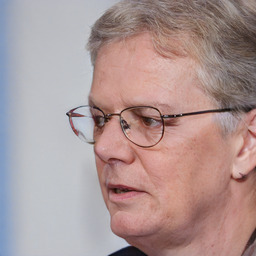} &
        \includegraphics[width=0.125\textwidth]{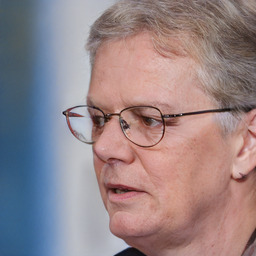} &
        \includegraphics[width=0.125\textwidth]{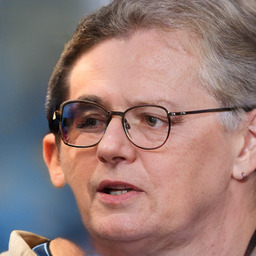} &
        \includegraphics[width=0.125\textwidth]{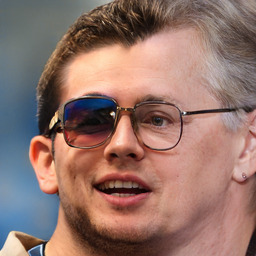} &
        \includegraphics[width=0.125\textwidth]{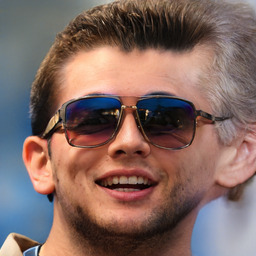} &
        \includegraphics[width=0.125\textwidth]{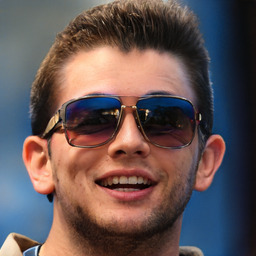} &
        \includegraphics[width=0.125\textwidth]{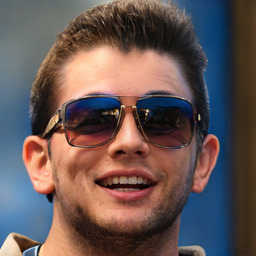} \\
    \end{tabular}
    \caption{
        Similarly to Figure \ref{fig:mix-SWp}, we obtain latent variables in $\NSWp$ space
        and translate a mixing mask to simulate spatial transition between images
        (vertical or horizontal).
    }
    \label{fig:interp-SWp}
\end{figure}

\begin{figure}
    \includegraphics[width=1\textwidth,align=c]{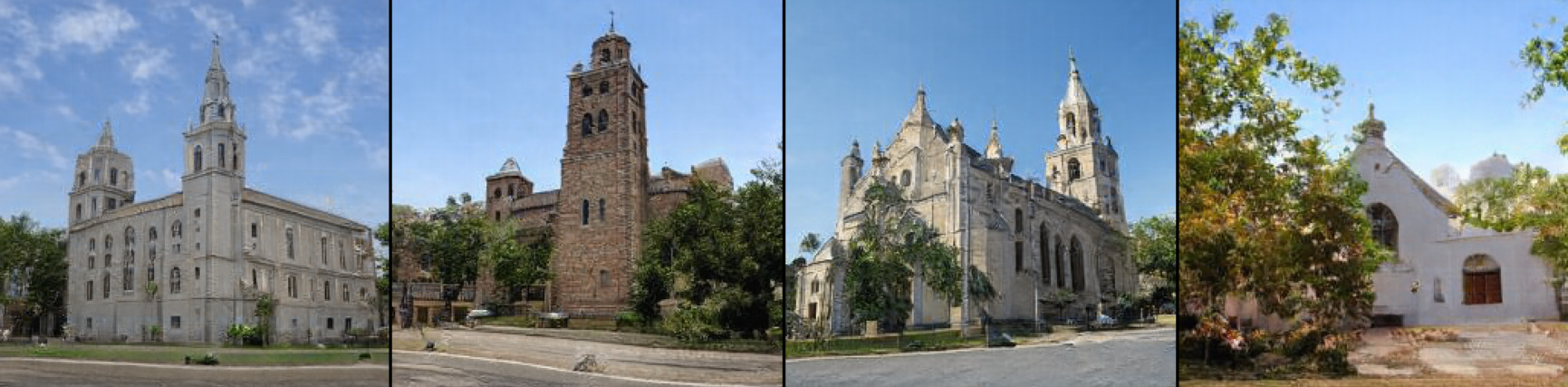}
    \includegraphics[width=1\textwidth,align=c]{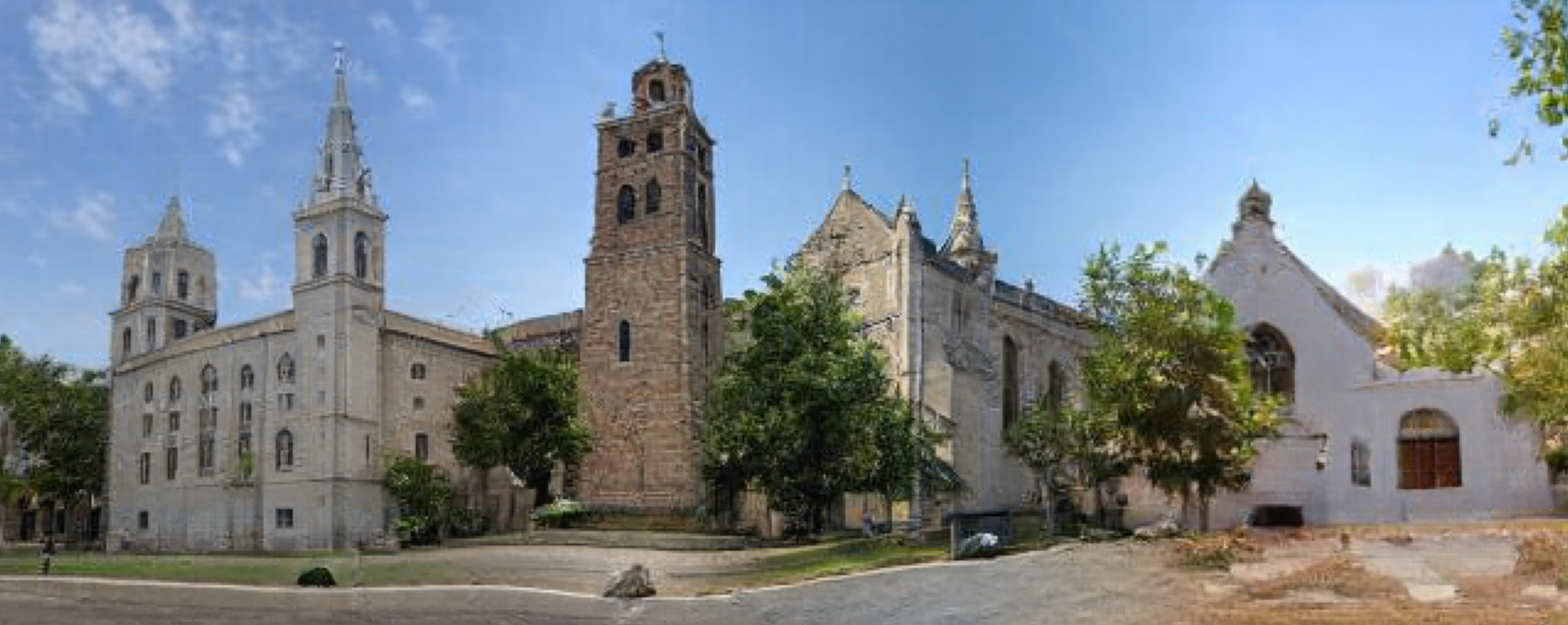}
    \caption{
        We take an official StyleGAN2 model trained on LSUN Churches \cite{DBLP:journals/corr/YuZSSX15}.
        We sample four representations from $\NWp$, and transform into $\FNSWp{5}$.
        We combine them into one $\FNSWp{5}$ representation by spatially extending all latent maps.
        We use a smoothed mask to obtain a smoother transition between images.
    }
    \label{fig:church-mix}
\end{figure}

\chapter{Latent Attribute Editing}
\label{chap:edit}
A common way to edit latent representation in order to obtain a targeted effects,
\eg changing the age, gender, adding glasses \etc,
is to obtain an attribute direction vector in $\W$ or $\Wp$ space.
To obtain such vector, one can train a linear regression that transforms a latent representation into a value
for a given attribute. The normal vector of a hyperplane of the logistic regression model
is a desired direction vector, because translating a latent representation in the same space along this vector
results in a change of a given attribute.

\section{Approach}

We propose a method that uses existing attribute direction vectors in $\mathcal{W}$ or $\Wp$ space.
Application of the attribute vector is straightforward in the $\SWp$ space, by translating style maps at every spatial coordinate along the vector.
However such approach in $\FSWp{}$ and $\FWp{}$ fails to correctly capture coarse styles that are formed at
style layers that are before the feature map in the representation (see Figure~\ref{fig:attribute-model-effect}).

To tackle that problem, we train an \textit{attribute model} to find a transformation for feature and RGB maps in the $\FWp{}$ space for a given attribute direction vector.
For every iteration we sample a mini-batch from $\NWp$ and obtain a prediction target by applying direction vector in $\NWp$.
We transform the input mini-batch into $\FNWp{i}$, pass feature maps and RGB maps through the attribute model and
apply the attribute direction vector on the remaining style maps.
The attribute model predicts the difference from the input, which is multiplied by a target attribute strength.
The objective is to maximize similarity of obtained feature and RGB maps to intermediate activations from the forward pass of edited $\NWp$ representations, as well as maximize similarity of generated images.
For image similarity, we use LPIPS \cite{DBLP:journals/corr/abs-1801-03924} loss.
Algorithm~\ref{alg:attribute-model} and Figure~\ref{fig:attribute-model} show the pseudocode and the diagram with details.
Note that attribute model is also able to operate on $\FNSWp{i}$ or spaces without noise maps.

\begin{algorithm}
    \caption{
        Attribute Model Training. \\
        Model training consists of generating a mini-batch from $\NWp$ and editing it using an attribute vector. We consider such edited latent representations a prediction target for the model.
        We convert the initial mini-batch into the $\FNWp{i}$ space and try to predict
        the offset for feature and RGB maps to match the assumed prediction target.
    }
    \label{alg:attribute-model}
    \begin{algorithmic}
        \Require{ \\
            \;\; $v$ -- attribute direction vector \\
            \;\; $c_{\mathrm{min}}$, $c_{\mathrm{max}}$ -- \algparbox{0.7\linewidth}{min and max value for the attribute direction vector scaling coefficient \vspace{0.35em}} \\
            \;\; $\theta$ -- generator model weights \\
            \;\; $i$ -- index of a feature and RGB map defining the latent space \\
            \;\; $\lambda_f$ -- coefficient for feature and RGB map loss \\
            \;\; $\lambda_{\mathrm{lpips}}$ -- coefficient for LPIPS loss
        }
        \vspace{0.5em}
        \State $M \gets \text{\algparbox{0.9}{randomly initialized 2D point-wise convolution layer without bias}}$
        \While{\text{not converged}}
        \State $b \gets \text{sample a mini-batch from $\NWp$}$
        \State $c \gets \text{sample $\mathrm{length}\funarg{b}$ coefficients from $[c_{\mathrm{min}}, c_{\mathrm{max}}]$}$
        \Comment{\algparbox{0.35}{A list of attribute vector coefficients \vspace{0.35em}}}
        \State $y \gets \text{translate $b$ by vector $c \cdot v$}$
        \Comment{\algparbox{0.35}{Apply attribute vector with a given strength to mini-batch $b$}}
        \State $y_I \gets \text{\algparbox{0.55}{generate images from $y$ using weights $\theta$}}$
        \State $y_f, y_r \gets \text{\algparbox{0.5}{take feature and RGB maps from $y$ transformed into $\FNWp{i}$ space}}$
        \Comment{\algparbox{0.35}{we consider $y_I$, $y_f$, $y_r$ prediction targets that we want to predict from spatial representations}}
        \State $f, r \gets \text{\algparbox{0.5}{take feature and RGB maps from $b$ transformed into $\FNWp{i}$ space}}$
        \State $df, dr \gets M\funarg{\mathrm{concatenate}\funarg{f, r}}$
        \Comment{\algparbox{0.35}{The $M$ model predicts an offset of feature and RGB maps}}
        \State $x_f, x_r \gets (f, r) + c \cdot (df, dr)$
        \State $x \gets \text{\algparbox{0.6}{form $\FNWp{i}$ representations using \\ $x_f$, $x_r$ and style and noise maps from $y$}}$
        \State $x_I \gets \text{\algparbox{0.85}{generate images from $x$ using weights $\theta$}}$
        \State $L \gets$ \algparbox{0.8}{$\lambda_f (\mathrm{MSE}\funarg{x_f, y_f} + \mathrm{MSE}\funarg{x_r, y_r}) + \lambda_{\mathrm{lpips}} \mathrm{LPIPS}\funarg{x_I, y_I}$}
        \State $\text{optimize weights of $M$ using loss $L$}$
        \EndWhile
        \State \Return $M$
    \end{algorithmic}
\end{algorithm}

\begin{figure}
    \centering \includegraphics[width=0.8\textwidth]{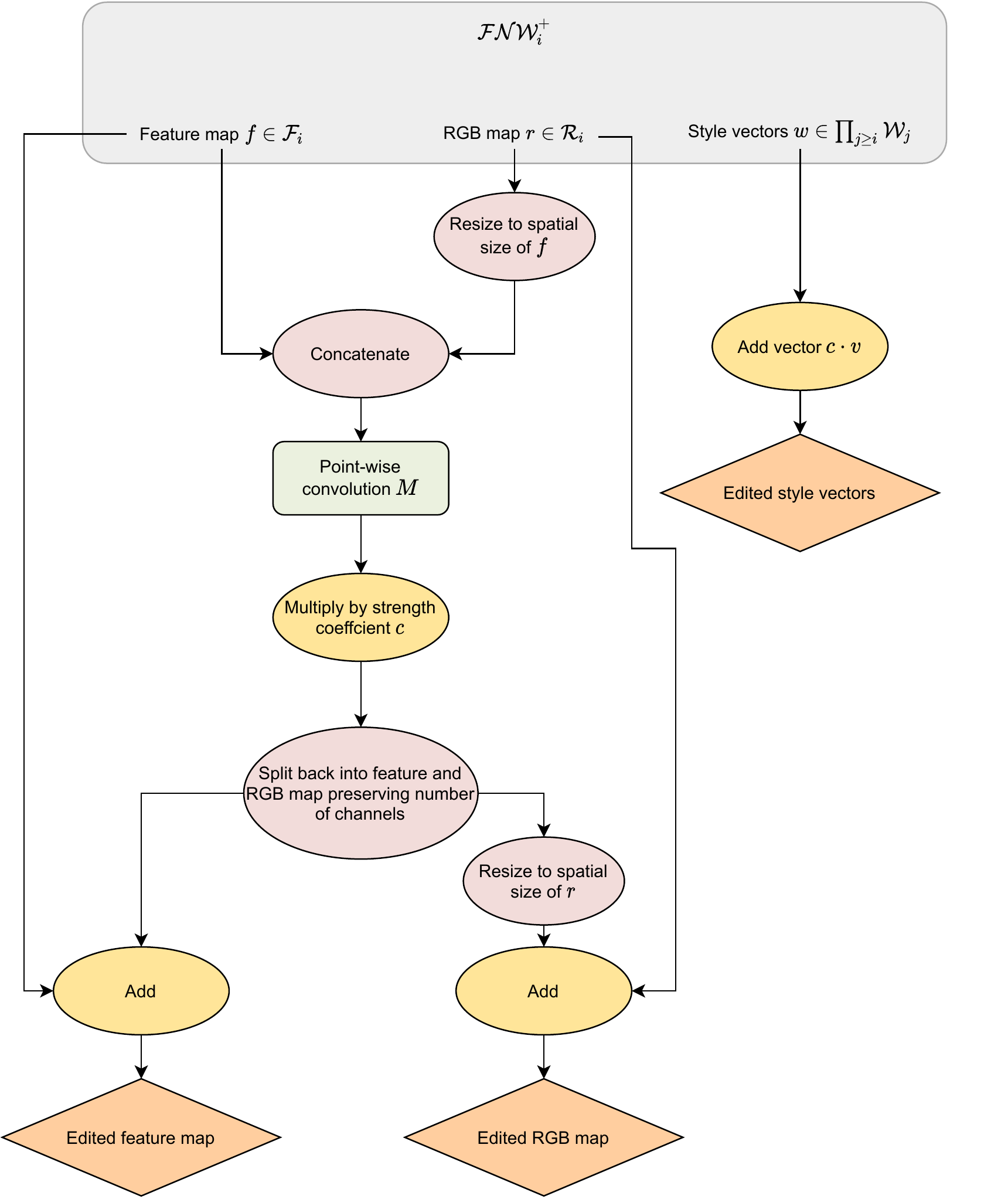}
    \caption{
        Diagram of attribute model.
        The feature and RGB map in the representation are concatenated and unit offset is predicted using the convolution layer $M$.
        We multiply the offset by the strength coefficient $c$ and add it to input feature and RGB maps, obtaining edited representation.
    }
    \label{fig:attribute-model}
\end{figure}

\section{Experiments}

We use latent attribute direction vectors from \cite{twitter:latent-dirs1} and \cite{github:latent-dirs2}.
Attribute models are trained using $\lambda_f = \lambda_{\mathrm{lpips}} = 1$, Adam optimizer \cite{kingma2014adam} with learning rate of $1\text{e-}4$, $\beta_1 = 0.9$, $\beta_2 = 0.999$, batch size of 12, and direction vector length scale coefficient $c_{\mathrm{max}} = -c_{\mathrm{min}} = \frac{20}{||v||_2}$, when $v$ is the attribute direction vector, \ie the length of a vector is up to 20 units. We downscale images 4 times for the LPIPS loss, as it carries a large memory footprint otherwise.

In Figure~\ref{fig:mix-FSWp}, we show that $\FNSWp{}$ besides being capable of representing out-of-sample images
it also preserves semantic information about elements on the image, like face parts.
Our attribute models are fully capable of operating on such representations. 

\begin{figure}
    \setlength{\tabcolsep}{0pt}
    \renewcommand{\arraystretch}{0}
    \def\ssize{0.15}
    \def\wsize{0.1875}
    \centering
    \begin{tabular}{lcccc}
    
        \rotatebox[]{90}{Source 1}\; &
        \includegraphics[align=c,width=\ssize\textwidth]{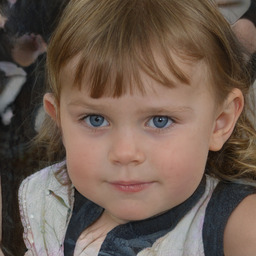} &
        \includegraphics[align=c,width=\ssize\textwidth]{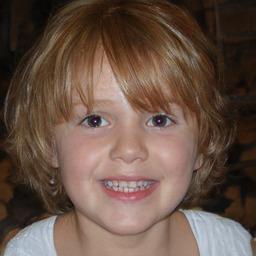} &
        \includegraphics[align=c,width=\ssize\textwidth]{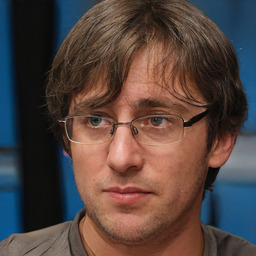} &
        \includegraphics[align=c,width=\ssize\textwidth]{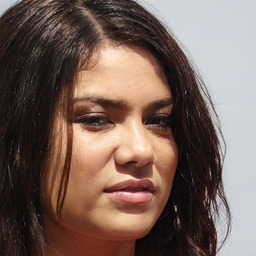} \\
        
        \rotatebox[]{90}{Source 2}\; &
        \includegraphics[align=c,width=\ssize\textwidth]{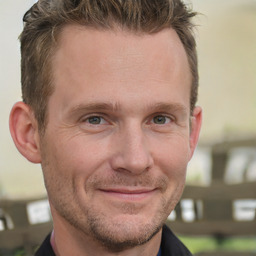} &
        \includegraphics[align=c,width=\ssize\textwidth]{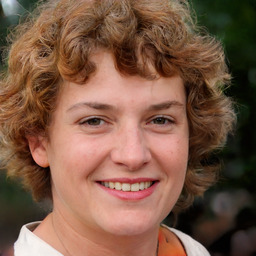} &
        \includegraphics[align=c,width=\ssize\textwidth]{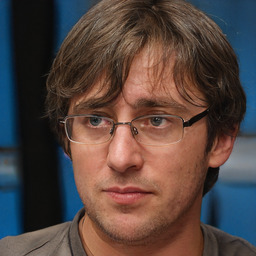} &
        \includegraphics[align=c,width=\ssize\textwidth]{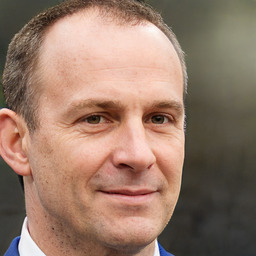} \\
        
        \rotatebox[]{90}{Raw mix}\; &
        \includegraphics[align=c,width=\wsize\textwidth]{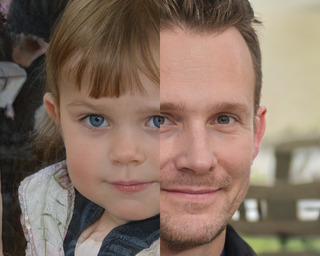} &
        \includegraphics[align=c,width=\wsize\textwidth]{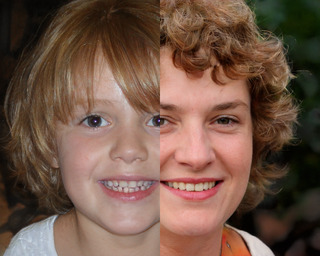} &
        \includegraphics[align=c,width=\ssize\textwidth]{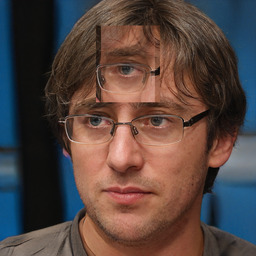} &
        \includegraphics[align=c,width=\ssize\textwidth]{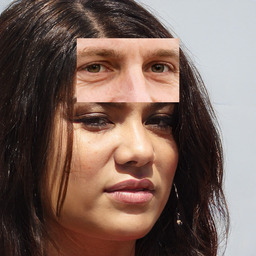} \\
        
        \rotatebox[]{90}{\centering \parbox[c]{\ssize\textwidth}{\centering Mix in $\FNSWp{5}$}}\; &
        \includegraphics[align=c,width=\wsize\textwidth]{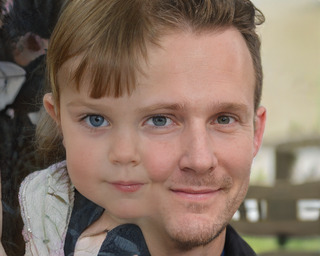} &
        \includegraphics[align=c,width=\wsize\textwidth]{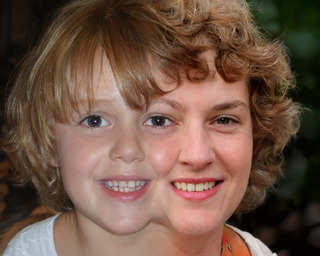} &
        \includegraphics[align=c,width=\ssize\textwidth]{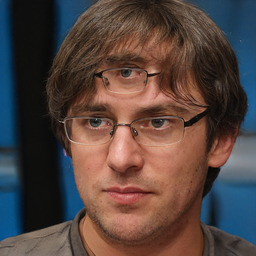} &
        \includegraphics[align=c,width=\ssize\textwidth]{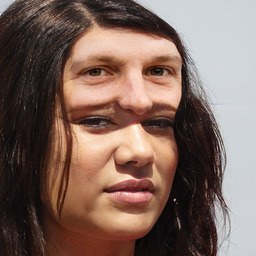} \\
        
        \rotatebox[]{90}{\small $\pm$ Age}\; &
        \includegraphics[align=c,width=\wsize\textwidth]{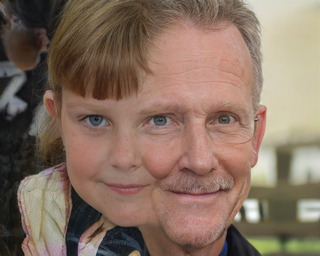} &
        \includegraphics[align=c,width=\wsize\textwidth]{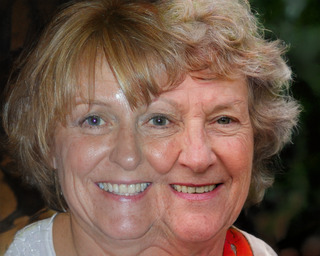} &
        \includegraphics[align=c,width=\ssize\textwidth]{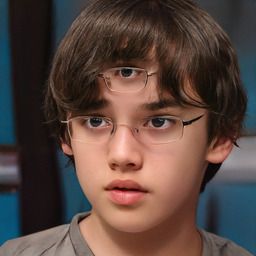} &
        \includegraphics[align=c,width=\ssize\textwidth]{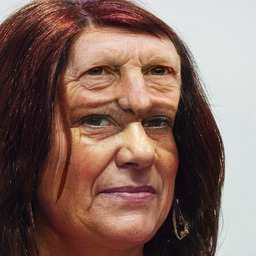} \\
        
        \rotatebox[]{90}{\small $\pm$ Eyes open}\; &
        \includegraphics[align=c,width=\wsize\textwidth]{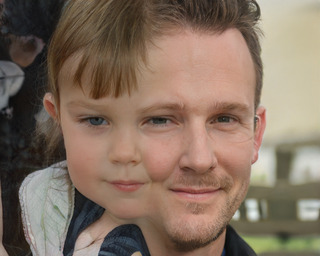} &
        \includegraphics[align=c,width=\wsize\textwidth]{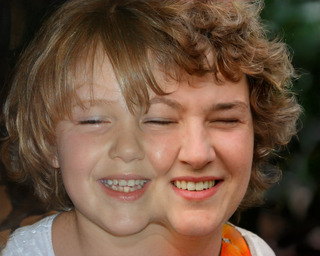} &
        \includegraphics[align=c,width=\ssize\textwidth]{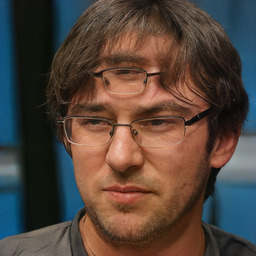} &
        \includegraphics[align=c,width=\ssize\textwidth]{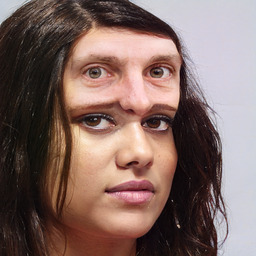} \\
        
        \rotatebox[]{90}{\small $\pm$ Glasses}\; &
        \includegraphics[align=c,width=\wsize\textwidth]{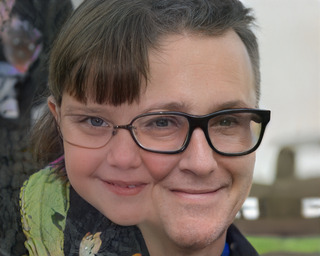} &
        \includegraphics[align=c,width=\wsize\textwidth]{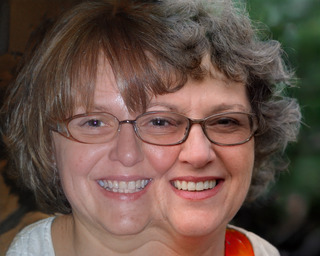} &
        \includegraphics[align=c,width=\ssize\textwidth]{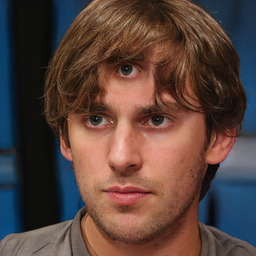} &
        \includegraphics[align=c,width=\ssize\textwidth]{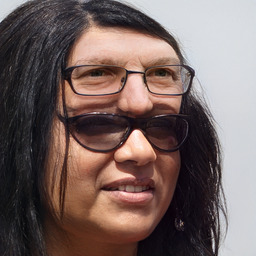} \\
        
        \vspace{0.2em} \\
        
        & a) & b) & c) & d)
        
    \end{tabular}
    \caption{
        For every column, we sample two latent variables from $\NWp$, transform into $\FNSWp{5}$
        and mix them using our method. Before mixing, we
        translate and crop the latent representation to align left and right eye in a) and b),
        take area around eyes, average them, paste them on the forehead in c),
        crop the area containing eyes and part of a nose and paste them on the forehead in d).
        We observe that semantic information is preserved,
        for example by applying our attribute model we can also close/open additional eyes,
        shadow from the upper nose in d)
        is cast on the lower nose,
        glasses generalize to three eyes in a) and [).
        Note that off-target effects like changes in colors and age when adding glasses come from attribute direction vectors, not from our methods.
    }
    \label{fig:mix-FSWp}
\end{figure}

\chapter{Latent Space Projection}
\label{chap:proj}
Informally, latent space projection is a problem of predicting a latent variable given a model and an image.
The generated image from the latent representation should be visually similar to the input image.
However, in practice the other goal is to preserve semantics of the representation, \eg ability to use image manipulation methods for the representation.

Formally, a projection of an image $I \in \mathbb{R}^{C \times H \times W}$,
given model parameters $\theta$, generator architecture $\generator{}$ and a loss function $L$,
where $H, W \in \mathbb{N}_+$ are spatial dimensions and $C$ indicates the number of channels (\eg $C = 3$ for RGB images),
into a latent space $\mathbb{S}$, can be defined as:
\begin{equation}
    \text{find $l \in \mathbb{S}$, that minimizes $L\funarg{\generator{\theta}\funarg{l}, I}$}.
\end{equation}

This problem may require regularizations or other techniques for preserving latent representation semantics, especially for high dimensional latent spaces.

\section{Approach}

We take a hybrid approach for transforming an image into $\FNWp{}$ latent space.
First, we train an image encoder into $\FWp{i}$
and then use its prediction as an initial latent representation for the optimization algorithm.

Algorithm~\ref{alg:projection-model} shows a method for training an encoder into the $\FWp{}$ latent space.
Figure~\ref{fig:projection-model} shows a diagram of the architecture.
The method trains a CNN backbone together with three heads for predicting feature maps, RGB maps and style vectors.
We sample latent representations from the $\NWp$ space, convert them into $\FNWp{i}$ and train a model to predict
feature maps, RGB maps and style vectors for a given image.
We call such model a $\FWp{i}$ encoder model as it encodes the image directly into that space.

\begin{algorithm}
    \caption{
        $\FWp{i}$ Encoder Model Training. \\
        We sample a mini-batch from $\NWp$ and convert it into the $\FWp{i}$ space, which is considered a prediction target.
        We use a CNN backbone from which we extract an appropriate intermediate activation maps from the generated image and apply prediction heads on them. For predicting the style we use the activation vector after global pooling.
    }
    \label{alg:projection-model}
    \begin{algorithmic}
        \Require{ \\
            \;\; $s$ -- CNN backbone input image scale \\
            \;\; $\theta$ -- generator model weights \\
            \;\; $\lambda$ -- \algparbox{0.9}{coefficient for feature, RGB map and style vector loss} \\
            \;\; $\lambda_{\mathrm{lpips}}$ -- coefficient for LPIPS loss \\
            \;\; $\lambda_{\mathrm{mse}}$ -- coefficient for image MSE loss
        }
        \vspace{0.5em}
        
        \State $M \gets \text{\algparbox{0.95}{randomly initialized or pretrained CNN backbone}}$
        \State $F \gets \text{\algparbox{0.95}{randomly initialized 2D point-wise convolution layer}}$
        \State $R \gets \text{\algparbox{0.95}{randomly initialized 2D point-wise convolution layer}}$
        \State $S \gets \text{\algparbox{0.95}{randomly initialized linear layer}}$
        \While{\text{not converged}}
        \State $y \gets \text{sample a mini-batch from $\NWp$}$
        \State $y_f, y_r, y_s \gets \text{\algparbox{0.75}{take feature maps, RGB maps and style vectors from $y$ transformed into $\FNWp{i}$ space}}$
        \State $y_I \gets \text{generate images from $y$ using weights $\theta$}$
        \State $I \gets \text{downscale $y_I$ with scale $s$}$
        \State $A \gets \text{\algparbox{0.85}{get activation tensors between blocks and final activation vector in $M$ from images $I$}}$
        \State $x_f \gets F\text{\algparbox{0.8}{(get last activation tensor from $A$ matching \\ \hphantom{(}spatial size of an expected feature map)}}$
        \State $x_r \gets R\text{\algparbox{0.8}{(get last activation tensor from $A$ matching \\ \hphantom{(}spatial size of an expected RGB map)}}$
        \State $x_s \gets S(\text{get activation vector from $A$})$
        \State $x \gets \text{\algparbox{0.9}{form $\FNWp{i}$ latent representation using \\ $x_f$, $x_r$, $x_s$ and noise maps from $y$}}$
        \State $x_I \gets \text{\algparbox{0.85}{generate images from $x$ using weights $\theta$}}$
        \State $L \gets $ \algparbox{0.8}{$\lambda \big(\mathrm{MSE}\funarg{x_f, y_f} + \mathrm{MSE}\funarg{x_r, y_r} + \mathrm{MSE}\funarg{x_s, y_s}\big) + \lambda_{\mathrm{mse}} \mathrm{MSE}\funarg{x_I, y_I} + \lambda_{\mathrm{lpips}} \mathrm{LPIPS} $ \algparbox{1}{(downscale $x_I$ with scale $s$, $I$)}}
        \State $\text{optimize weights of $M$, $F$, $R$, and $S$ using loss $L$}$
        \EndWhile
        \State \Return $(M, F, R, S)$
    \end{algorithmic}
\end{algorithm}

\begin{figure}
    \centering \includegraphics[width=0.8\textwidth]{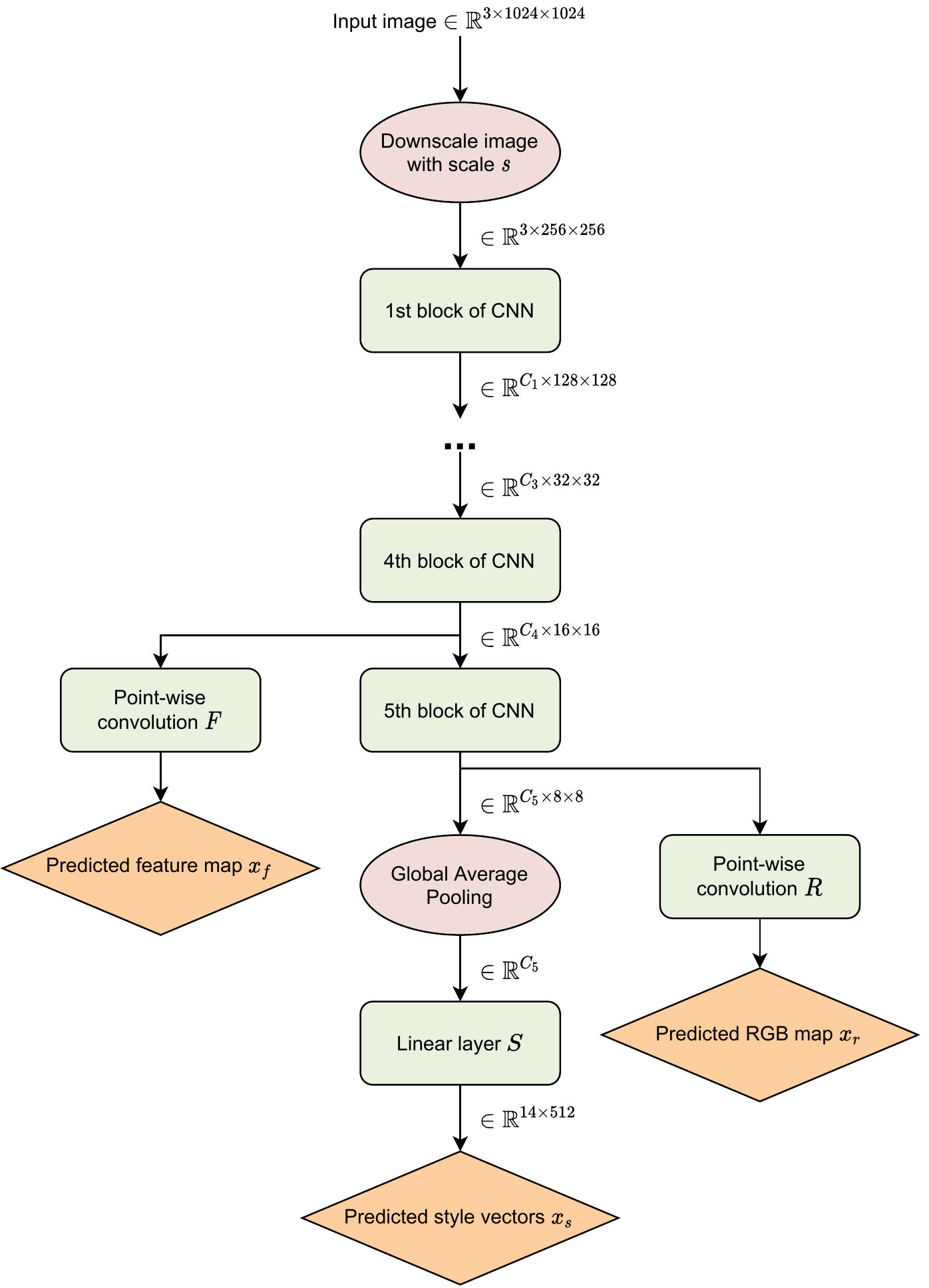}
    \caption{
        $\FWp{5}$ Encoder Model Architecture.
        5th block of StyleGAN2 takes a feature map of spatial size $16 \times 16$, and an RGB map $8 \times 8$.
        Note that the number of output style vectors is 14, because four first vectors for four first StyleGAN2 blocks are not in the $\FWp{5}$ space.
    }
    \label{fig:projection-model}
\end{figure}

Algorithm \ref{alg:latent-fitting} shows a method for projecting an image into $\FNWp{i}$ latent space.
We start from the prediction from a trained encoder model. We optimize the latent representation using gradient descent.
We use noise regularization from StyleGAN2 \cite{Karras2019stylegan2}.

\begin{algorithm}
    \caption{
        $\FNWp{i}$ Latent Representation Optimization. \\
        We start from the representation predicted by the pre-trained $\FWp{i}$ encoder model.
        For further optimization, we use backpropagation algorithm.
    }
    \label{alg:latent-fitting}
    \begin{algorithmic}
        \Require{ \\
            \;\; $I$ -- image \\
            \;\; $N$ -- number of iterations \\
            \;\; $M$ -- trained $\FWp{i}$ encoder model \\
            \;\; $\theta$ -- generator model weights \\
            \;\; $\lambda_{\mathrm{lpips}}$ -- coefficient for LPIPS loss \\
            \;\; $\lambda_{\mathrm{mse}}$ -- coefficient for MSE loss \\
            \;\; $\lambda_{\mathrm{noise}}$ -- \algparbox{0.7}{coefficient for noise regularization (from StyleGAN2)}
        }
        \vspace{0.5em}
        
        \State $l \gets M\funarg{I}$
        \State $l \gets \text{add random noise maps to $l$ to form $\FNWp{i}$}$
        \For{$N$ steps}
        \State $G \gets \text{generate an image from $l$ using weights $\theta$}$
        \State $L \gets$ \algparbox{0.5}{$\lambda_{\mathrm{lpips}} \mathrm{LPIPS}\funarg{G, I} + \lambda_{\mathrm{mse}} \mathrm{MSE}\funarg{G, I} + \lambda_{\mathrm{noise}} \mathrm{NoiseReg}\funarg{\text{get noise maps from $l$}}$}
        \State $\text{optimize latent representation $l$ using loss $L$}$
        \State $\text{normalize noise maps from $l$ following the procedure from StyleGAN2 \cite{Karras2019stylegan2}}$
        \EndFor
        \State \Return $l$
    \end{algorithmic}
\end{algorithm}

\section{Experiments}
The $\FWp{5}$ encoder model (Algorithm~\ref{alg:projection-model}) is trained starting from a pre-trained EfficientNet-B0 \cite{DBLP:journals/corr/abs-1905-11946} as EfficientNet family is considered one of the best convolutional backbone for various tasks like image classification or object detection. We also tried to use heavier EfficientNet-B3, however we didn't notice any improvement.
We use $\lambda = \lambda_{\mathrm{lpips}} = \lambda_{\mathrm{mse}} = 1$, Adam optimizer with learning rate of $2\text{e-}4$, $\beta_1 = 0.9$, $\beta_2 = 0.999$ and batch size of 12. Similarly to attribute model, we downscale images 4 times for LPIPS loss ($s = 0.25$).
For optimizing spatial latent representations (Algorithm~\ref{alg:latent-fitting}), we use $\lambda_{\mathrm{lpips}} = 1$, $\lambda_{\mathrm{mse}} = 0.25$, $\lambda_{\mathrm{noise}} = 4\text{e}5$,
Adam optimizer with learning rate of $0.1$, $\beta_1 = 0.9$, $\beta_2 = 0.999$. We optimize for $N = 250$ iterations, using linear learning rate warmup for 50 iterations.
For optimizing $\NWp$ representations we use the same hyperparameters, but optimize for $N = 1000$ iterations.
We show importance of shorter training for spatial representations and use of the encoder model output for initialization in Figure~\ref{fig:fitting-init-and-len}.

Figure \ref{fig:projection-simple-FWp} shows the quality of projected images into $\FNWp{}$ is much better compared to $\NWp$.
Further, Figure~\ref{fig:edit-FWp} shows that we can further
use our attribute models on projected images preserving the projection quality.

In Figure~\ref{fig:projection-equivariance-FWp} we show that our method is equivariant to translations.
Figure \ref{fig:projection-interp-FWp} shows spatial interpolation between projected images in the $\FNSWp{}$ space.

We further show that our projection method is able to project images that are out-of-sample in terms of the number and arrangement of object parts into $\FNWp{}$ space, and that such latent representations can be edited in a meaningful way.
In Figure~\ref{fig:oos1-FWp}, we project and edit a face with only one eye, and in Figure~\ref{fig:oos2-FWp} -- an image with four people.

\begin{figure}
\begin{minipage}[t][\textheight][t]{\textwidth}
    \vspace{\fill}
    \setlength{\tabcolsep}{0pt}
    \renewcommand{\arraystretch}{0}
    \centering
    \begin{tabular}{cccc}
        Original & Projection into $\NWp$ & \parbox[c]{0.25\textwidth}{\centering Projection into $\FNWp{5}$ trained for 250 iterations} &
        \parbox[c]{0.25\textwidth}{\centering Projection into $\FNWp{5}$ trained till convergence} \\
        \vspace{0.5em} \\
          
        \vspace{0.5em} \\
        & LPIPS: 0.2458 & LPIPS: 0.2076 & LPIPS: 0.1614  \\
        \vspace{0.2em} \\
        \includegraphics[align=c,width=0.25\textwidth]{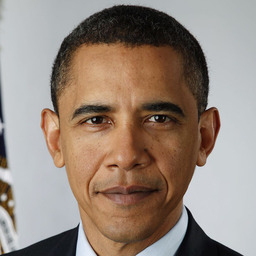} &
        \includegraphics[align=c,width=0.25\textwidth]{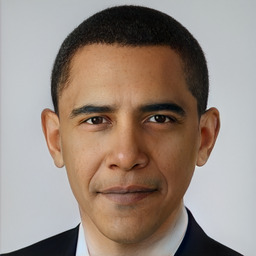} &
        \includegraphics[align=c,width=0.25\textwidth]{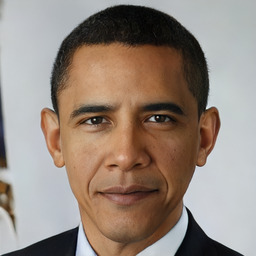} &
        \includegraphics[align=c,width=0.25\textwidth]{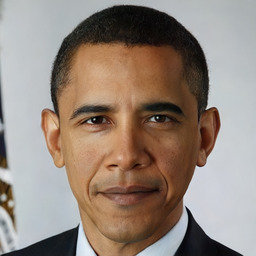} \\
        
        \vspace{0.5em} \\
        & LPIPS: 0.3330 & LPIPS: 0.2985 & LPIPS: 0.2430 \\
        \vspace{0.2em} \\
        \includegraphics[align=c,width=0.25\textwidth]{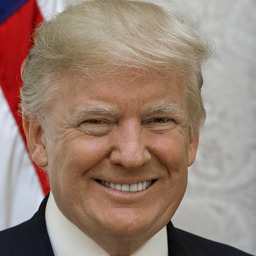} &
        \includegraphics[align=c,width=0.25\textwidth]{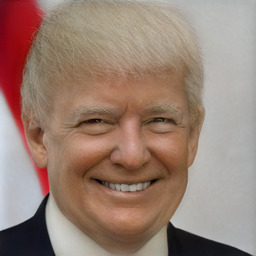} &
        \includegraphics[align=c,width=0.25\textwidth]{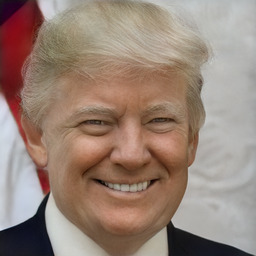} &
        \includegraphics[align=c,width=0.25\textwidth]{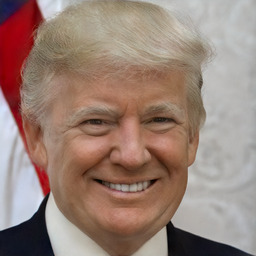} \\
        
        \vspace{0.5em} \\
        & LPIPS: 0.3137 & LPIPS: 0.2786 & LPIPS: 0.2211 \\
        \vspace{0.2em} \\
        \includegraphics[align=c,width=0.25\textwidth]{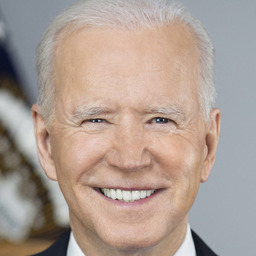} &
        \includegraphics[align=c,width=0.25\textwidth]{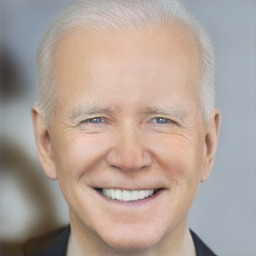} &
        \includegraphics[align=c,width=0.25\textwidth]{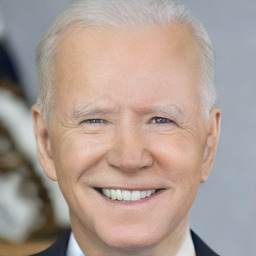} &
        \includegraphics[align=c,width=0.25\textwidth]{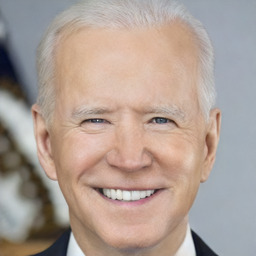}
    \end{tabular}
    \caption{
        We project images\protect\footnote{
            Images were taken from
            \url{https://commons.wikimedia.org/wiki/File:Official_portrait_of_Barack_Obama.jpg},
            \url{https://www.whitehouse.gov/wp-content/uploads/2021/01/45_donald_trump.jpg},
            \url{https://www.whitehouse.gov/wp-content/uploads/2021/04/P20210303AS-1901.jpg}.
        }
        to $\NWp$ and $\FNWp{5}$.
        Images projected into the $\FNWp{5}$ space are able to capture much more details including the background,
        look more realistic and have 10\%--15\% improvement in terms of LPIPS score for projection trained for 250 iterations and 25--35\% for projections trained till convergence.
    }
    \label{fig:projection-simple-FWp}
    \vspace{\fill}
\end{minipage}
\end{figure}

\begin{figure}
    \setlength{\tabcolsep}{0pt}
    \renewcommand{\arraystretch}{0}
    \centering
    \begin{tabular}{lcccp{0.025\textwidth}ccc}
        &
        Original &
        \parbox[t]{0.15\textwidth}{\centering Projection \\ into $\NWp$} &
        \parbox[t]{0.15\textwidth}{\centering Projection \\ into $\FNWp{5}$} & &
        Original &
        \parbox[t]{0.15\textwidth}{\centering Projection \\ into $\NWp$} &
        \parbox[t]{0.15\textwidth}{\centering Projection \\ into $\FNWp{5}$} \\
        
        \rotatebox[]{90}{no offset}\; &
        \includegraphics[align=c,width=0.15\textwidth]{samples/basic_fit__obama_gt} &
        \includegraphics[align=c,width=0.15\textwidth]{samples/basic_fit__obama_NWp.jpg} &
        \includegraphics[align=c,width=0.15\textwidth]{samples/basic_fit__obama_F5NWp.jpg} & &
        \includegraphics[align=c,width=0.15\textwidth]{samples/basic_fit__trump_gt} &
        \includegraphics[align=c,width=0.15\textwidth]{samples/basic_fit__trump_NWp.jpg} &
        \includegraphics[align=c,width=0.15\textwidth]{samples/basic_fit__trump_F5NWp.jpg} \\
        
        \rotatebox[]{90}{50px offset}\; &
        \includegraphics[align=c,width=0.15\textwidth]{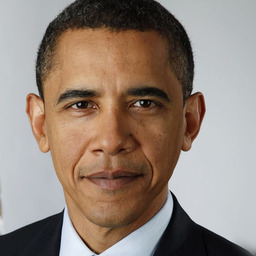} &
        \includegraphics[align=c,width=0.15\textwidth]{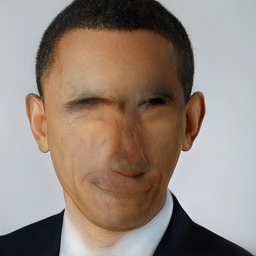} &
        \includegraphics[align=c,width=0.15\textwidth]{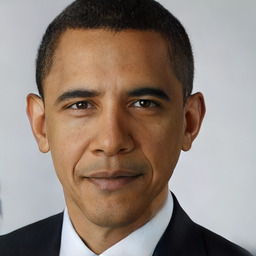} & &
        \includegraphics[align=c,width=0.15\textwidth]{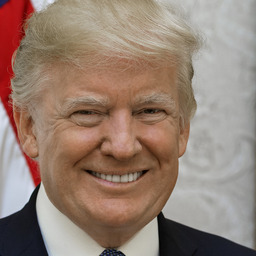} &
        \includegraphics[align=c,width=0.15\textwidth]{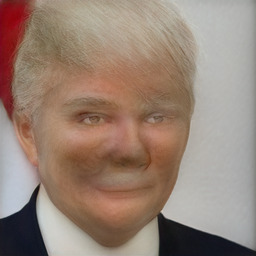} &
        \includegraphics[align=c,width=0.15\textwidth]{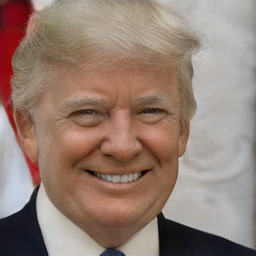} \\
        
        \rotatebox[]{90}{100px offset}\; &
        \includegraphics[align=c,width=0.15\textwidth]{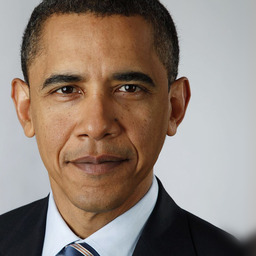} &
        \includegraphics[align=c,width=0.15\textwidth]{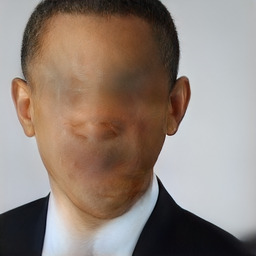} &
        \includegraphics[align=c,width=0.15\textwidth]{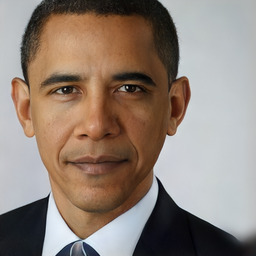} & &
        \includegraphics[align=c,width=0.15\textwidth]{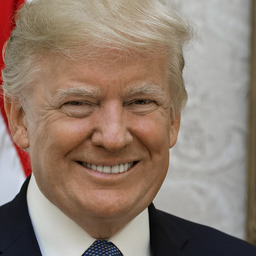} &
        \includegraphics[align=c,width=0.15\textwidth]{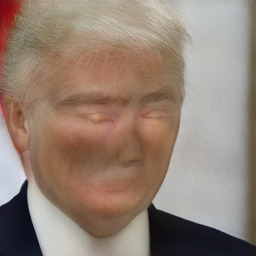} &
        \includegraphics[align=c,width=0.15\textwidth]{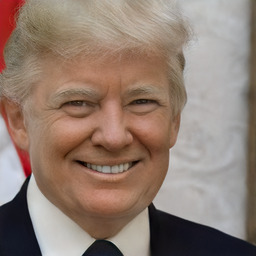} \\
        
    \end{tabular}
    \caption{
        We project spatially offset images by 50px and 100px into $\NWp$ and $\FNWp{5}$. $\FNWp{}$ and our projection algorithm show equivariance to translations.
        Non-spatial latent spaces are very sensitive to translations and even a slight translation can produce unrealistic images.
    }
    \label{fig:projection-equivariance-FWp}
\end{figure}

\newcommand{\customrow}[9]{
    \rotatebox[]{90}{\parbox[c]{0.135\textwidth}{\centering #1}}\; &
    \includegraphics[align=c,width=0.135\textwidth]{samples/basic_fit__#2_gt} &
    \includegraphics[align=c,width=0.135\textwidth]{samples/edit__attr__#2_#3__1.jpg} &
    \includegraphics[align=c,width=0.135\textwidth]{samples/edit__attr__#2_#3__#9age=#4.jpg} &
    \includegraphics[align=c,width=0.135\textwidth]{samples/edit__attr__#2_#3__#9gender=#5.jpg} &
    \includegraphics[align=c,width=0.135\textwidth]{samples/edit__attr__#2_#3__#9glasses=#6.jpg} &
    \includegraphics[align=c,width=0.135\textwidth]{samples/edit__attr__#2_#3__#9smile=#7.jpg} &
    \includegraphics[align=c,width=0.135\textwidth]{samples/edit__attr__#2_#3__#9eyes_open=#8.jpg} \\
}

\begin{figure}
    \setlength{\tabcolsep}{0pt}
    \renewcommand{\arraystretch}{0}
    \centering

    \begin{tabular}{lccccccc}
        & Original & Projection & $\pm$ age & $\pm$ gender & $\pm$ glasses & $\pm$ smile & $\pm$ eyes open \\
        
        \customrow{$\NWp$}{obama}{NWp}{700}{-400}{-700}{-600}{600}{offset=50_slope10.}
        \customrow{$\FNWp{5}$}{obama}{F5NWp}{700}{-400}{-700}{-600}{600}{offset=50_slope10.}
        \customrow{$\NWp$}{trump}{NWp}{700}{-400}{-1000}{400}{-600}{}
        \customrow{$\FNWp{5}$}{trump}{F5NWp}{700}{-400}{-1000}{400}{-600}{}
        \customrow{$\NWp$}{biden}{NWp}{800}{-600}{-600}{400}{-800}{}
        \customrow{$\FNWp{5}$}{biden}{F5NWp}{800}{-600}{-600}{400}{-800}{}
    \end{tabular}
    \caption{
        We project images to $\NWp$ and $\FNWp{5}$.
        Images projected into the $\FNWp{5}$ space are able to capture much more details including the background
        and look more realistic.
        We apply our attribute model on $\FNWp{5}$ projections and compare results to projections into $\NWp$ translated by an attribute direction vector.
        For one image we use the same strength of attribute edition for both projections.
        Changes in $\FNWp{5}$ are comparable to these in the $\NWp$ space, while preserving details captured in the $\FNWp{5}$ latent space.
    }
    \label{fig:edit-FWp}
\end{figure}

\begin{figure}
    \setlength{\tabcolsep}{0pt}
    \centering
    
    \def\ssize{0.125}
    \begin{tabular}{cccccccc}
            Source 1 & \multicolumn{6}{c}{$\xrightarrow{\makebox[0.7\linewidth]{}}$} & Source 2 \\
            \includegraphics[width=\ssize\textwidth]{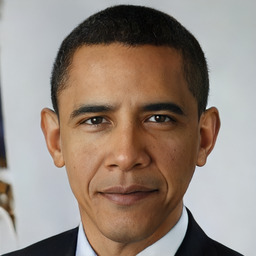} &
            \includegraphics[width=\ssize\textwidth]{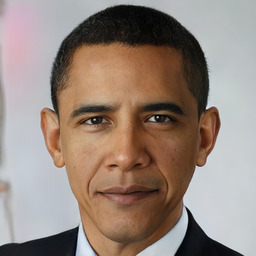} &
            \includegraphics[width=\ssize\textwidth]{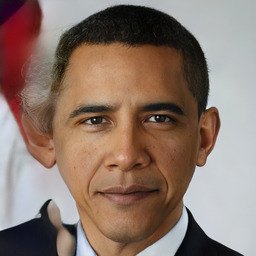} &
            \includegraphics[width=\ssize\textwidth]{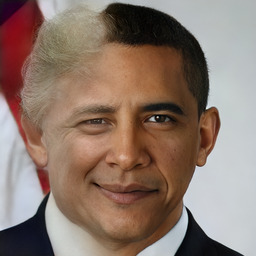} &
            \includegraphics[width=\ssize\textwidth]{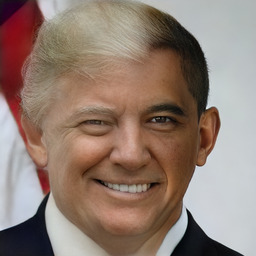} &
            \includegraphics[width=\ssize\textwidth]{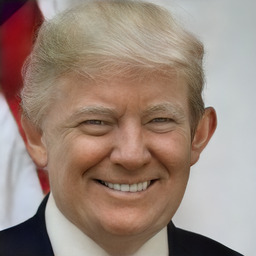} &
            \includegraphics[width=\ssize\textwidth]{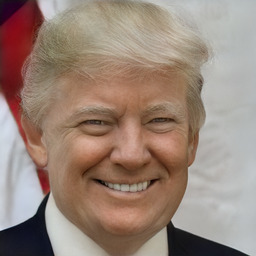} &
            \includegraphics[width=\ssize\textwidth]{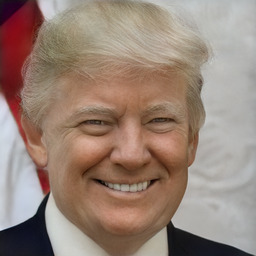} \\
            
            \includegraphics[width=\ssize\textwidth]{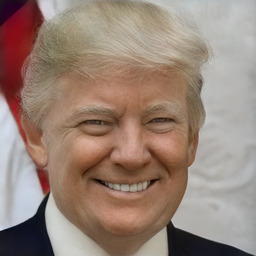} &
            \includegraphics[width=\ssize\textwidth]{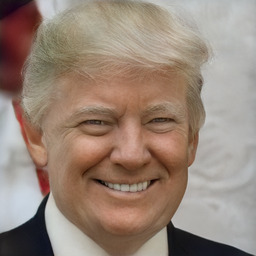} &
            \includegraphics[width=\ssize\textwidth]{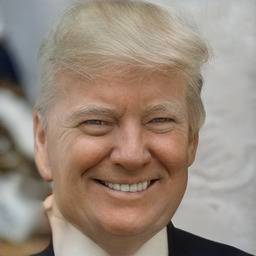} &
            \includegraphics[width=\ssize\textwidth]{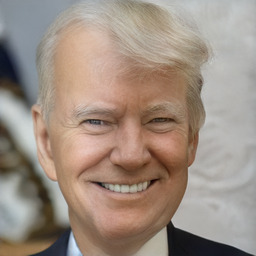} &
            \includegraphics[width=\ssize\textwidth]{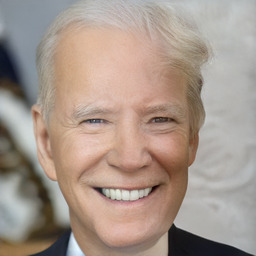} &
            \includegraphics[width=\ssize\textwidth]{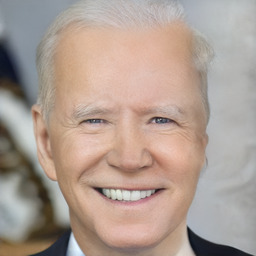} &
            \includegraphics[width=\ssize\textwidth]{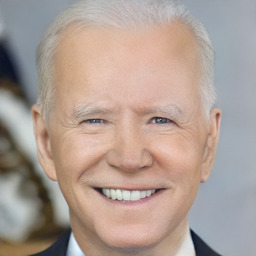} &
            \includegraphics[width=\ssize\textwidth]{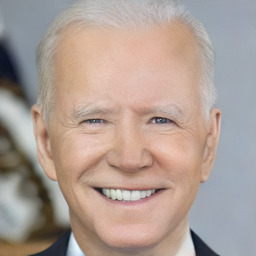}
    \end{tabular}
    \caption{
        Spatial interpolation between projections in $\FNSWp{5}$ space.
    }
    \label{fig:projection-interp-FWp}
\end{figure}

\begin{figure}
    \setlength{\tabcolsep}{0pt}
    \centering
    
    \begin{tabular}{ccccc}
        Original & \parbox[c]{0.2\textwidth}{\centering \small $\FWp{5}$ encoder model projection} & \parbox[c]{0.2\textwidth}{\centering \small $\FNWp{5}$ projection with optimization} & + Age & - Age \\
        \includegraphics[align=c,width=0.2\textwidth]{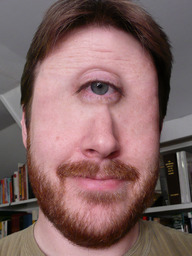} &
        \includegraphics[align=c,width=0.2\textwidth]{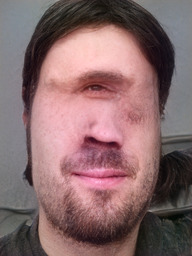} &
        \includegraphics[align=c,width=0.2\textwidth]{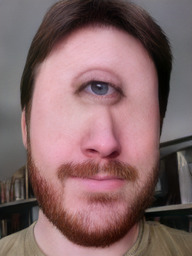} &
        \includegraphics[align=c,width=0.2\textwidth]{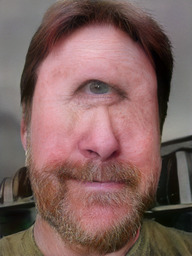} &
        \includegraphics[align=c,width=0.2\textwidth]{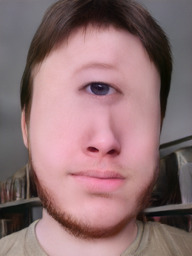} \\
        
        Gender & + Smile & \small - Horizontal angle & \small + Horizontal angle & Close eye \\
        \includegraphics[align=c,width=0.2\textwidth]{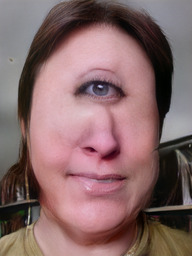} &
        \includegraphics[align=c,width=0.2\textwidth]{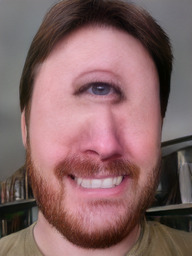} &
        \includegraphics[align=c,width=0.2\textwidth]{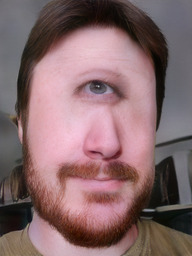} &
        \includegraphics[align=c,width=0.2\textwidth]{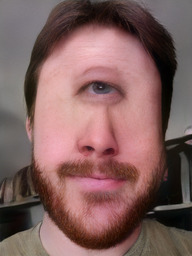} &
        \includegraphics[align=c,width=0.2\textwidth]{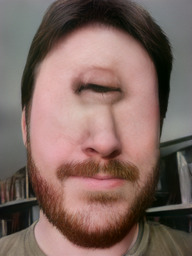} \\
    
    \end{tabular}
    \caption{
        The original image was taken from Flickr \cite{flickr:one-eye}.
        Our projection method into $\FNWp{5}$ latent space is able to semantically meaningfully encode an unnatural face with only one eye.
        Note that the input and output image is not square, which can be dealt with using our methods, as nothing constraints
        feature and RGB maps to be square.
    }
    \label{fig:oos1-FWp}
\end{figure}

\begin{figure}
\begin{minipage}[t][\textheight][t]{\textwidth}
    \setcounter{mpfootnote}{1} 
    \vspace{\fill}
    \setlength{\tabcolsep}{0pt}
    \centering
    
    \begin{tabular}{cc}
        Original & $\FNWp{5}$ Projection \\
        \includegraphics[align=c,width=0.5\textwidth]{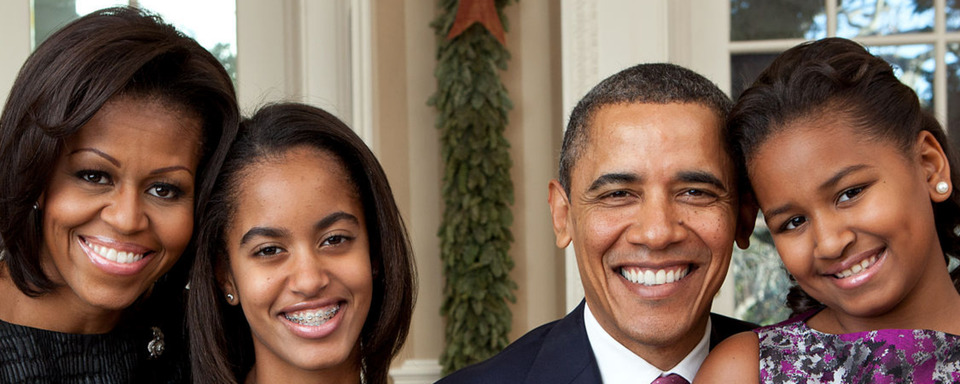} &
        \includegraphics[align=c,width=0.5\textwidth]{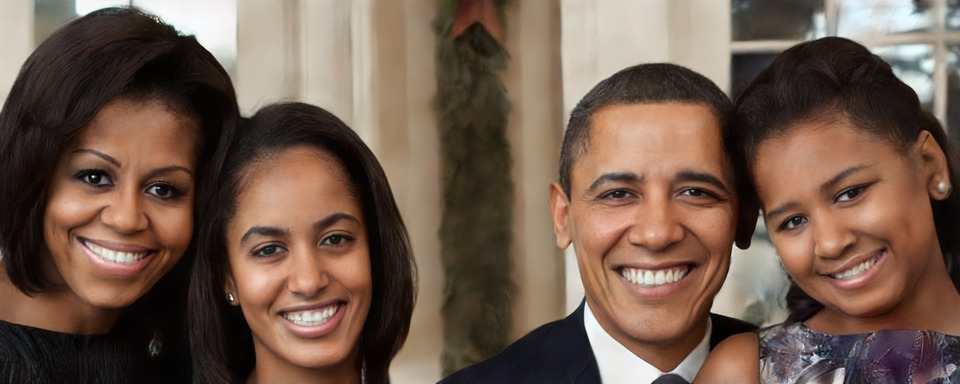} \\
        
        + Age & - Age  \\
        \includegraphics[align=c,width=0.5\textwidth]{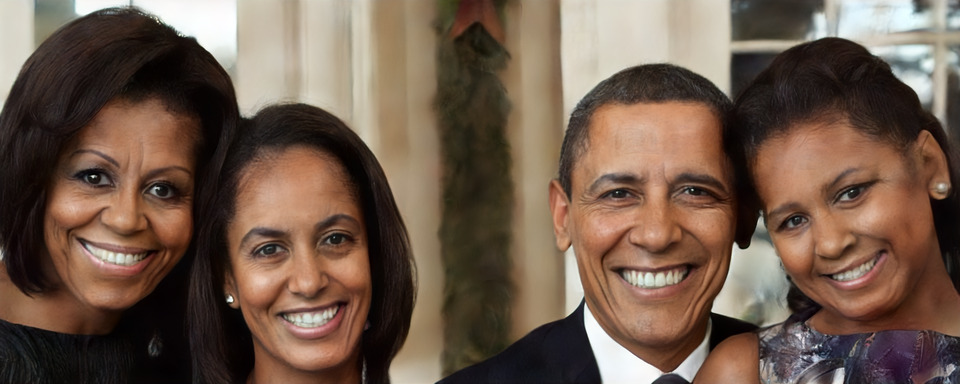} &
        \includegraphics[align=c,width=0.5\textwidth]{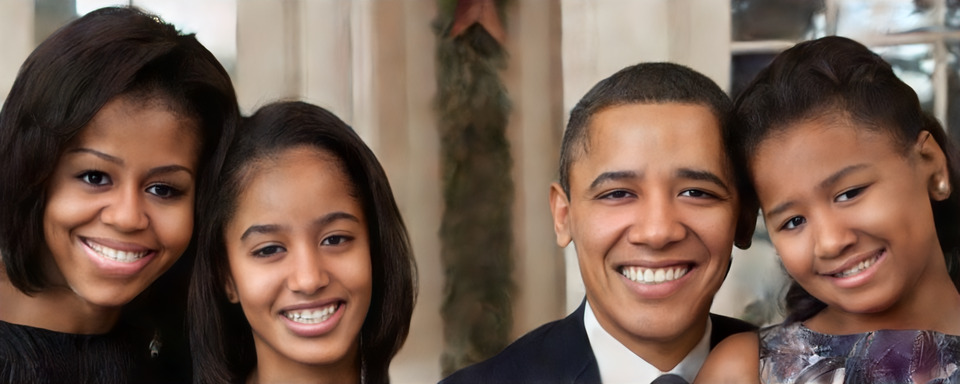} \\
        
        Eyes closed & - Smile  \\
        \includegraphics[align=c,width=0.5\textwidth]{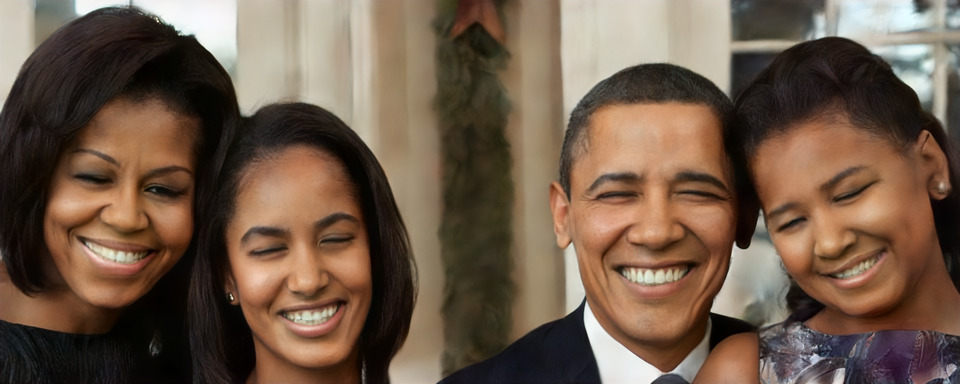} &
        \includegraphics[align=c,width=0.5\textwidth]{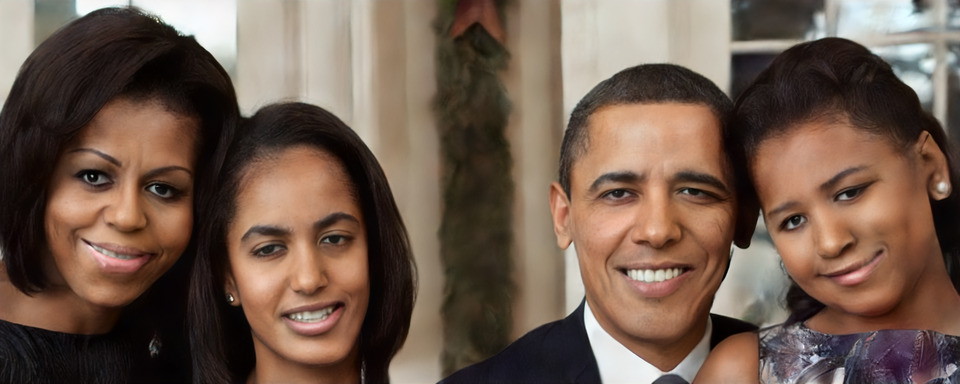} \\
        
        + Glasses \\
        \includegraphics[align=c,width=0.5\textwidth]{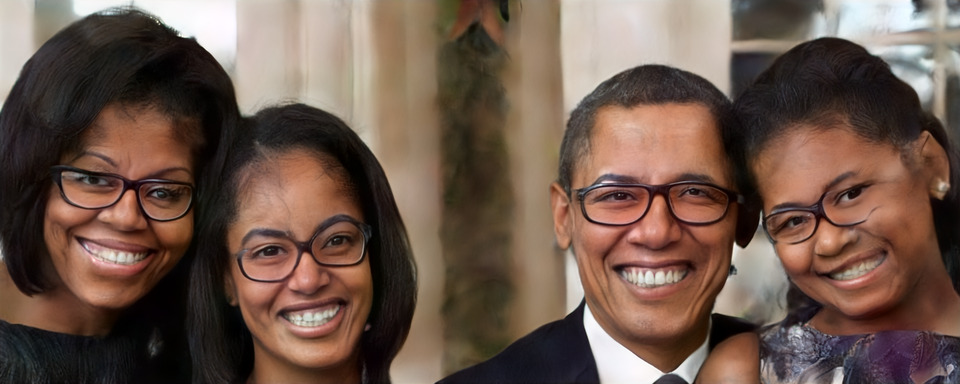} \\
    
    \end{tabular}
    \caption{
        We use our methods to project an image\protect\footnote{
            The image was taken from \url{https://en.wikipedia.org/wiki/First_family_of_the_United_States\#/media/File:Barack_Obama_family_portrait_2011.jpg}
        }
        with multiple people into one $\FNWp{5}$ latent representation.
        Note that the input and output images are high resolution, in a way that one face roughly corresponds
        to the size that the generator normally generates.
        The projection is able to capture semantic information about all people on the image
        and we can apply our attribute model to edit attributes of all people at the same time.
    }
    \label{fig:oos2-FWp}
    \vspace{\fill}
\end{minipage}
\end{figure}

\begin{figure}
    \setlength{\tabcolsep}{0pt}
    \centering
    
    \def\ssize{0.175}
    
    \begin{tabular}{lcccc}
    
        & \small Original & \scriptsize - Horizontal angle & \scriptsize + Horizontal angle & \small $\pm$ Glasses \\
        
        \rotatebox[]{90}{\centering \scriptsize Edition in $\NWp$}\; &
        \includegraphics[align=c,width=\ssize\textwidth]{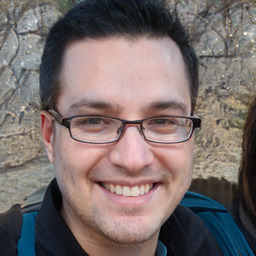} &
        \includegraphics[align=c,width=\ssize\textwidth]{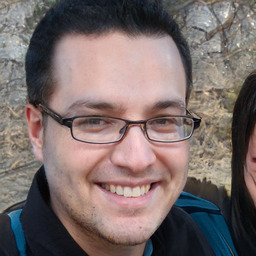} &
        \includegraphics[align=c,width=\ssize\textwidth]{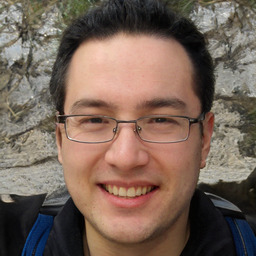} &
        \includegraphics[align=c,width=\ssize\textwidth]{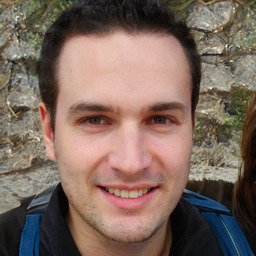} \\
        
        \rotatebox[]{90}{\parbox[c]{\ssize\textwidth}{\centering \scriptsize Edition in $\FNWp{5}$ on style vectors only}}\; & &
        \includegraphics[align=c,width=\ssize\textwidth]{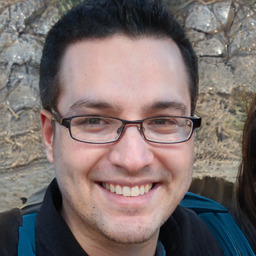} &
        \includegraphics[align=c,width=\ssize\textwidth]{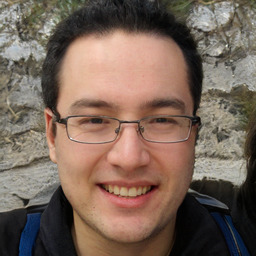} &
        \includegraphics[align=c,width=\ssize\textwidth]{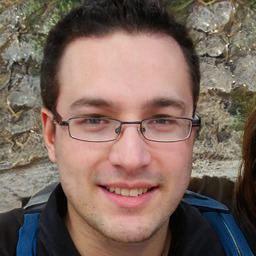} \\
        
        \rotatebox[]{90}{\parbox[c]{\ssize\textwidth}{\centering \scriptsize Edition in $\FNWp{5}$ using attribute model}}\; & &
        \includegraphics[align=c,width=\ssize\textwidth]{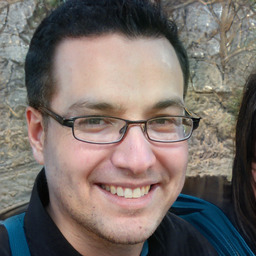} &
        \includegraphics[align=c,width=\ssize\textwidth]{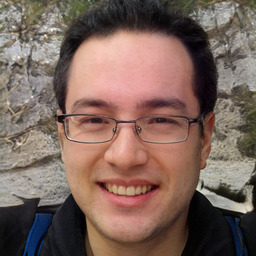} &
        \includegraphics[align=c,width=\ssize\textwidth]{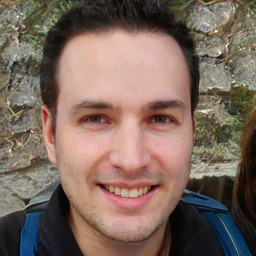} \\
        
        \vspace{0.3em}\\
       
        \rotatebox[]{90}{\parbox[c]{\ssize\textwidth}{\centering \scriptsize Edition in $\NWp$}}\; &
        \includegraphics[align=c,width=\ssize\textwidth]{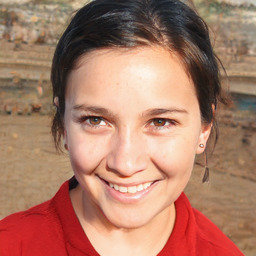} &
        \includegraphics[align=c,width=\ssize\textwidth]{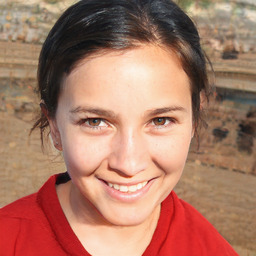} &
        \includegraphics[align=c,width=\ssize\textwidth]{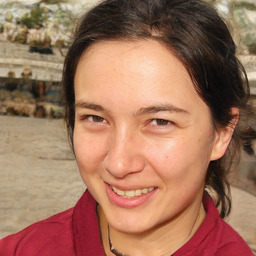} &
        \includegraphics[align=c,width=\ssize\textwidth]{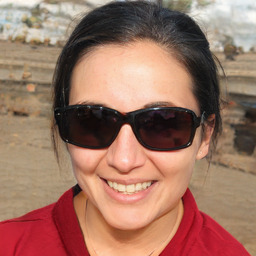} \\
        
        \rotatebox[]{90}{\parbox[c]{\ssize\textwidth}{\centering \scriptsize Edition in $\FNWp{5}$ on style vectors only}}\; & &
        \includegraphics[align=c,width=\ssize\textwidth]{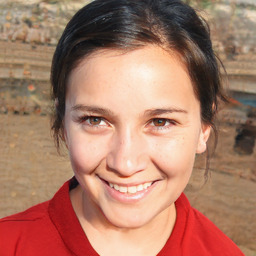} &
        \includegraphics[align=c,width=\ssize\textwidth]{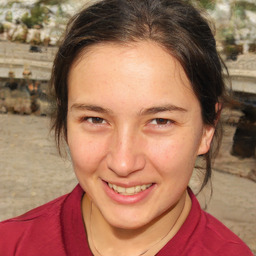} &
        \includegraphics[align=c,width=\ssize\textwidth]{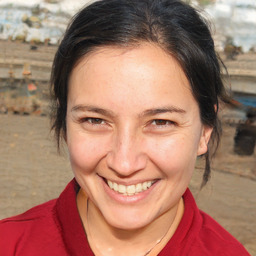} \\
        
        \rotatebox[]{90}{\parbox[c]{\ssize\textwidth}{\centering \scriptsize Edition in $\FNWp{5}$ using attribute model}}\; & &
        \includegraphics[align=c,width=\ssize\textwidth]{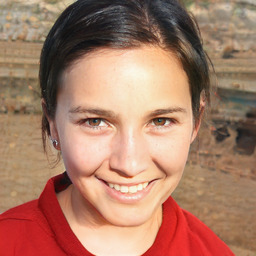} &
        \includegraphics[align=c,width=\ssize\textwidth]{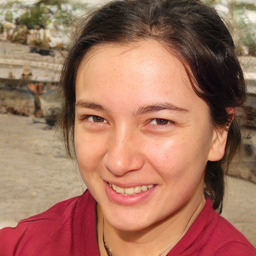} &
        \includegraphics[align=c,width=\ssize\textwidth]{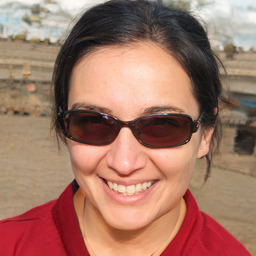} \\
    
    \end{tabular}
    \caption{
        To evaluate our attribution model we compare its performance to attribute edition in the $\NWp$ space.
        We generate the $\NWp$ latent representation and use attribute vectors (for changing horizontal angle and adding/removing glasses)
        to all style vectors in $\NWp$ to obtain a baseline (first row).
        We convert the latent representation into the $\FNWp{5}$ space for the second and the third row.
        In the second row we don't change feature and RGB maps, but only translate style vectors present in the $\FNWp{5}$ space.
        In the third row we use our attribute model to edit feature and RGB maps, together with translating style vectors.
        We see that the use of attribute models is necessary to match the expected results.
        For coarse styles, edition in $\FNWp{5}$ without the use of attribute models, results only in unwanted off-target effects coming from the attribute vector,
        instead of changes in a specified attribute.
    }
    \label{fig:attribute-model-effect}
\end{figure}

\begin{figure}
\begin{minipage}[t][\textheight][t]{\textwidth}
    \setcounter{mpfootnote}{2} 
    \vspace{\fill}
    \renewcommand{\arraystretch}{0}
    \centering
    
    \def\ssize{0.13}
    
    \begin{tabular}{lcccc}
    
        & &
        \parbox[c]{\ssize\textwidth}{\centering \scriptsize Original\protect\footnote{The image was taken from \url{https://en.wikipedia.org/wiki/Barack_Obama\#/media/File:Obama_family_portrait_in_the_Green_Room.jpg}}} &
        \parbox[c]{\ssize\textwidth}{\centering \scriptsize Encoder model output} \\
        
        \vspace{0.25em} \\
        & &
        \includegraphics[align=c,height=\ssize\textwidth]{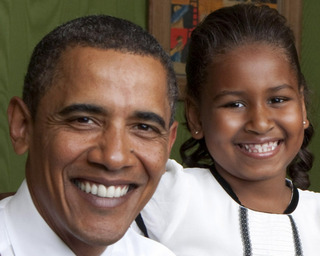} &
        \includegraphics[align=c,height=\ssize\textwidth]{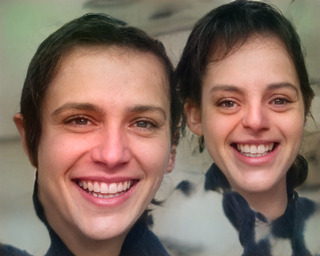} \\
        \vspace{0.75em} \\
        
        &
        \parbox[c]{\ssize\textwidth}{\centering \scriptsize Latent optimization starting from encoder output trained for 250 iterations} &
        \parbox[c]{\ssize\textwidth}{\centering \scriptsize Latent optimization starting from encoder output trained till convergence} &
        \parbox[c]{\ssize\textwidth}{\centering \scriptsize Latent optimization starting from random initialization trained for 250 iterations} &
        \parbox[c]{\ssize\textwidth}{\centering \scriptsize Latent optimization starting from random initialization trained till convergence} \\
        
        \vspace{0.4em} \\
        & \small LPIPS: 0.2428 & \small LPIPS: 0.1776 & \small LPIPS: 0.2972 & \small LPIPS: 0.1987 \\
        \vspace{0.2em} \\
        
        \rotatebox[]{90}{\parbox[c]{\ssize\textwidth}{\centering \small $\FNWp{5}$ projection }}\; &
        \includegraphics[align=c,height=\ssize\textwidth]{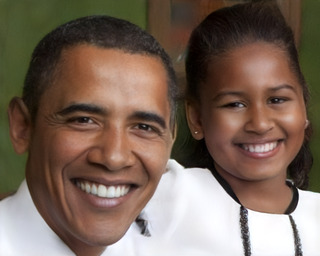} &
        \includegraphics[align=c,height=\ssize\textwidth]{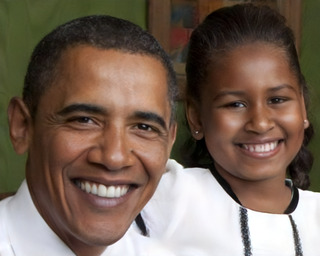} &
        \includegraphics[align=c,height=\ssize\textwidth]{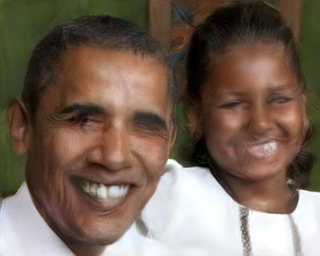} &
        \includegraphics[align=c,height=\ssize\textwidth]{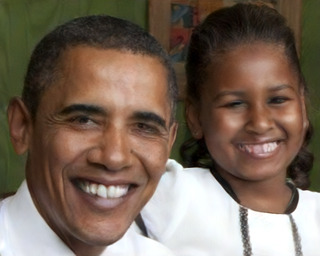} \\
        
        \rotatebox[]{90}{\parbox[c]{\ssize\textwidth}{\centering \small - Smile }}\; &
        \includegraphics[align=c,height=\ssize\textwidth]{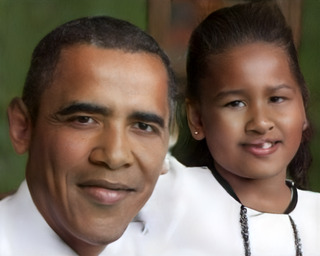} &
        \includegraphics[align=c,height=\ssize\textwidth]{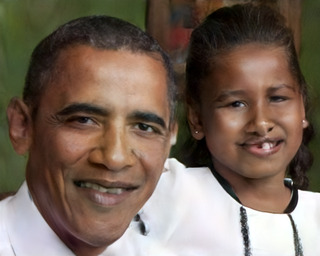} &
        \includegraphics[align=c,height=\ssize\textwidth]{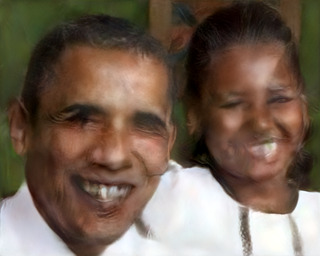} &
        \includegraphics[align=c,height=\ssize\textwidth]{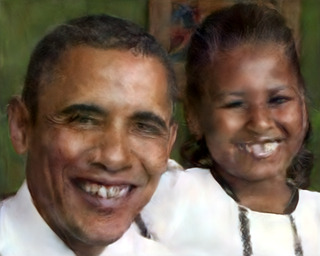} \\
        
        \rotatebox[]{90}{\parbox[c]{\ssize\textwidth}{\centering \small + Eyes open }}\; &
        \includegraphics[align=c,height=\ssize\textwidth]{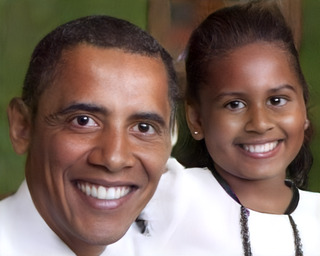} &
        \includegraphics[align=c,height=\ssize\textwidth]{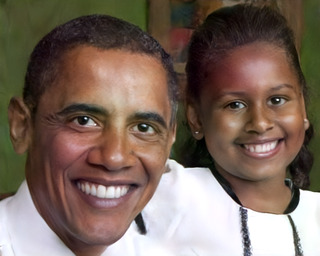} &
        \includegraphics[align=c,height=\ssize\textwidth]{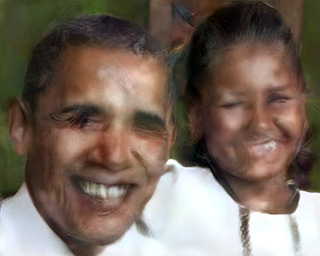} &
        \includegraphics[align=c,height=\ssize\textwidth]{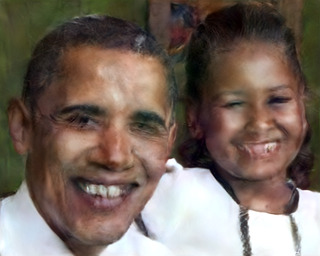} \\
        
        \rotatebox[]{90}{\parbox[c]{\ssize\textwidth}{\centering \small + Age }}\; &
        \includegraphics[align=c,height=\ssize\textwidth]{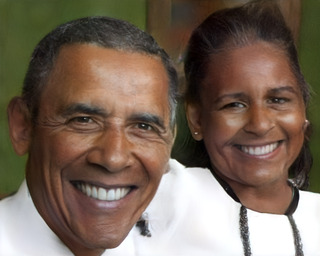} &
        \includegraphics[align=c,height=\ssize\textwidth]{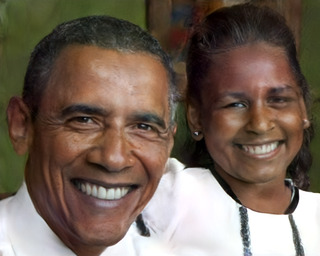} &
        \includegraphics[align=c,height=\ssize\textwidth]{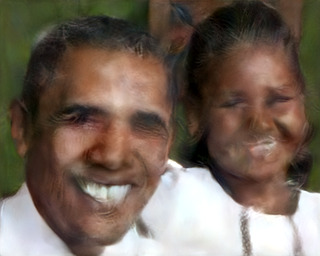} &
        \includegraphics[align=c,height=\ssize\textwidth]{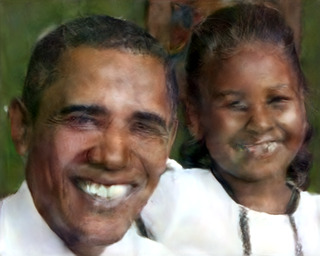} \\
        
        \rotatebox[]{90}{\parbox[c]{\ssize\textwidth}{\centering \small Gender }}\; &
        \includegraphics[align=c,height=\ssize\textwidth]{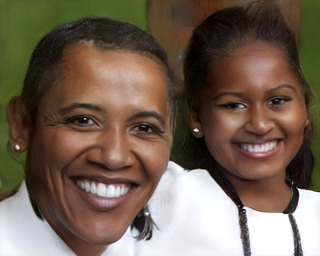} &
        \includegraphics[align=c,height=\ssize\textwidth]{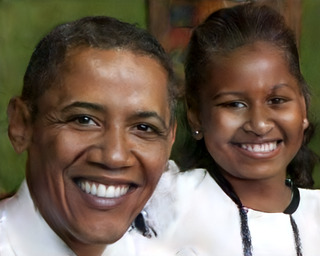} &
        \includegraphics[align=c,height=\ssize\textwidth]{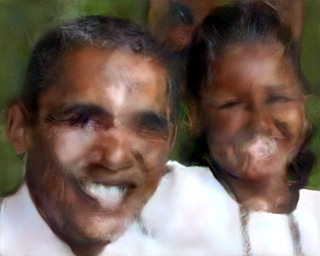} &
        \includegraphics[align=c,height=\ssize\textwidth]{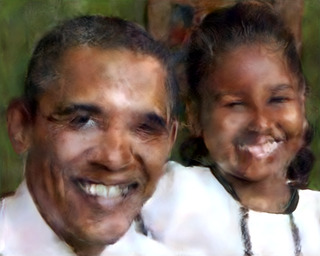} \\

    \end{tabular}
    \caption{
        We show the effect of early stopping of latent optimization after 250 iterations and importance of using
        the encoder model prediction for initialization for further latent representation optimization.
        We use $\FNWp{5}$ and for the encoder model $\FWp{5}$.
        The output from the encoder model is not similar to the original image, however
        it captures general semantics of the image, \eg face positions, face parts positions, background segmentation.
        When we use random latent representation for initialization for Algorithm~\ref{alg:latent-fitting} (3rd and 4th column),
        resulting latent representations, despite having very poor quality, also fail to capture any semantic information, making attribute editing impossible regardless how long we optimize.
        When we use the output from the encoder model for initialization (1st and 2nd column), optimization till convergence (2nd column) gives better results than optimization for 250 iterations (1st column) as much as 25\% in terms of LPIPS score, however
        attribute edition sometimes yields lower quality.
        In the rest of our experiments, we always optimize for 250 iterations because it's more suitable for image manipulation and to keep the inversion method consistent.
    }
    \label{fig:fitting-init-and-len}
    \vspace{\fill}
\end{minipage}
\end{figure}

\begin{figure}
    \setlength{\tabcolsep}{0pt}
    \renewcommand{\arraystretch}{0}
    \centering
    \begin{tabular}{cp{0.25em}ccc}
        \includegraphics[align=c,width=0.2\textwidth]{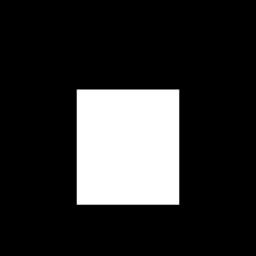} & &
        \includegraphics[align=c,width=0.2\textwidth]{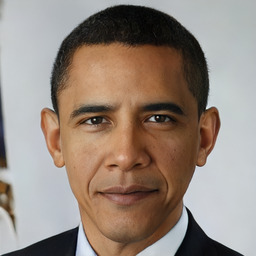} &
        \includegraphics[align=c,width=0.2\textwidth]{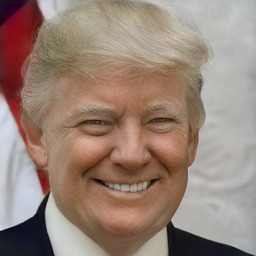} &
        \includegraphics[align=c,width=0.2\textwidth]{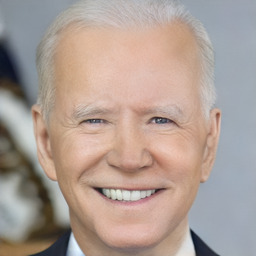} \\
        \vspace{0.25em} \\
        
        \includegraphics[align=c,width=0.2\textwidth]{samples/FSWp_mix__obama_trump__1.jpg} & &
        &
        \includegraphics[align=c,width=0.2\textwidth]{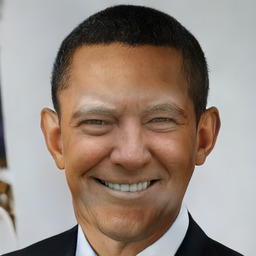} &
        \includegraphics[align=c,width=0.2\textwidth]{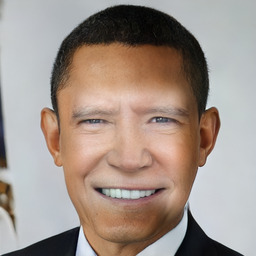} \\
        
        \includegraphics[align=c,width=0.2\textwidth]{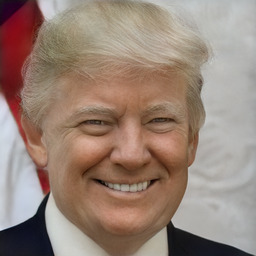} & &
        \includegraphics[align=c,width=0.2\textwidth]{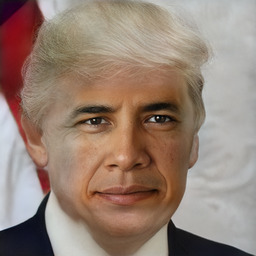} &
        &
        \includegraphics[align=c,width=0.2\textwidth]{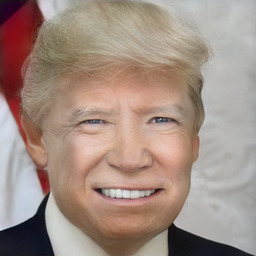} \\
        
        \includegraphics[align=c,width=0.2\textwidth]{samples/FSWp_mix__biden_trump__1.jpg} & &
        \includegraphics[align=c,width=0.2\textwidth]{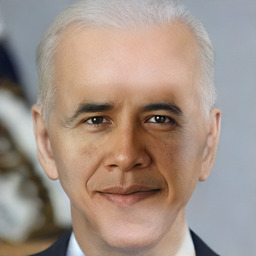} &
        \includegraphics[align=c,width=0.2\textwidth]{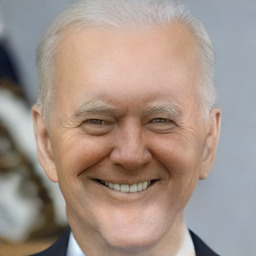}
        & \\
    \end{tabular}
    \caption{
        We take our projections into $\FNWp{5}$, convert them into $\FNSWp{5}$
        and mix them in this space, using a mask covering central part of the face.
    }
    \label{fig:face-edit-FWp}
\end{figure}

\begin{figure}
    \setlength{\tabcolsep}{0pt}
    \renewcommand{\arraystretch}{0}
    \centering
    \begin{tabular}{cp{0.25em}ccc}
        \includegraphics[align=c,width=0.2\textwidth]{samples/FSWp_mix__obama_trump__strength=0.mask.jpg} & &
        \includegraphics[align=c,width=0.2\textwidth]{samples/FSWp_mix__obama_trump__1.jpg} &
        \includegraphics[align=c,width=0.2\textwidth]{samples/FSWp_mix__trump_biden__1.jpg} &
        \includegraphics[align=c,width=0.2\textwidth]{samples/FSWp_mix__biden_trump__1.jpg} \\
        \vspace{0.25em} \\
        
        \includegraphics[align=c,width=0.2\textwidth]{samples/FSWp_mix__obama_trump__1.jpg} & &
        &
        \includegraphics[align=c,width=0.2\textwidth]{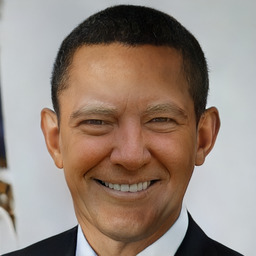} &
        \includegraphics[align=c,width=0.2\textwidth]{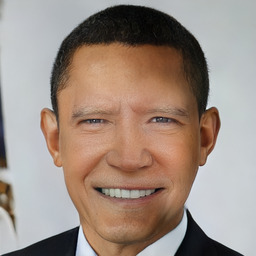} \\
        
        \includegraphics[align=c,width=0.2\textwidth]{samples/FSWp_mix__trump_obama__1.jpg} & &
        \includegraphics[align=c,width=0.2\textwidth]{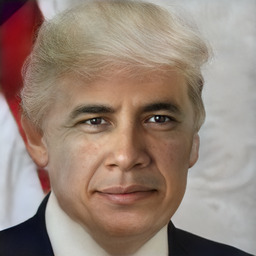} &
        &
        \includegraphics[align=c,width=0.2\textwidth]{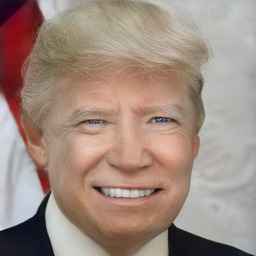} \\
        
        \includegraphics[align=c,width=0.2\textwidth]{samples/FSWp_mix__biden_trump__1.jpg} & &
        \includegraphics[align=c,width=0.2\textwidth]{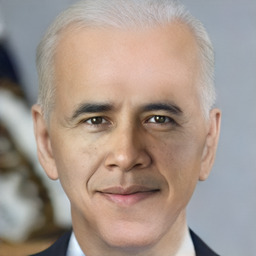} &
        \includegraphics[align=c,width=0.2\textwidth]{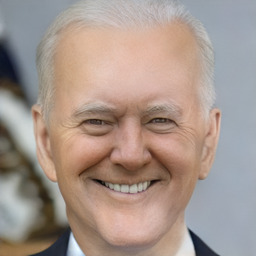}
        & \\
    \end{tabular}
    \caption{
        We follow the procedure from Figure~\ref{fig:face-edit-FWp}, but during mixing we don't change last two style vectors.
        By doing that we can preserve color scheme much better.
    }
    \label{fig:face-edit-FWp-2style}
\end{figure}

\chapter{Latent Representation Equalization}
\label{chap:equaliz}
During experimentation with attribute editing, we notice that projected representations are usually less suitable for editing purposes than generated ones.
Attribute editing on projected representations often yields less satisfying results.
In Figure~\ref{fig:heavy-attr-edit}, we show the effect of a strong attribute edition, where projected representations result in much worse quality and more visible off-target effects.
In this chapter, we aim to investigate such discrepancies.
We look at differences between distributions of latent representation obtained from $z \in \mathcal{Z}$, $z \sim \mathcal{N}\funarg{0, I}$
and latent representations of projected images, and propose methods to counteract them.

\begin{figure}[b!]
    \setlength{\tabcolsep}{0pt}
    \renewcommand{\arraystretch}{0}
    \centering

    \def\ssize{0.135}
    \begin{tabular}{cccccccc}
        & \multicolumn{3}{r}{$\xleftarrow{\makebox[0.38\linewidth]{}}$} & \tiny Horizontal angle & \multicolumn{3}{l}{$\xrightarrow{\makebox[0.38\linewidth]{}}$} \\
        \vspace{0.25em} \\

        \rotatebox[]{90}{\parbox[c]{\ssize\textwidth}{\centering \small Generated}}\; &
        \includegraphics[align=c,width=\ssize\textwidth]{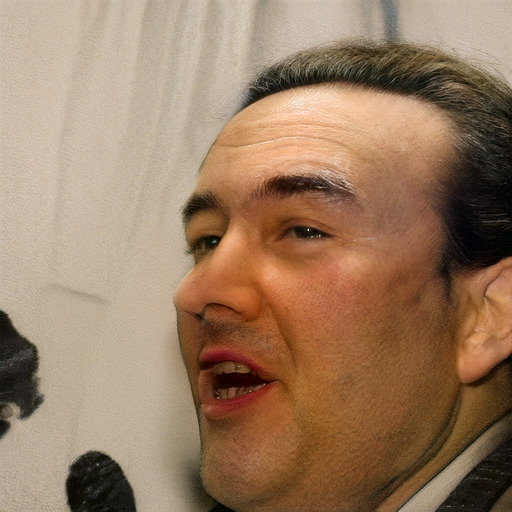} &
        \includegraphics[align=c,width=\ssize\textwidth]{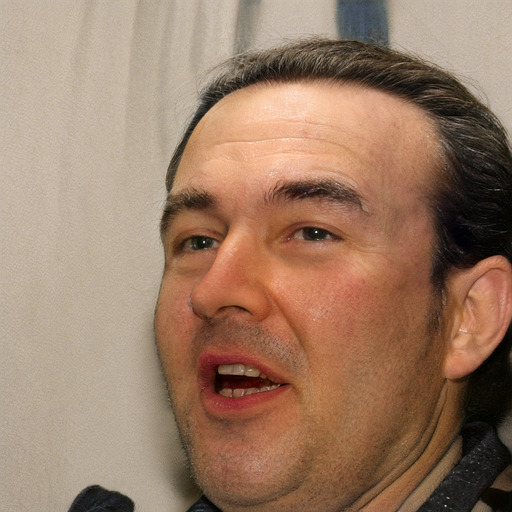} &
        \includegraphics[align=c,width=\ssize\textwidth]{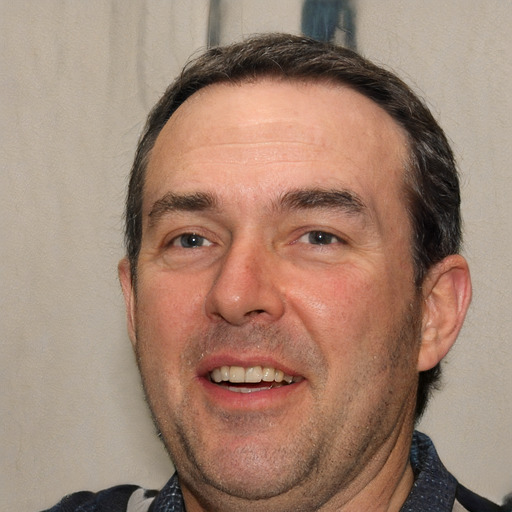} &
        \includegraphics[align=c,width=\ssize\textwidth]{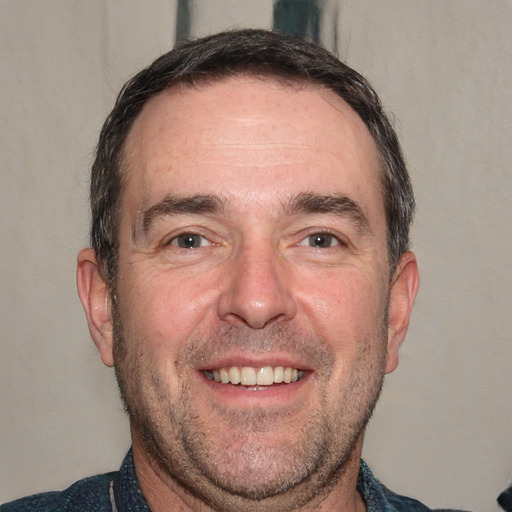} &
        \includegraphics[align=c,width=\ssize\textwidth]{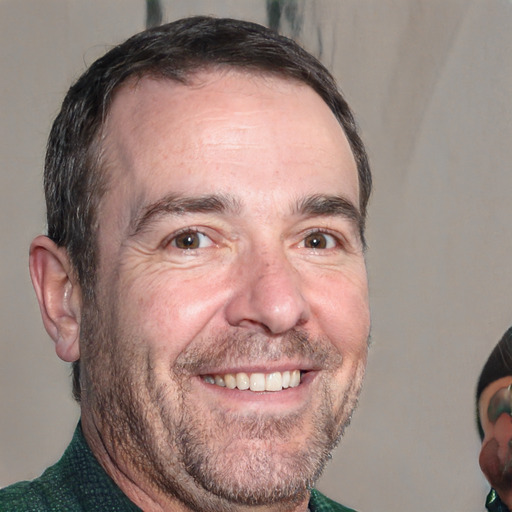} &
        \includegraphics[align=c,width=\ssize\textwidth]{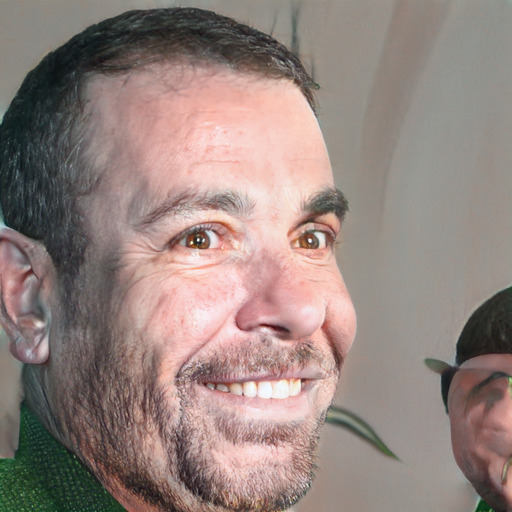} &
        \includegraphics[align=c,width=\ssize\textwidth]{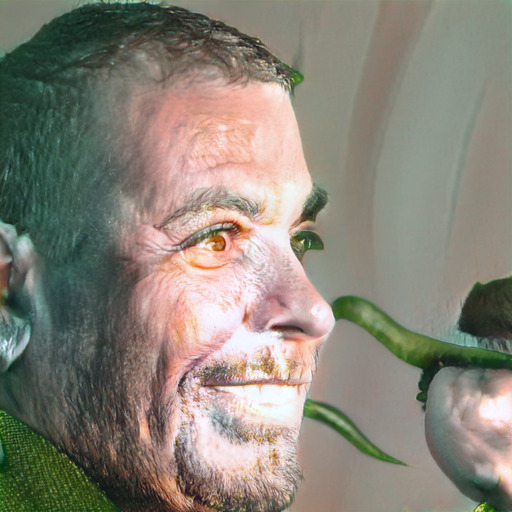} \\
        
        \rotatebox[]{90}{\parbox[c]{\ssize\textwidth}{\centering \small Projected}}\; &
        \includegraphics[align=c,width=\ssize\textwidth]{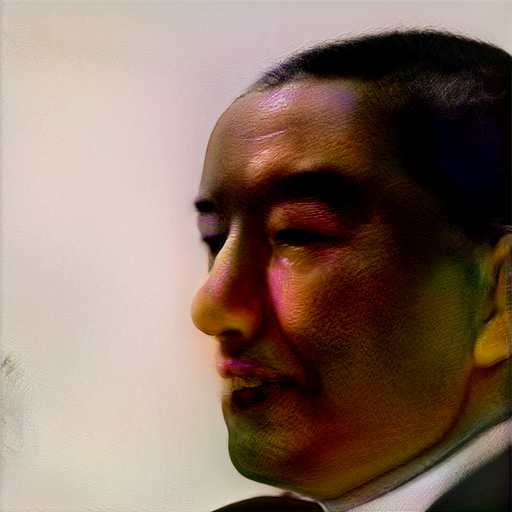} &
        \includegraphics[align=c,width=\ssize\textwidth]{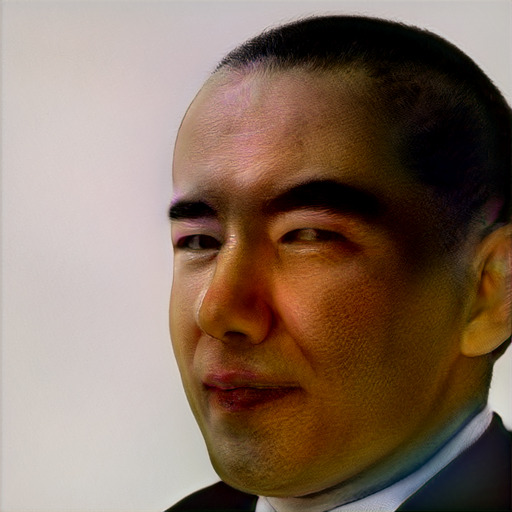} &
        \includegraphics[align=c,width=\ssize\textwidth]{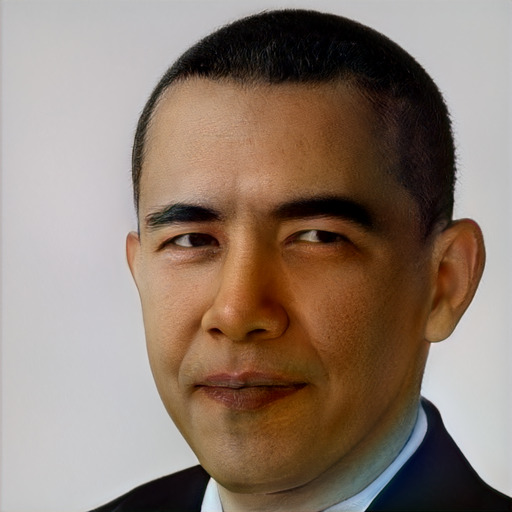} &
        \includegraphics[align=c,width=\ssize\textwidth]{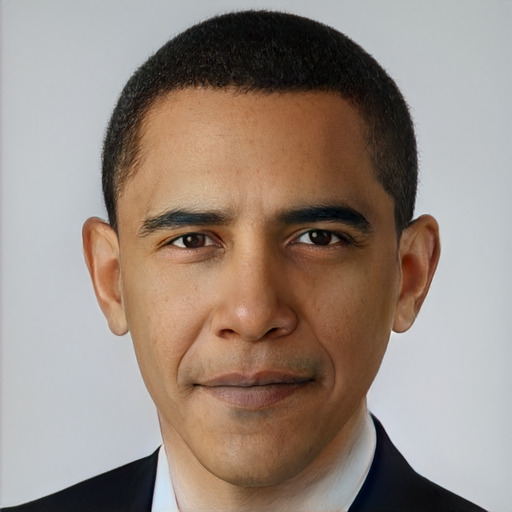} &
        \includegraphics[align=c,width=\ssize\textwidth]{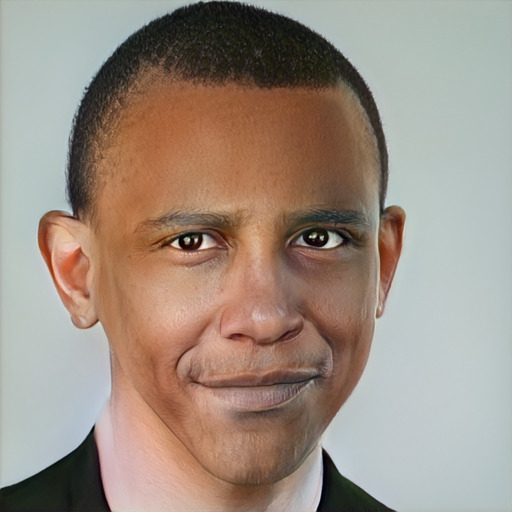} &
        \includegraphics[align=c,width=\ssize\textwidth]{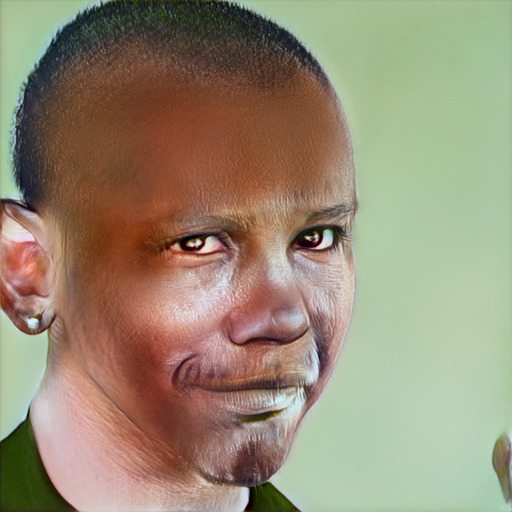} &
        \includegraphics[align=c,width=\ssize\textwidth]{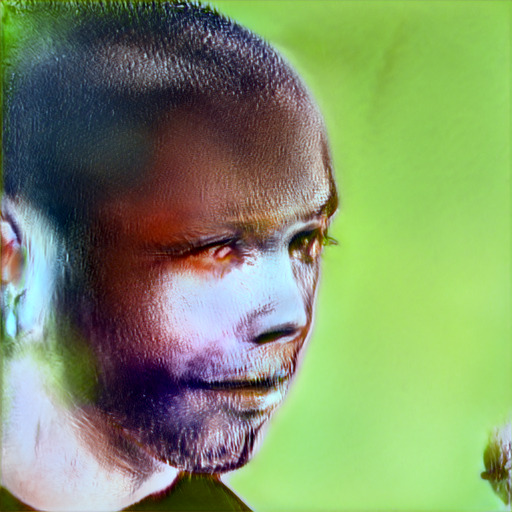} \\
        
        \vspace{0.5em} \\
        
        & \multicolumn{3}{r}{$\xleftarrow{\makebox[0.38\linewidth]{}}$} & \small Smile & \multicolumn{3}{l}{$\xrightarrow{\makebox[0.38\linewidth]{}}$} \\
        \vspace{0.25em} \\
        
        \rotatebox[]{90}{\parbox[c]{\ssize\textwidth}{\centering \small Generated}}\; &
        \includegraphics[align=c,width=\ssize\textwidth]{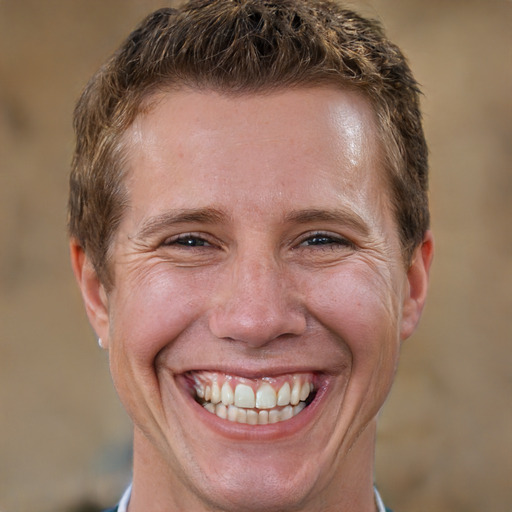} &
        \includegraphics[align=c,width=\ssize\textwidth]{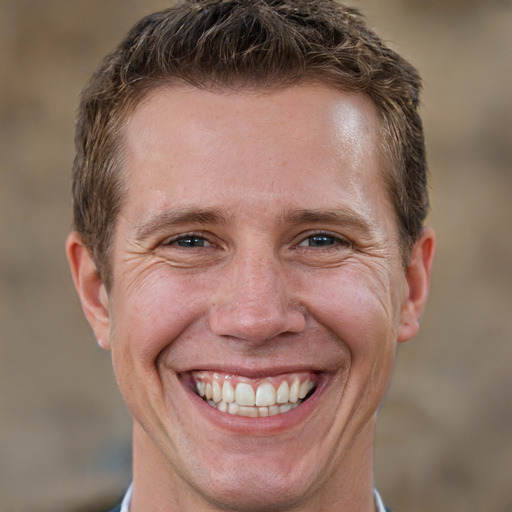} &
        \includegraphics[align=c,width=\ssize\textwidth]{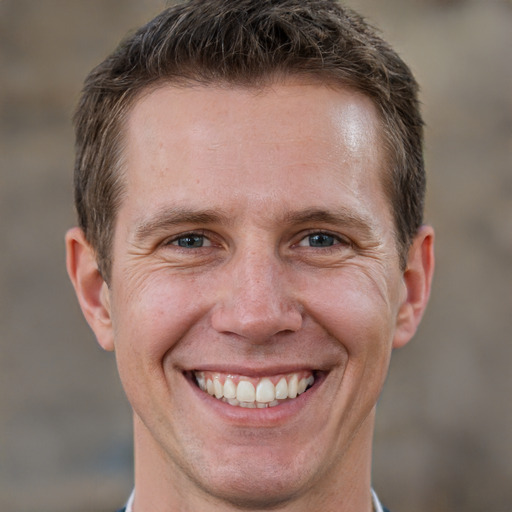} &
        \includegraphics[align=c,width=\ssize\textwidth]{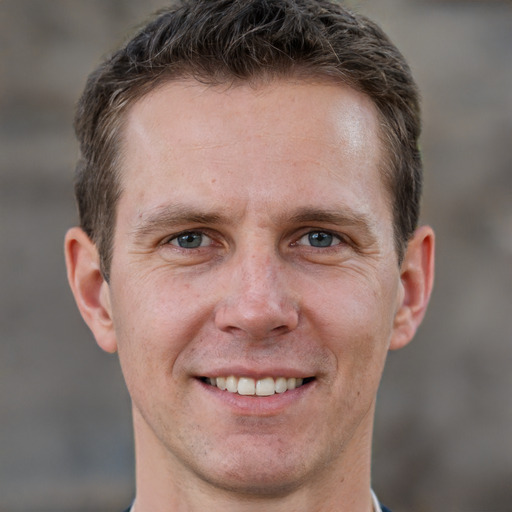} &
        \includegraphics[align=c,width=\ssize\textwidth]{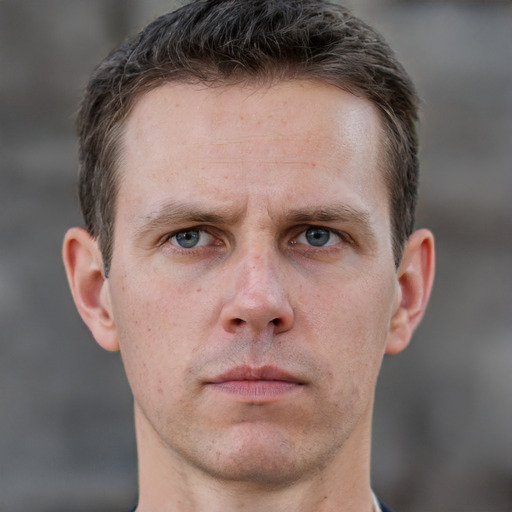} &
        \includegraphics[align=c,width=\ssize\textwidth]{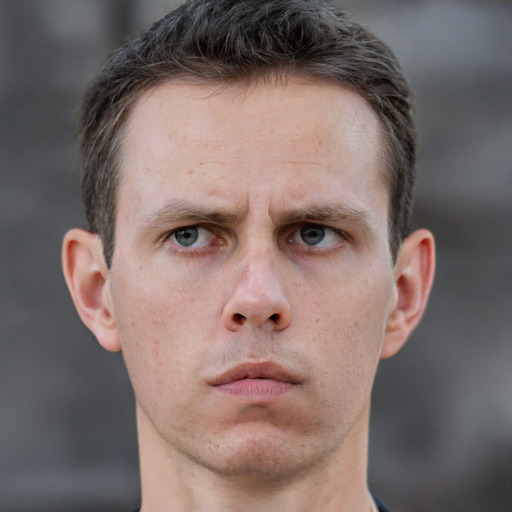} &
        \includegraphics[align=c,width=\ssize\textwidth]{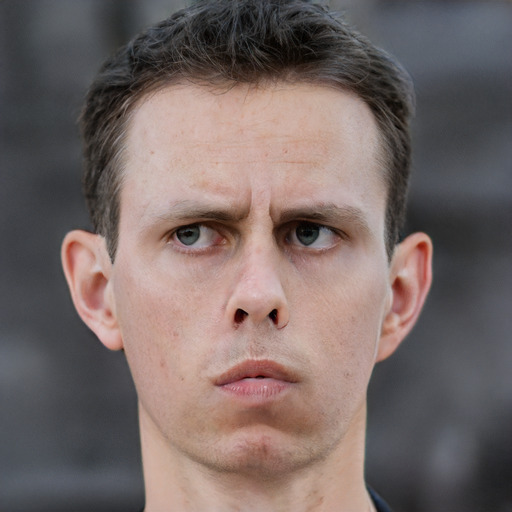} \\
        
        \rotatebox[]{90}{\parbox[c]{\ssize\textwidth}{\centering \small Projected}}\; &
        \includegraphics[align=c,width=\ssize\textwidth]{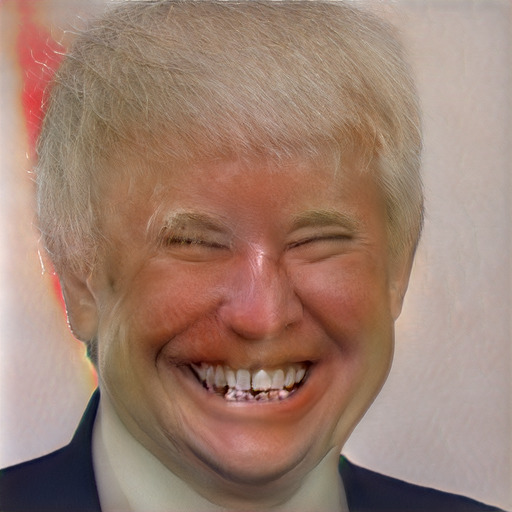} &
        \includegraphics[align=c,width=\ssize\textwidth]{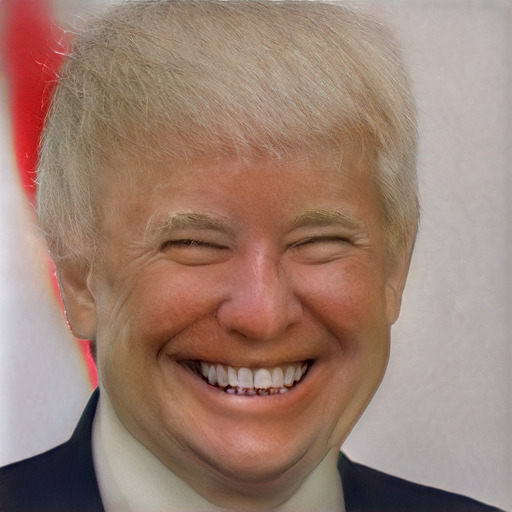} &
        \includegraphics[align=c,width=\ssize\textwidth]{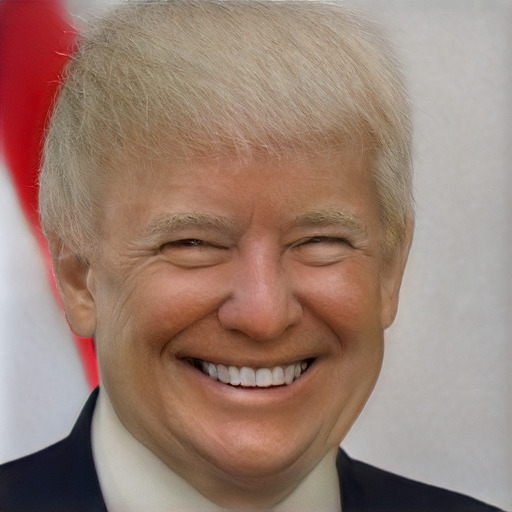} &
        \includegraphics[align=c,width=\ssize\textwidth]{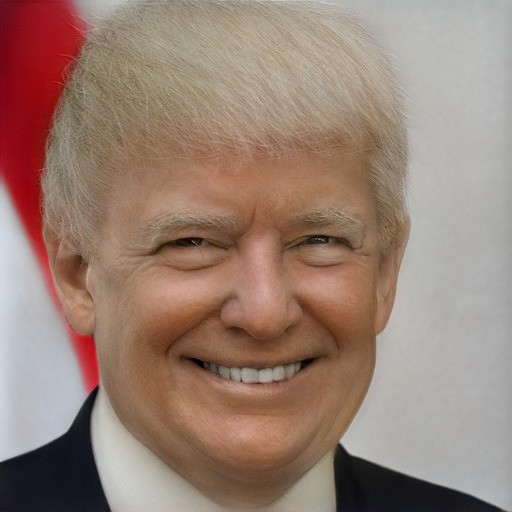} &
        \includegraphics[align=c,width=\ssize\textwidth]{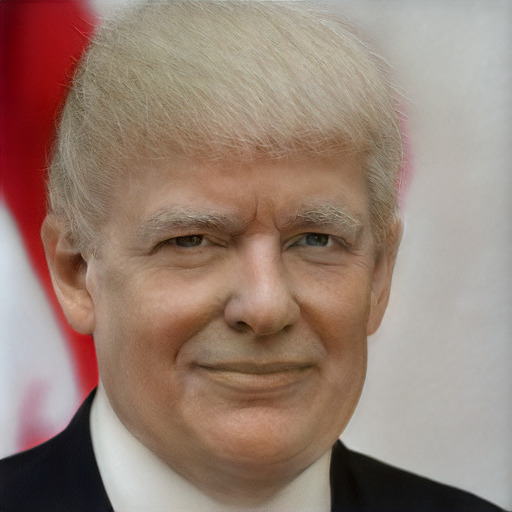} &
        \includegraphics[align=c,width=\ssize\textwidth]{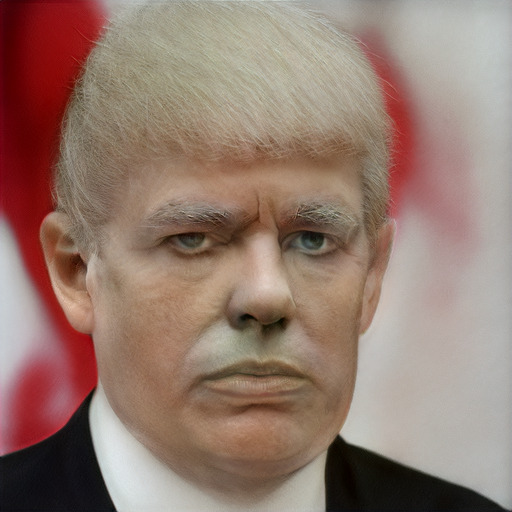} &
        \includegraphics[align=c,width=\ssize\textwidth]{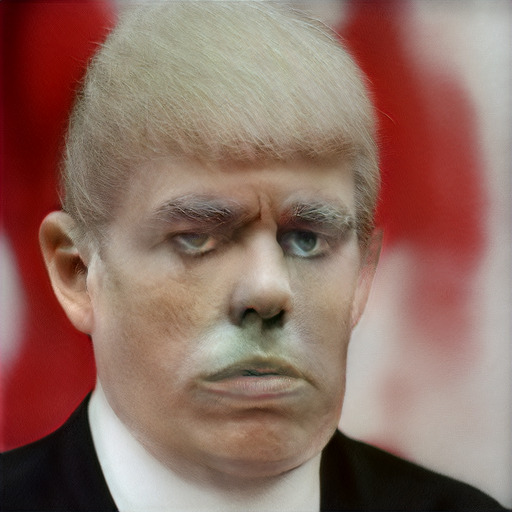} \\
        
    \end{tabular}
    \caption{
        The middle column represents latent representations that undergo attribute edition.
        For the first and the third row, representations are generated in the $\NWp$ space,
        and for the second and the forth row, representations are projected from the image into the $\NWp$ space.
        To obtain visually similar attribution edition effect for both generated and projected representation,
        we use attribute vectors with 2.5 times larger strength for projected representations.
        Quality of edited projected representations is worse, especially at the end of the spectrum and off-target effects are more visible.
    }
    \label{fig:heavy-attr-edit}
\end{figure}

\section{Distribution Interpolation}
Figure~\ref{fig:Wp-feat-dists} depicts a histogram of values in the $\mathcal{W}$ space for generated images.
We notice that all scalars of $\mathcal{W}$ space have similar distribution.
In Figure~\ref{fig:Wp-dist-and-corr},
we show that features are only weakly linearly correlated, and
the distribution of all scalars in $\mathcal{W}$ is similar to single scalar distributions.

\begin{figure}
    \centering \includegraphics[width=1\textwidth]{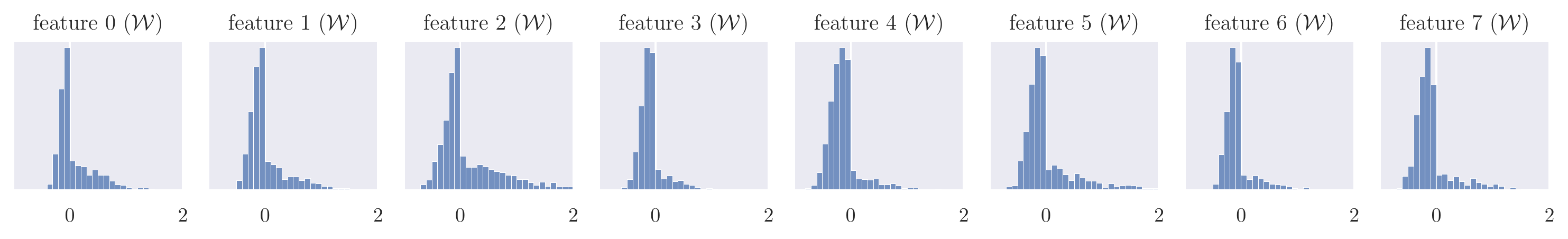}
    \caption{
        We sample 16384 latent vectors from the $\mathcal{Z}$ space, and convert them into $\mathcal{W}$
        by passing through the mapping network.
        We show histograms for first 8 features (of 512) in the $\mathcal{W}$ space.
        We notice that all features have similar distribution.
    }
    \label{fig:Wp-feat-dists}
\end{figure}

\begin{figure}
    \centering
    \begin{subfigure}{0.49\textwidth}
        \centering
        \includegraphics[width=1\textwidth]{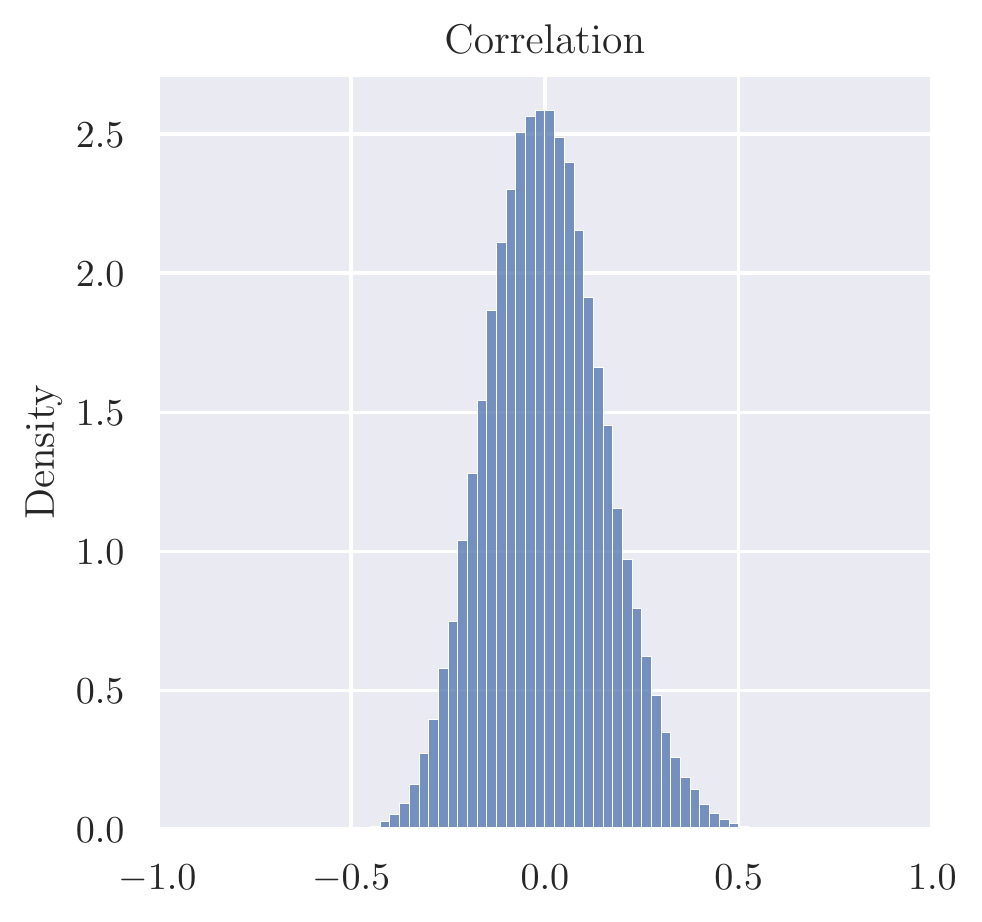}
        \caption{
            Histogram of the Pearson correlation between features in $\mathcal{W}$.
        }
    \end{subfigure}
    \begin{subfigure}{0.49\textwidth}
        \centering
        \includegraphics[width=1\textwidth]{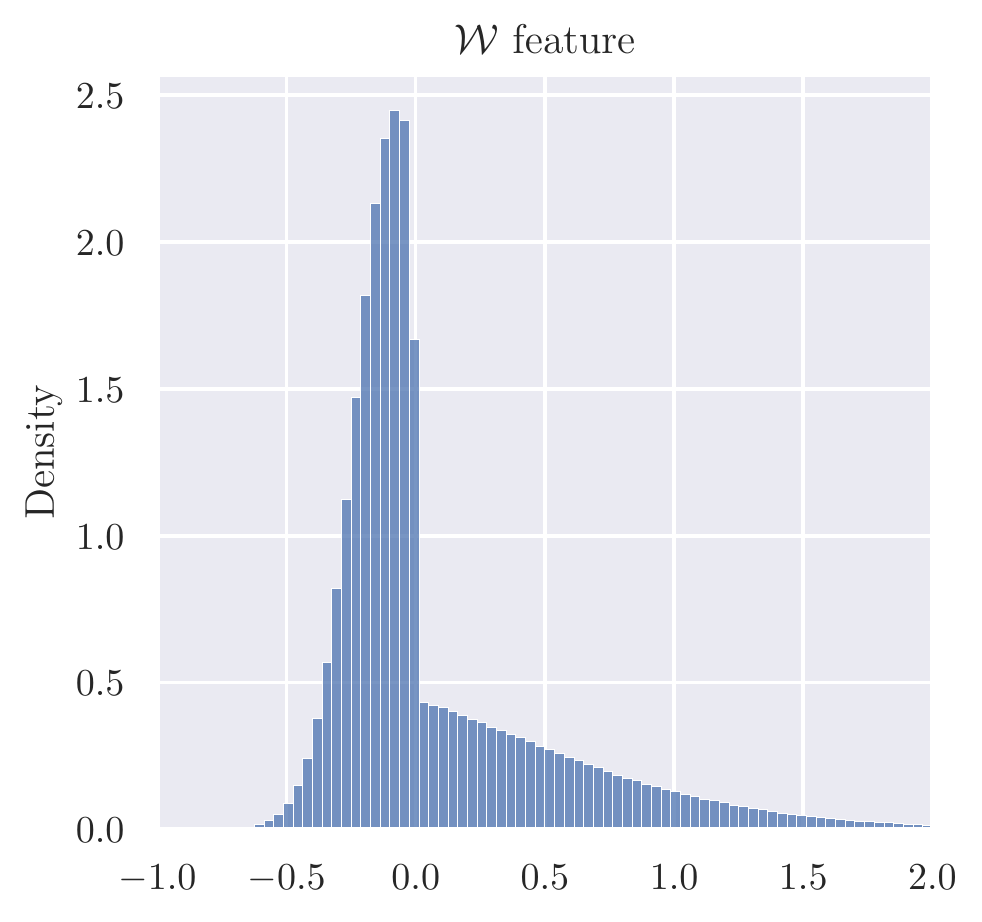}
        \caption{
            Histogram of scalars in $\mathcal{W}$
        }
    \end{subfigure}
    
    \caption{
        We can see that most features are not correlated with each other,
        or generally correlations are weak.
        We use this fact to compare distributions of single samples to distribution of all scalars in $\mathcal{W}$ as if they were independent.
    }
    \label{fig:Wp-dist-and-corr}
\end{figure}

We aim to use these facts to define a metric of being out-of-sample.
Assuming that features in $\mathcal{W}$ are independent and identically distributed,
we could calculate a distance between distribution constructed from scalars of a single sample
to the distribution of scalars in $\mathcal{W}$.
As shown previously, $\mathcal{W}$ seems to be close enough to these assumptions
to obtain a rough estimation for a our-of-sample measure.

In Figure~\ref{fig:Wp-proj-dist} we show such distributions for projected images into $\NWp$.
Distributions of projected images are far from the target distribution.
We speculate that such high discrepancy can have detrimental effect for image manipulation purposes.
For example, if a latent manipulation method relies on training using generated representations, it may consider projected representations out-of-sample,
\eg an attribute vector found by training a linear regression from the latent representation into the age of a person of generated image using a pre-trained age prediction model.
In particular, we expect attribute direction vectors to have lower impact on projected representations, because such representations are more scattered and
it is necessary to translate features further by increasing the length of a vector to achieve the same targeted effect,
which can make off-target effects of an attribute vector more visible.

\begin{figure}
    \setlength{\tabcolsep}{0pt}
    \renewcommand{\arraystretch}{0}
    \centering
    
    \def\ssize{0.23}
    \def\lsize{0.23}
    \begin{tabular}{cp{0.0em}ccc}
        & &
        \includegraphics[align=c,width=\lsize\textwidth]{samples/basic_fit__obama_NWp.jpg} &
        \includegraphics[align=c,width=\lsize\textwidth]{samples/basic_fit__trump_NWp.jpg} &
        \includegraphics[align=c,width=\lsize\textwidth]{samples/basic_fit__biden_NWp.jpg} \\
        
        \includegraphics[width=0.3\textwidth]{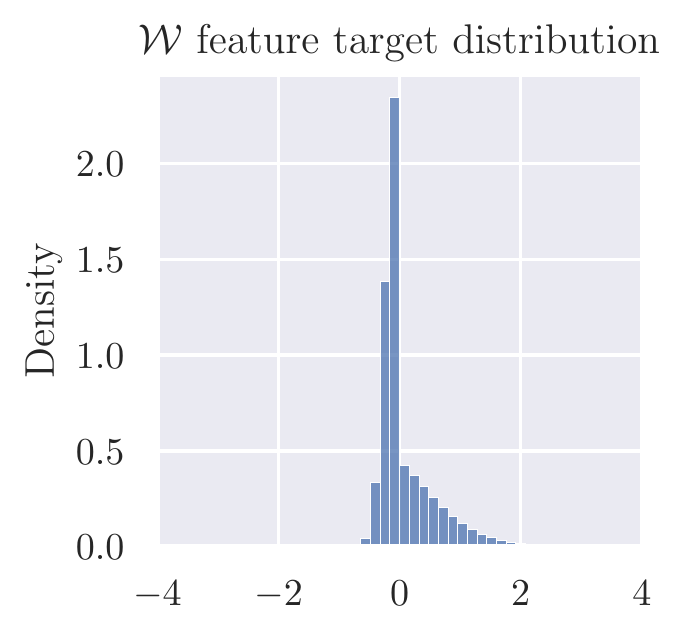} & &
        \includegraphics[width=\ssize\textwidth]{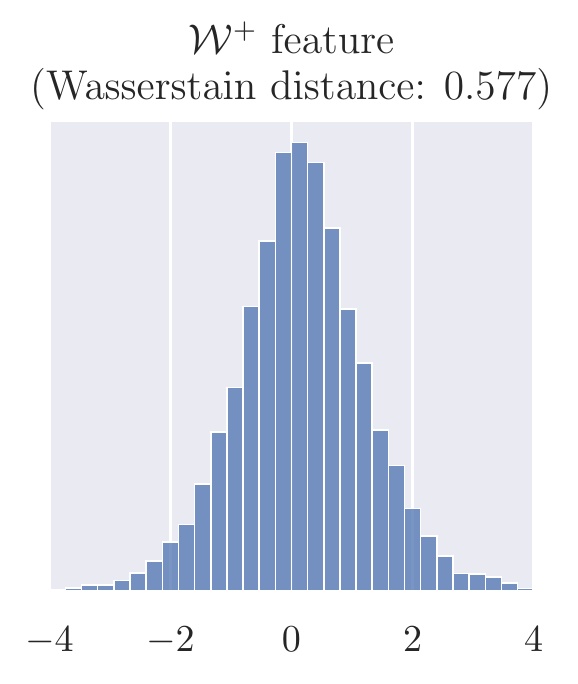} &
        \includegraphics[width=\ssize\textwidth]{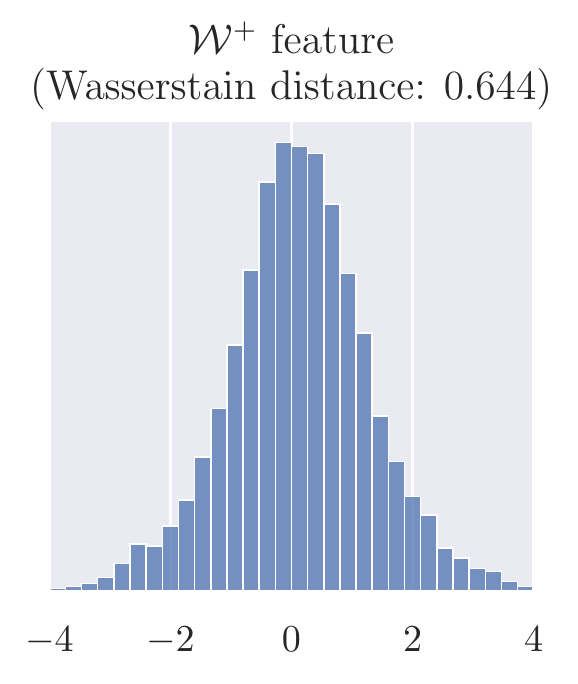} &
        \includegraphics[width=\ssize\textwidth]{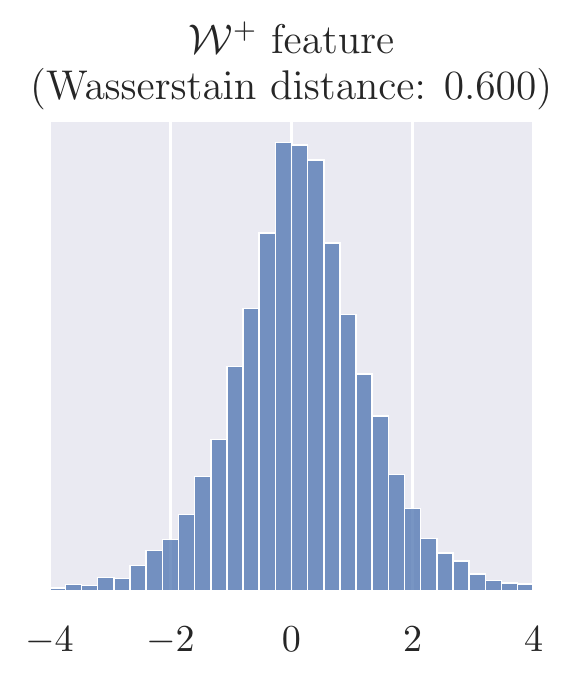}
    \end{tabular}
    
    \caption{
        We plot distributions of $\Wp$ values for $\NWp$ projections.
        The values are far from the expected range.
        The Wasserstein distance is in range 0.55--0.65, whereas
        randomly generated latent representations have 0.047 in the expected case
        with standard deviation of 0.0284.
    }
    \label{fig:Wp-proj-dist}
\end{figure}

To validate that effect, we interpolate sample distributions into the target distribution.
We interpolate distributions in a way that the Wasserstein distance decreases linearly with the coefficient $c$
and the order of elements is preserved.
\begin{equation}
\text{distribution-interpolation}_{S, T}\funarg{x, c} = (1 - c) \cdot x + c \cdot T^{-1}\funarg{S\funarg{x}}.
\label{eq:dist-interp}
\end{equation}
This is a formula for transforming a scalar $x \in \mathbb{R}$ in the sample distribution using coefficient $c \in [0, 1]$, where $S$ is the cumulative sample distribution function and $T$ is the cumulative target distribution function.

Intuitively, $S\funarg{x}$ returns a quantile of value $x$ in the sample distribution. $T^{-1}\funarg{S\funarg{x}}$ takes that quantile and converts it into a value that in the target distribution has the same quantile. Then we linearly interpolate between $x$ and $T^{-1}\funarg{S\funarg{x}}$ using a coefficient $c$.

\subsection{Experiments}

We show the obtained images using this method and attribute edition on $\NWp$ in Figure~\ref{fig:Wp-equalization}.
Similarly, we experiment with distribution interpolation for $\FNWp{}$.
To do this, we additionally approximate target distributions of scalars for feature and RGB maps by sampling latent vectors from $\mathcal{Z}$ and converting them appropriately.
Figure~\ref{fig:F5Wp-equalization} shows the effect on $\FNWp{5}$ latent representations which were optimized only for 250 iterations instead of 1000 as in $\NWp$.
We can see that the visual difference and perceived attribute strength is much less significant. We suspect that the reason is distribution similarity coming from short optimizing of the latent representation.
In Figure~\ref{fig:F5Wp-equalization-convergence}, we show the effect again in $\FNWp{5}$ however with latent representations optimized till convergence.

Note that all experiments we did in this section, use representations obtained in a way as described in previous chapters.
We didn't constraint our latent representations to yield positive results when using such distribution interpolation,
but we still managed in some cases to obtain satisfying results.
That shows that relative values and the order of values is a very important part of representation.

\begin{figure}[b!]
    \setlength{\tabcolsep}{0pt}
    \renewcommand{\arraystretch}{0}
    \centering
    
    \def\ssize{0.155}
    \def\lsize{0.155}
    \begin{tabular}{ccccccc}
        & $c = 0$ & $c = 0.2$ & $c = 0.4$ & $c = 0.6$ & $c = 0.8$ & $c = 1$ \\
        \vspace{0.5em} \\
        
        &
        \includegraphics[align=c,width=\ssize\textwidth]{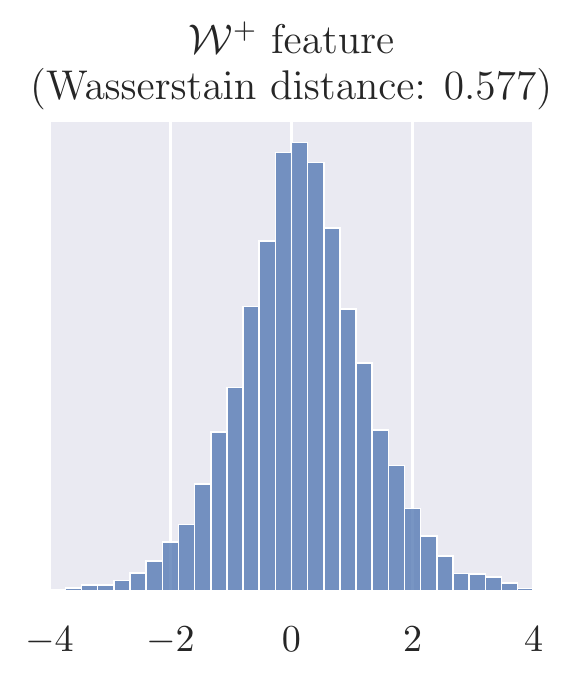} &
        \includegraphics[align=c,width=\ssize\textwidth]{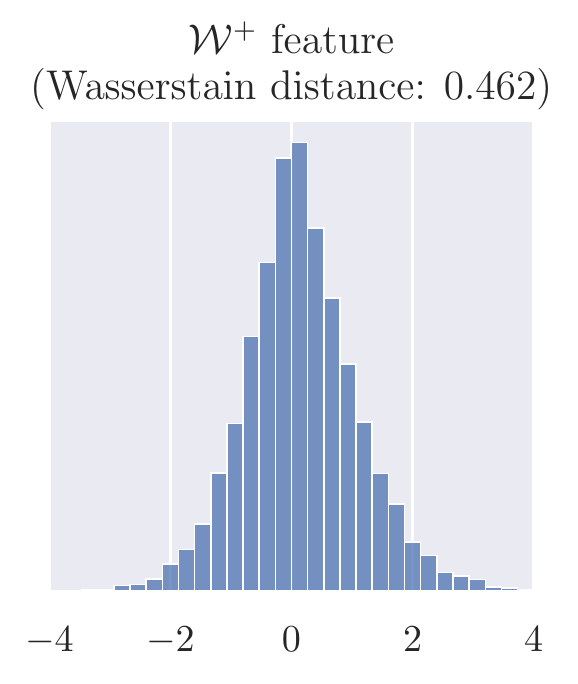} &
        \includegraphics[align=c,width=\ssize\textwidth]{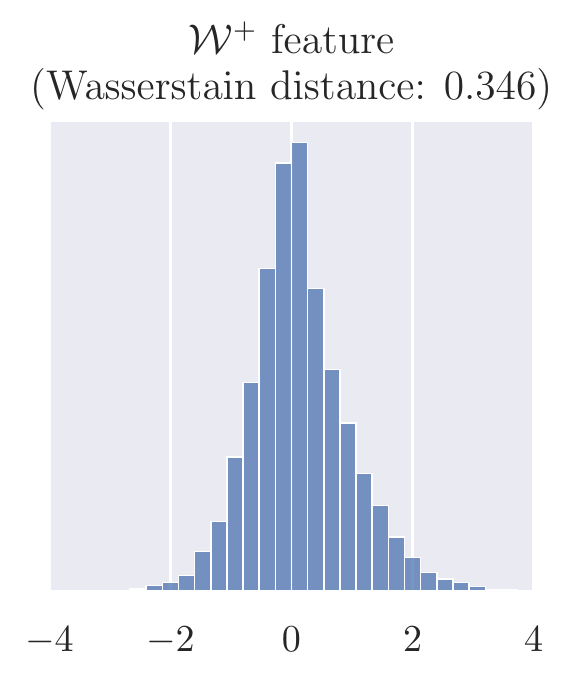} &
        \includegraphics[align=c,width=\ssize\textwidth]{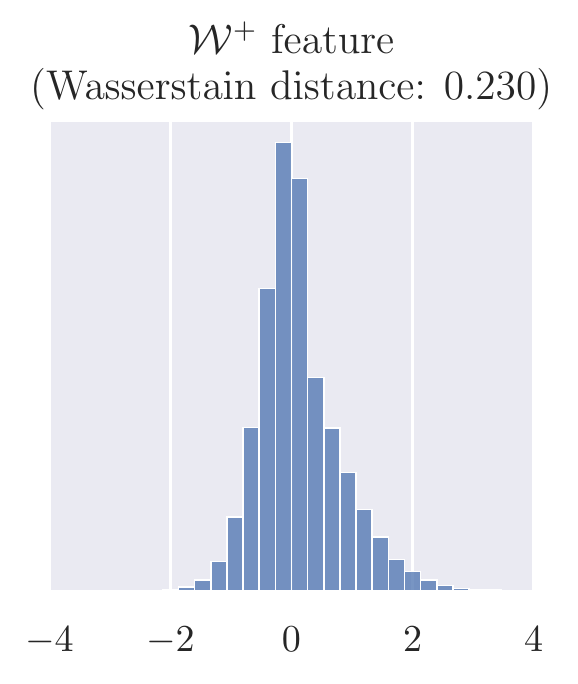} &
        \includegraphics[align=c,width=\ssize\textwidth]{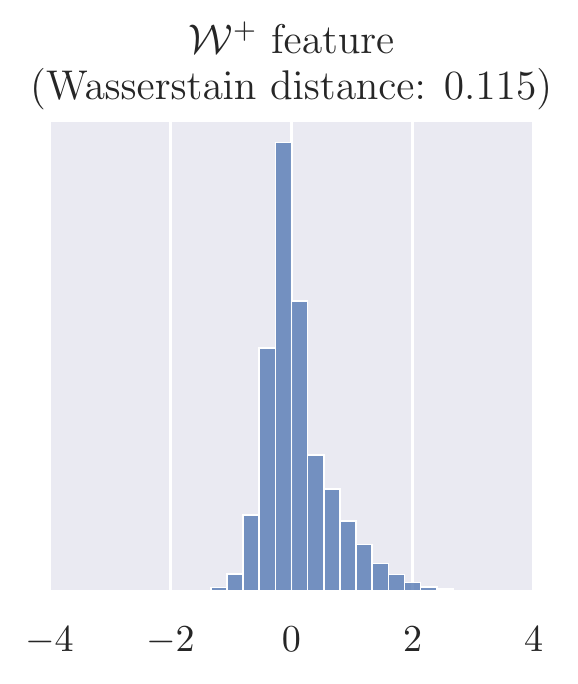} &
        \includegraphics[align=c,width=\ssize\textwidth]{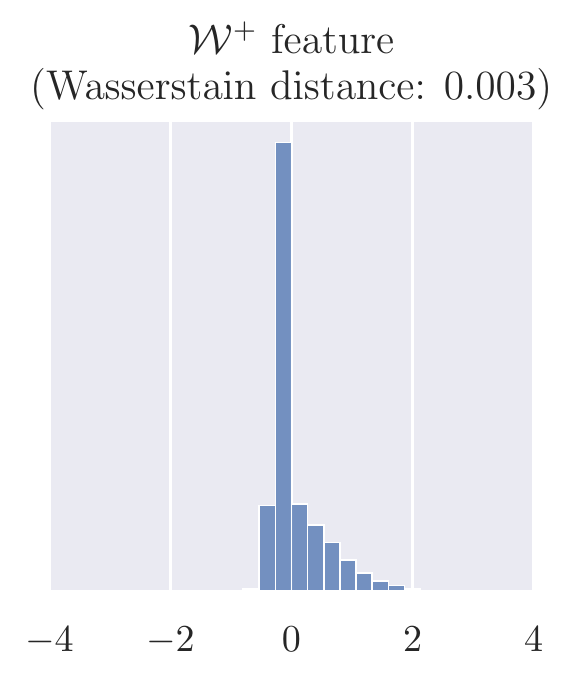} \\
        
        \vspace{0.25em} \\
        & \scriptsize LPIPS: 0.2458 & \scriptsize LPIPS: 0.2560 & \scriptsize LPIPS: 0.2770 & \scriptsize LPIPS: 0.3033 & \scriptsize LPIPS: 0.3365 & \scriptsize LPIPS: 0.3824 \\
        \vspace{0.25em} \\
        
        \rotatebox[]{90}{\parbox[c]{\lsize\textwidth}{\centering Image after interpolation}}\; &
        \includegraphics[align=c,width=\lsize\textwidth]{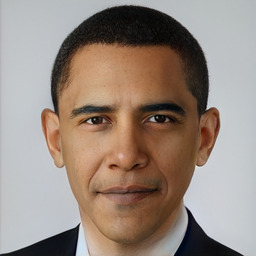} &
        \includegraphics[align=c,width=\lsize\textwidth]{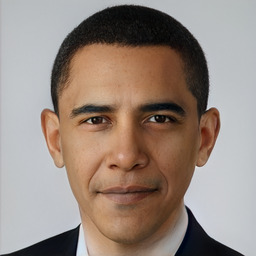} &
        \includegraphics[align=c,width=\lsize\textwidth]{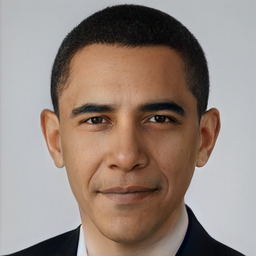} &
        \includegraphics[align=c,width=\lsize\textwidth]{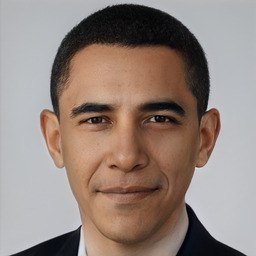} &
        \includegraphics[align=c,width=\lsize\textwidth]{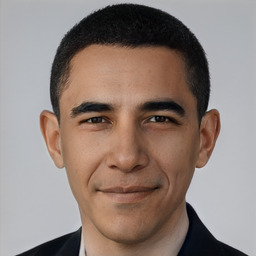} &
        \includegraphics[align=c,width=\lsize\textwidth]{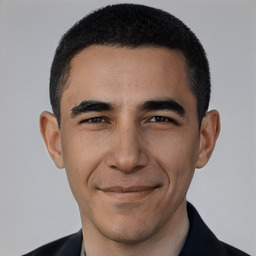} \\
        
        \rotatebox[]{90}{\parbox[c]{\lsize\textwidth}{\centering + Age}}\; &
        \includegraphics[align=c,width=\lsize\textwidth]{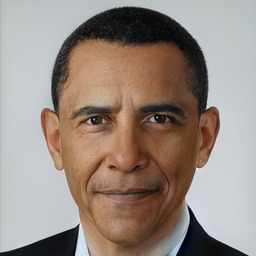} &
        \includegraphics[align=c,width=\lsize\textwidth]{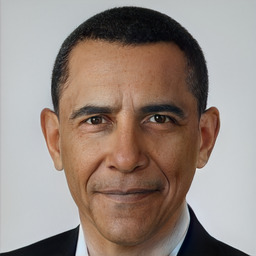} &
        \includegraphics[align=c,width=\lsize\textwidth]{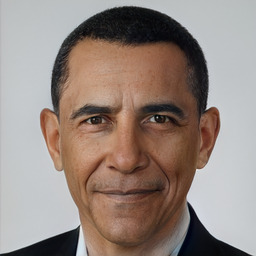} &
        \includegraphics[align=c,width=\lsize\textwidth]{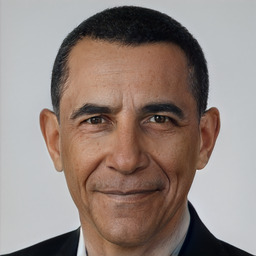} &
        \includegraphics[align=c,width=\lsize\textwidth]{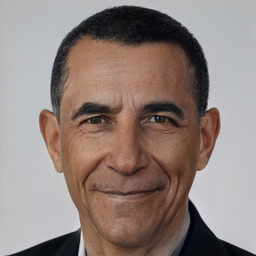} &
        \includegraphics[align=c,width=\lsize\textwidth]{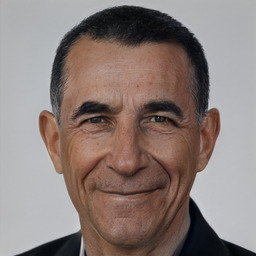} \\
        
        \rotatebox[]{90}{\parbox[c]{\lsize\textwidth}{\centering + Smile}}\; &
        \includegraphics[align=c,width=\lsize\textwidth]{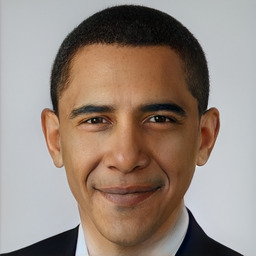} &
        \includegraphics[align=c,width=\lsize\textwidth]{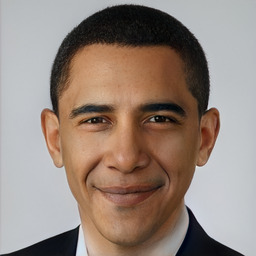} &
        \includegraphics[align=c,width=\lsize\textwidth]{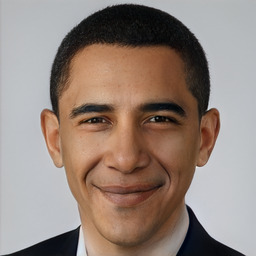} &
        \includegraphics[align=c,width=\lsize\textwidth]{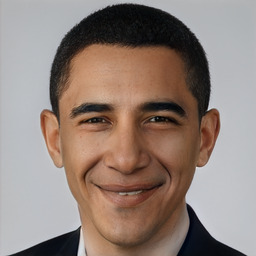} &
        \includegraphics[align=c,width=\lsize\textwidth]{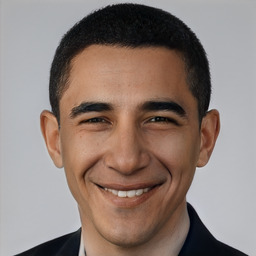} &
        \includegraphics[align=c,width=\lsize\textwidth]{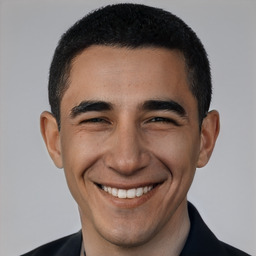} \\
        
        \rotatebox[]{90}{\parbox[c]{\lsize\textwidth}{\centering Gender}}\; &
        \includegraphics[align=c,width=\lsize\textwidth]{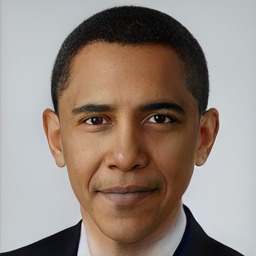} &
        \includegraphics[align=c,width=\lsize\textwidth]{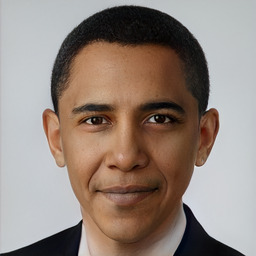} &
        \includegraphics[align=c,width=\lsize\textwidth]{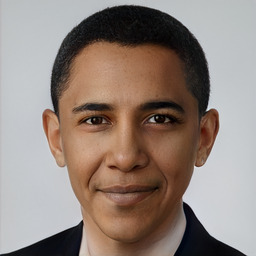} &
        \includegraphics[align=c,width=\lsize\textwidth]{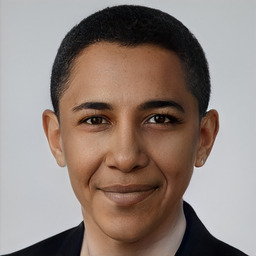} &
        \includegraphics[align=c,width=\lsize\textwidth]{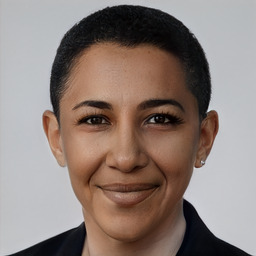} &
        \includegraphics[align=c,width=\lsize\textwidth]{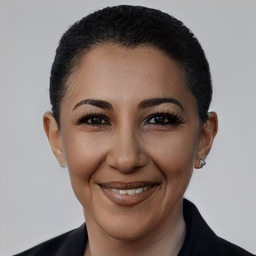} \\
        
        %
        
    \end{tabular}
    
    \caption{
        We take a projection into $\NWp$, and use Equation~\ref{eq:dist-interp} with different interpolation coefficients
        to obtain multiple latent representations. For every representation we use the same attribute strength.
        Using higher distribution interpolation coefficients make the image more different from the original image,
        both visually and in terms of LPIPS metric.
        However with higher coefficients, images become to look more natural,
        and are more sensitive to attribute edition.
    }
    \label{fig:Wp-equalization}
    
    \vspace{0.08\textheight}  

\end{figure}

\begin{figure}
    \setlength{\tabcolsep}{0pt}
    \renewcommand{\arraystretch}{0}
    \centering
    
    \def\ssize{0.14}
    \def\lsize{0.14}
    \begin{tabular}{ccccccc}
        & $c = 0$ & $c = 0.2$ & $c = 0.4$ & $c = 0.6$ & $c = 0.8$ & $c = 1$ \\
        \vspace{0.5em} \\
        
        &
        \includegraphics[align=c,width=\ssize\textwidth]{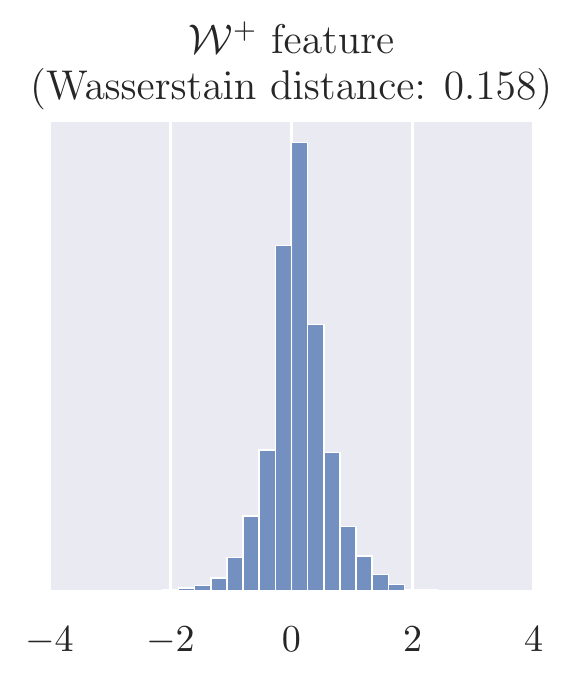} &
        \includegraphics[align=c,width=\ssize\textwidth]{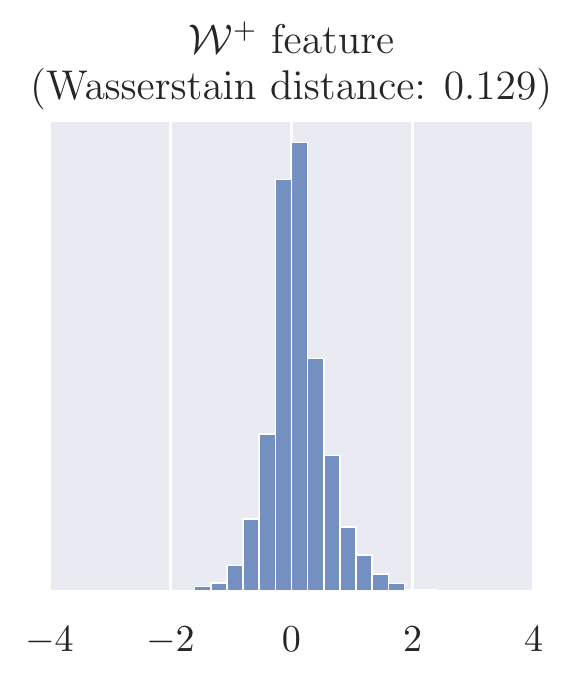} &
        \includegraphics[align=c,width=\ssize\textwidth]{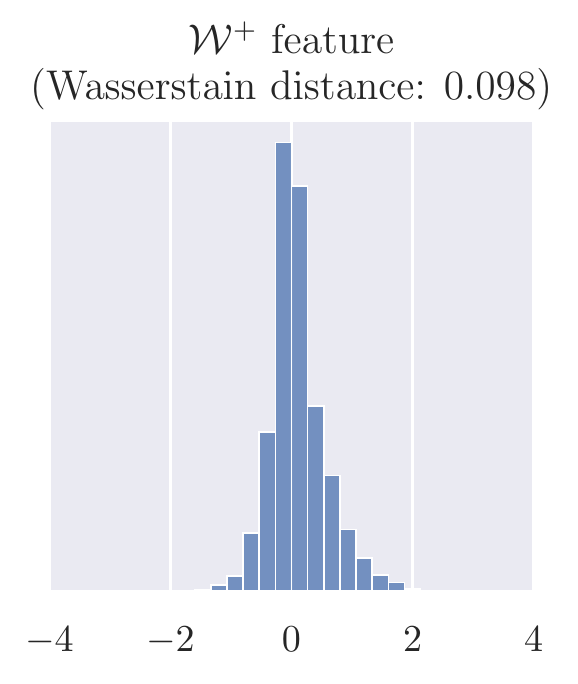} &
        \includegraphics[align=c,width=\ssize\textwidth]{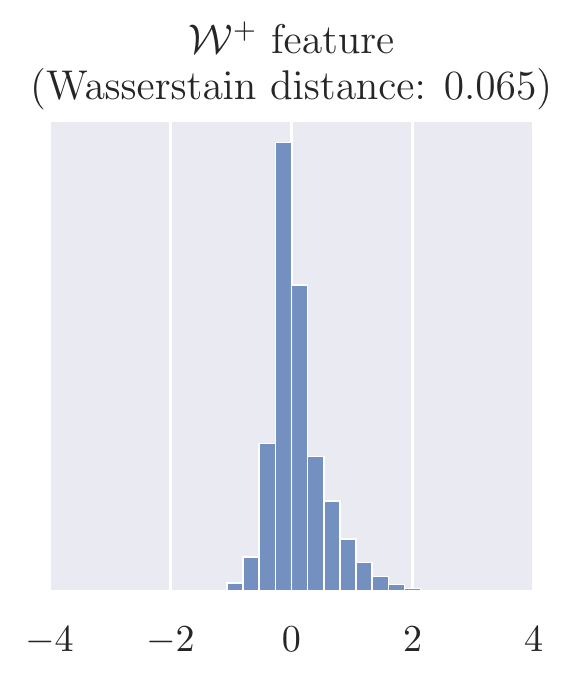} &
        \includegraphics[align=c,width=\ssize\textwidth]{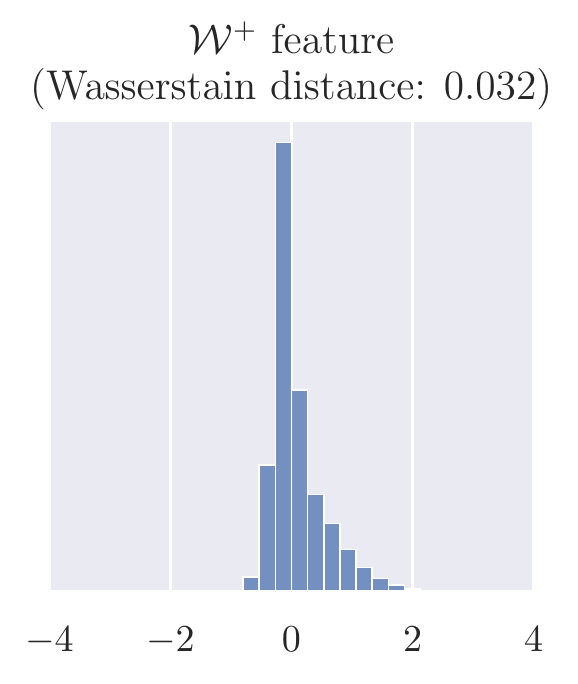} &
        \includegraphics[align=c,width=\ssize\textwidth]{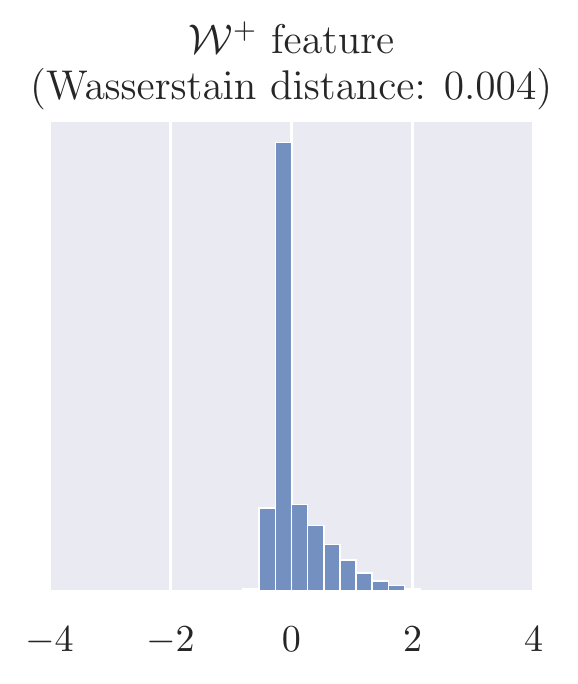} \\
        
        \vspace{0.25em} \\
        & \scriptsize LPIPS: 0.2076 & \scriptsize LPIPS: 0.2144 & \scriptsize LPIPS: 0.2322 & \scriptsize LPIPS: 0.2521 & \scriptsize LPIPS: 0.2736 & \scriptsize LPIPS: 0.2954 \\
        \vspace{0.25em} \\
        
        \rotatebox[]{90}{\parbox[c]{\lsize\textwidth}{\centering Image after interpolation}}\; &
        \includegraphics[align=c,width=\lsize\textwidth]{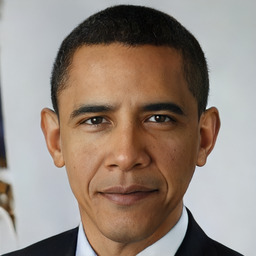} &
        \includegraphics[align=c,width=\lsize\textwidth]{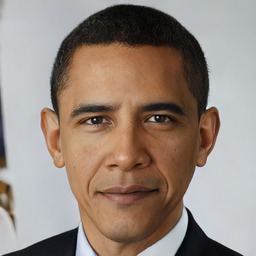} &
        \includegraphics[align=c,width=\lsize\textwidth]{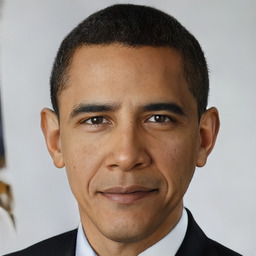} &
        \includegraphics[align=c,width=\lsize\textwidth]{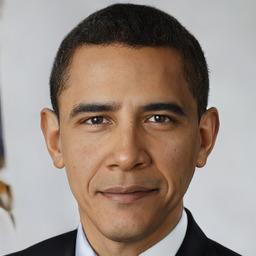} &
        \includegraphics[align=c,width=\lsize\textwidth]{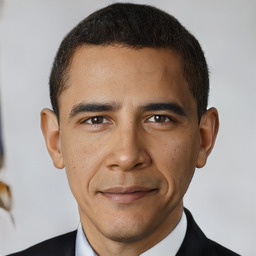} &
        \includegraphics[align=c,width=\lsize\textwidth]{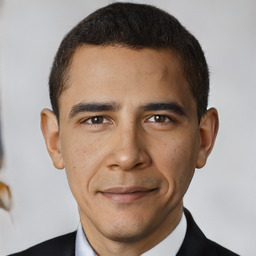} \\
        
        \rotatebox[]{90}{\parbox[c]{\lsize\textwidth}{\centering + Age}}\; &
        \includegraphics[align=c,width=\lsize\textwidth]{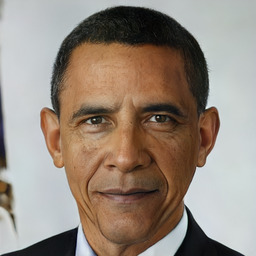} &
        \includegraphics[align=c,width=\lsize\textwidth]{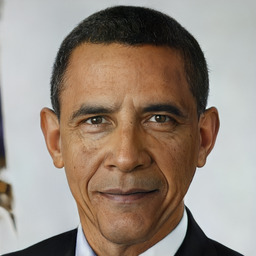} &
        \includegraphics[align=c,width=\lsize\textwidth]{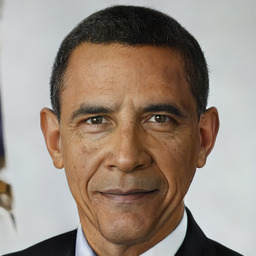} &
        \includegraphics[align=c,width=\lsize\textwidth]{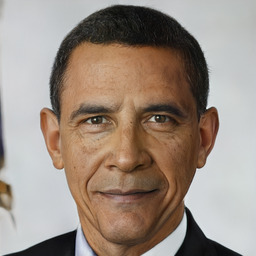} &
        \includegraphics[align=c,width=\lsize\textwidth]{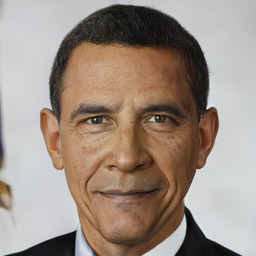} &
        \includegraphics[align=c,width=\lsize\textwidth]{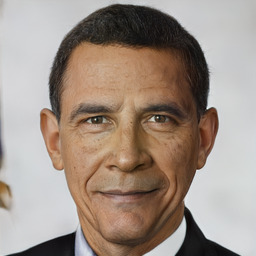} \\
        
        \rotatebox[]{90}{\parbox[c]{\lsize\textwidth}{\centering + Smile}}\; &
        \includegraphics[align=c,width=\lsize\textwidth]{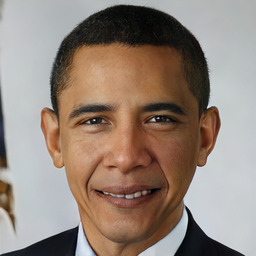} &
        \includegraphics[align=c,width=\lsize\textwidth]{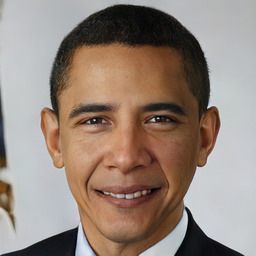} &
        \includegraphics[align=c,width=\lsize\textwidth]{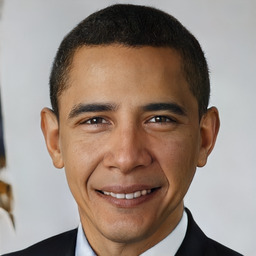} &
        \includegraphics[align=c,width=\lsize\textwidth]{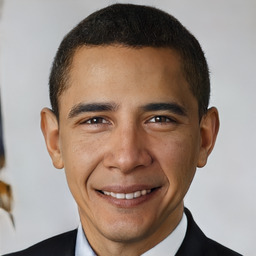} &
        \includegraphics[align=c,width=\lsize\textwidth]{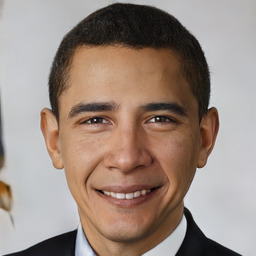} &
        \includegraphics[align=c,width=\lsize\textwidth]{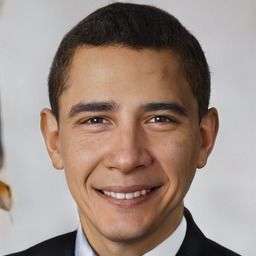} \\
        
        \rotatebox[]{90}{\parbox[c]{\lsize\textwidth}{\centering Gender}}\; &
        \includegraphics[align=c,width=\lsize\textwidth]{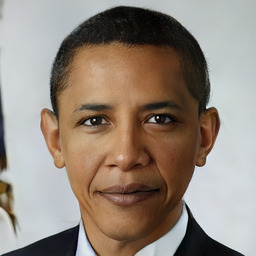} &
        \includegraphics[align=c,width=\lsize\textwidth]{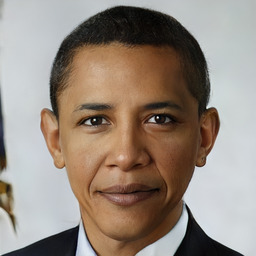} &
        \includegraphics[align=c,width=\lsize\textwidth]{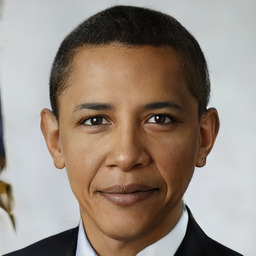} &
        \includegraphics[align=c,width=\lsize\textwidth]{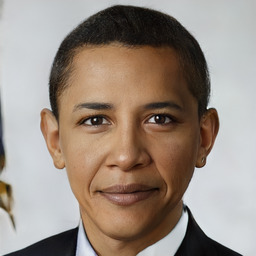} &
        \includegraphics[align=c,width=\lsize\textwidth]{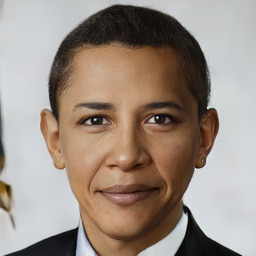} &
        \includegraphics[align=c,width=\lsize\textwidth]{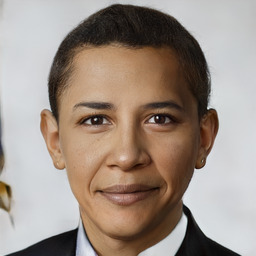} \\

    \end{tabular}
    
    \caption{
        We follow the procedure from Figure~\ref{fig:Wp-equalization}, but on $\FNWp{5}$.
        We interpolate style vectors, feature maps, and RGB maps independently,
        \ie we use Equation~\ref{eq:dist-interp} three times.
        In the first row we show distribution of $\Wp$.
        The visual difference and relative strength of attribute edition is much less visible compared to Figure~\ref{fig:Wp-equalization}.
        We suspect that the reason is greater distribution similarity of the projected latent representation.
        For comparison the Wasserstein distance is 0.159, whereas for the $\NWp$ projection it is 0.578.
    }
    \label{fig:F5Wp-equalization}
\end{figure}

\begin{figure}
    \setlength{\tabcolsep}{0pt}
    \renewcommand{\arraystretch}{0}
    \centering
    
    \def\ssize{0.14}
    \def\lsize{0.14}
    \begin{tabular}{ccccccc}
        & $c = 0$ & $c = 0.2$ & $c = 0.4$ & $c = 0.6$ & $c = 0.8$ & $c = 1$ \\
        \vspace{0.5em} \\
        
        &
        \includegraphics[align=c,width=\ssize\textwidth]{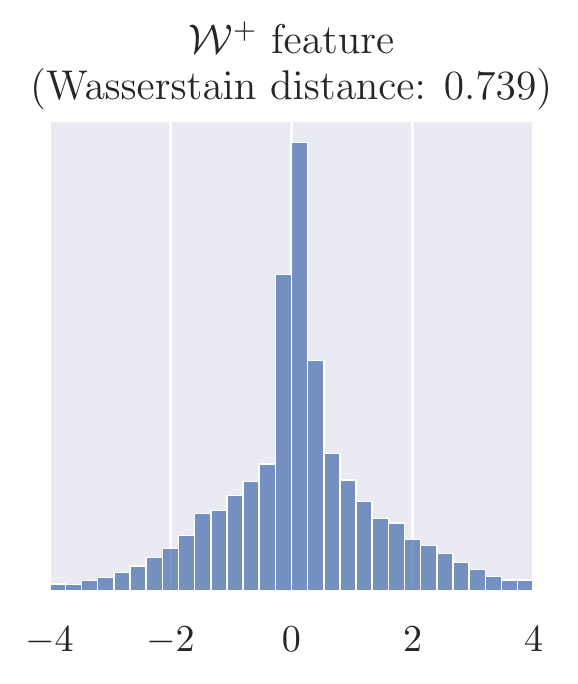} &
        \includegraphics[align=c,width=\ssize\textwidth]{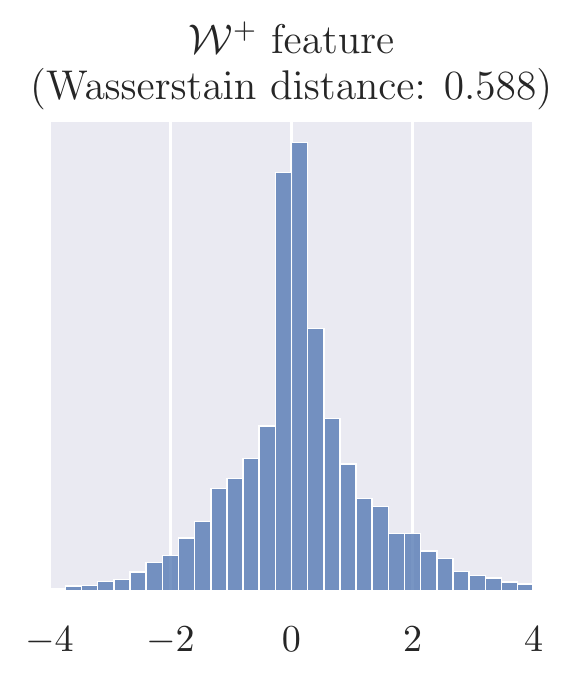} &
        \includegraphics[align=c,width=\ssize\textwidth]{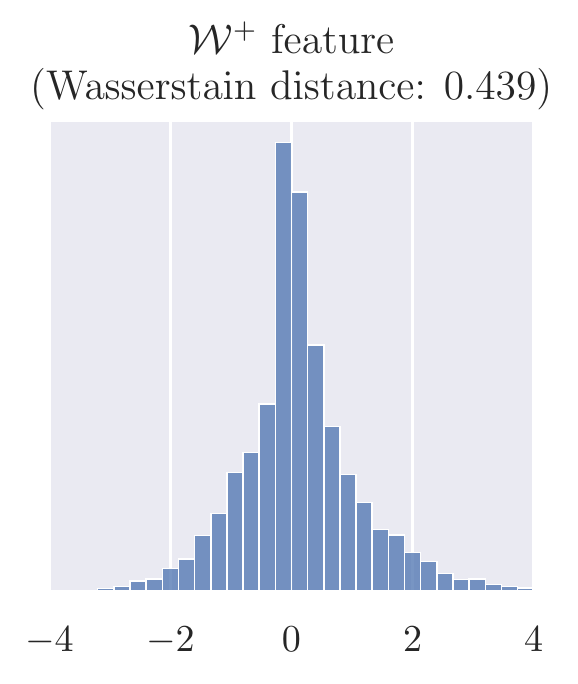} &
        \includegraphics[align=c,width=\ssize\textwidth]{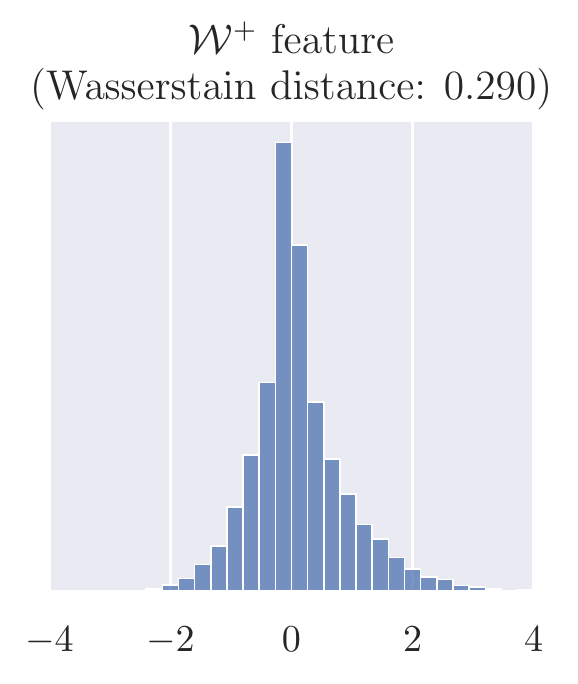} &
        \includegraphics[align=c,width=\ssize\textwidth]{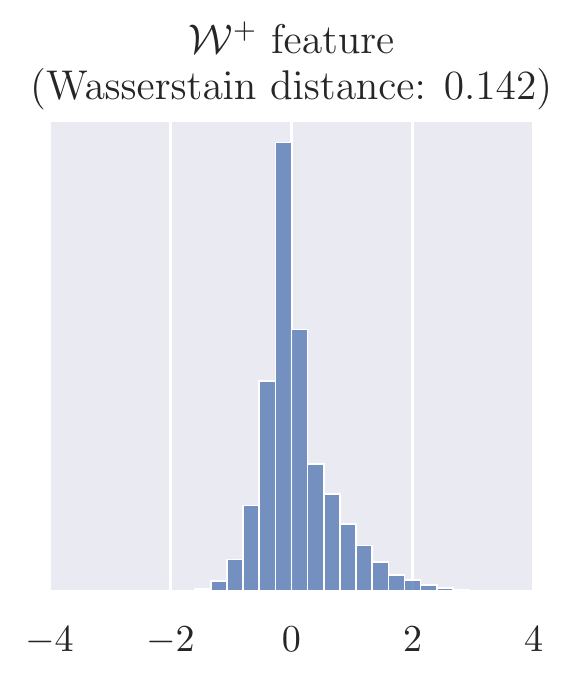} &
        \includegraphics[align=c,width=\ssize\textwidth]{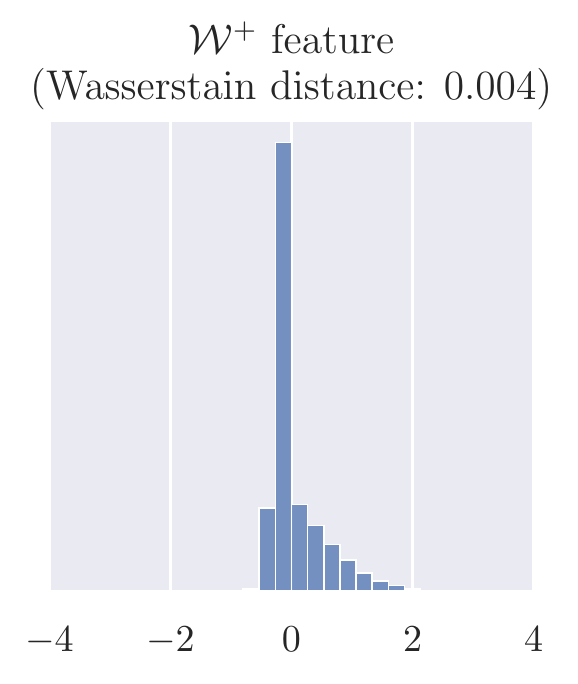} \\
        
        \vspace{0.25em} \\
        & \scriptsize LPIPS: 0.1614 & \scriptsize LPIPS: 0.1830 & \scriptsize LPIPS: 0.2257 & \scriptsize LPIPS: 0.2756 & \scriptsize LPIPS: 0.3353 & \scriptsize LPIPS: 0.4250 \\
        \vspace{0.25em} \\
        
        \rotatebox[]{90}{\parbox[c]{\lsize\textwidth}{\centering Image after interpolation}}\; &
        \includegraphics[align=c,width=\lsize\textwidth]{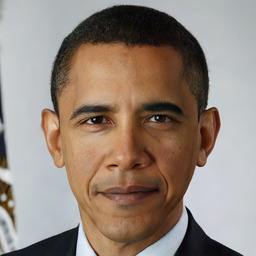} &
        \includegraphics[align=c,width=\lsize\textwidth]{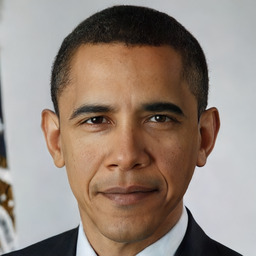} &
        \includegraphics[align=c,width=\lsize\textwidth]{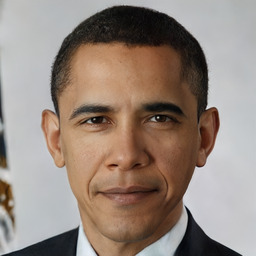} &
        \includegraphics[align=c,width=\lsize\textwidth]{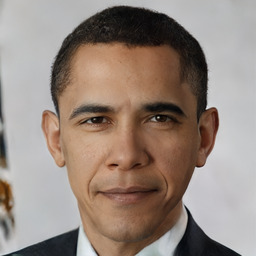} &
        \includegraphics[align=c,width=\lsize\textwidth]{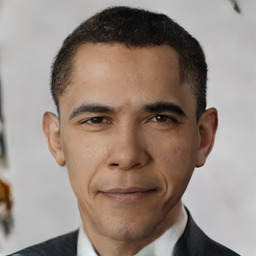} &
        \includegraphics[align=c,width=\lsize\textwidth]{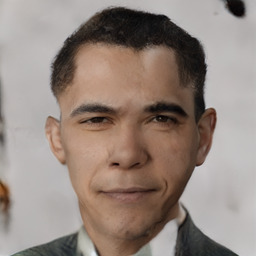} \\

    \end{tabular}
    
    \caption{
        We follow the procedure from Figure~\ref{fig:F5Wp-equalization}, but on the latent representation trained till convergence.
        As stated in Figure~\ref{fig:fitting-init-and-len}, such representations fail to capture semantics of the image,
        and distribution interpolation only worsen the results.
    }
    \label{fig:F5Wp-equalization-convergence}
\end{figure}

\clearpage
\section{Distribution Regularization}
Inspired by findings in the previous chapter, we try to constrain distributions during latent optimization.
We define a \textit{distribution regularization}.
\begin{equation}
\text{distribution-regularization}\funarg{S, T} = \text{MSE}\funarg{\text{sorted}\funarg{S}, \text{sorted}\funarg{T}},
\label{eq:dist-reg}
\end{equation}
where $S$ is a flattened tensor from a representation, and $T$ is a vector representing the target distribution.
We obtain a vector $T$, by
\begin{enumerate}
    \item sampling multiple ($4096$) latent vectors $z \in \mathcal{Z}$
    \item converting them into a given space (\eg $\mathcal{W}$, $\mathcal{F}_5$)
    \item flattening and sorting elements by their values
    \item reducing number of elements by averaging consecutive element groups in a way to obtain a correct number of elements \ie as many as the dimensionality of the space.
\end{enumerate}
Alternatively, we can think of it as L2 regularization with non-zero center, which can change depending on the order of elements.
We use the formula as an additional addend for loss in Algorithm~\ref{alg:latent-fitting}.

\subsection{Experiments}
In Figure~\ref{fig:Wp-disteq} and \ref{fig:Wp-disteq2}, we show the effect of using our regularization in the $\NWp$ space.
We can project images preserving visual and LPIPS quality, while keeping the $\Wp$ distribution as in generated images.
Similarly to Figure~\ref{fig:Wp-equalization}, we also notice difference in the strength of obtained attribute edition effect depending on loss coefficient of our distribution regularization.

Figures~\ref{fig:FWp-disteq} and \ref{fig:FWp-disteq2} show the effect of using distribution regularization for projections into $\FNWp{5}$. The regularization has a strong effect, making longer than 250 iterations optimization possible and avoid the problem described in Figure~\ref{fig:fitting-init-and-len}. 
We show that distributions are preserved when training longer, projected representations produce much higher quality images and it's possible to meaningfully edit them. We improve LPIPS score up to 17.5\% compared to projections in the same space without using the regularization. In total, we improve LPIPS score up to 30\% compared to projections into $\NWp$.

\begin{figure}
    \setlength{\tabcolsep}{0pt}
    \renewcommand{\arraystretch}{0}
    \centering
    
    \def\ssize{0.2}
    \def\lsize{0.2}
    \begin{tabular}{ccccccc}
        & $\lambda_{\text{dist}} = 0$ & $\lambda_{\text{dist}} = 0.1$ & $\lambda_{\text{dist}} = 0.2$ & $\lambda_{\text{dist}} = 0.5$ \\
        \vspace{0.5em} \\
        
        &
        \includegraphics[align=c,width=\ssize\textwidth]{plots/Wp_obama_NWp.pdf} &
        \includegraphics[align=c,width=\ssize\textwidth]{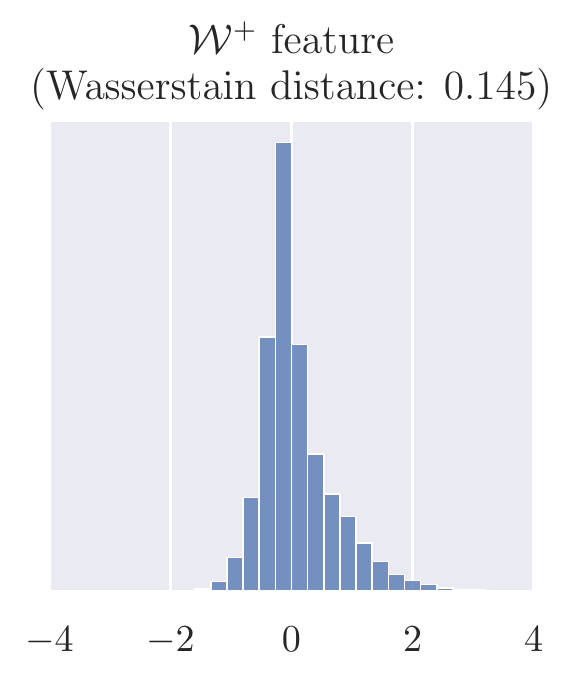} &
        \includegraphics[align=c,width=\ssize\textwidth]{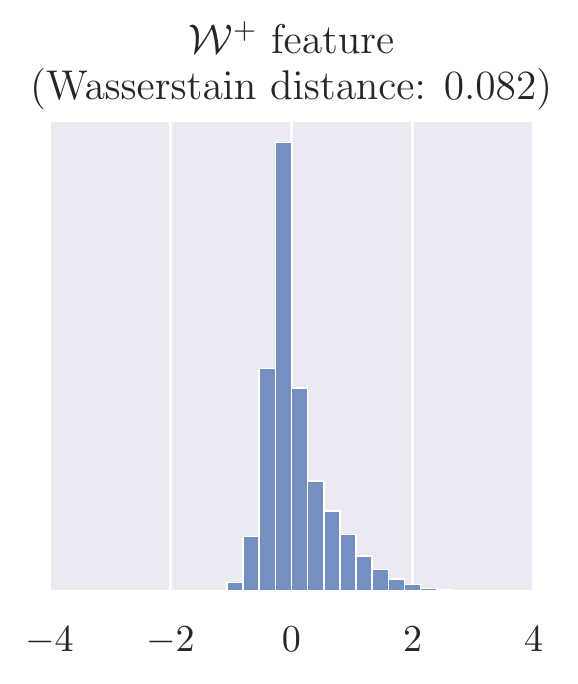} &
        \includegraphics[align=c,width=\ssize\textwidth]{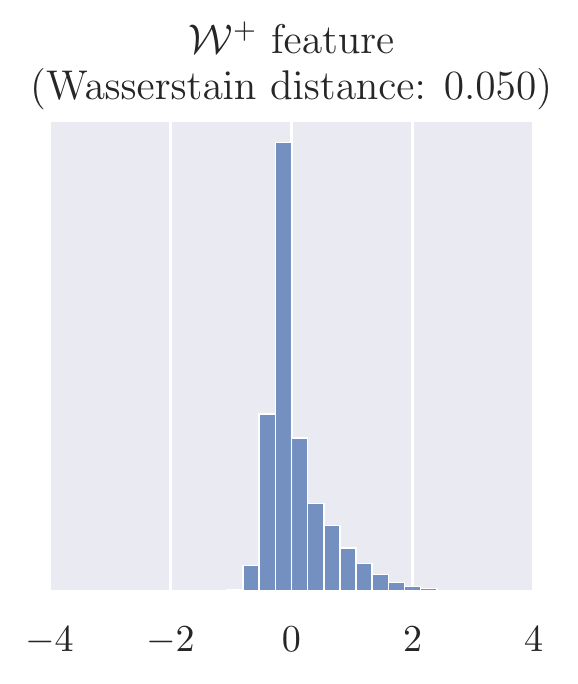} \\
        
        \vspace{0.25em} \\
        & \scriptsize LPIPS: 0.2458 & \scriptsize LPIPS: 0.2536 & \scriptsize LPIPS: 0.2454 & \scriptsize LPIPS: 0.2583 \\
        \vspace{0.25em} \\

        \rotatebox[]{90}{\parbox[c]{\lsize\textwidth}{\centering Projected image into $\NWp$}}\; &
        \includegraphics[align=c,width=\lsize\textwidth]{samples/basic_fit__obama_NWp.jpg} &
        \includegraphics[align=c,width=\lsize\textwidth]{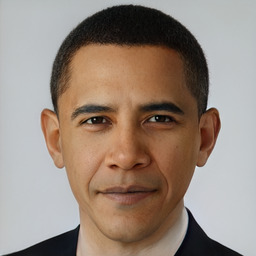} &
        \includegraphics[align=c,width=\lsize\textwidth]{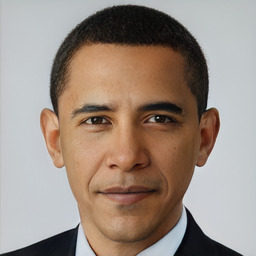} &
        \includegraphics[align=c,width=\lsize\textwidth]{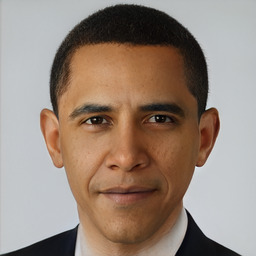} \\
        
        \rotatebox[]{90}{\parbox[c]{\lsize\textwidth}{\centering + Age }}\; &
        \includegraphics[align=c,width=\lsize\textwidth]{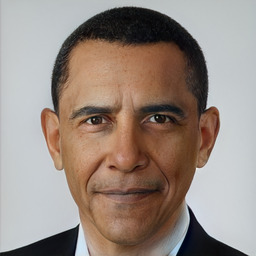} &
        \includegraphics[align=c,width=\lsize\textwidth]{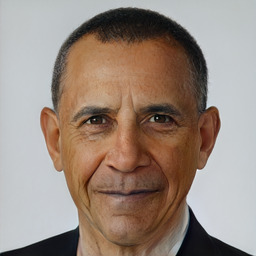} &
        \includegraphics[align=c,width=\lsize\textwidth]{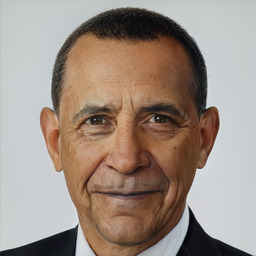} &
        \includegraphics[align=c,width=\lsize\textwidth]{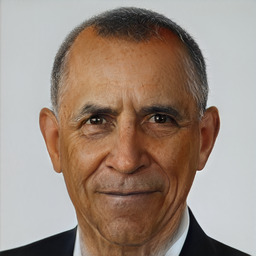} \\
        
        \rotatebox[]{90}{\parbox[c]{\lsize\textwidth}{\centering + Smile }}\; &
        \includegraphics[align=c,width=\lsize\textwidth]{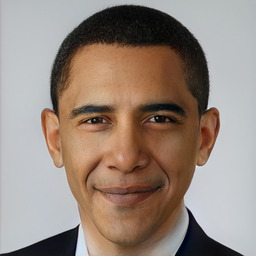} &
        \includegraphics[align=c,width=\lsize\textwidth]{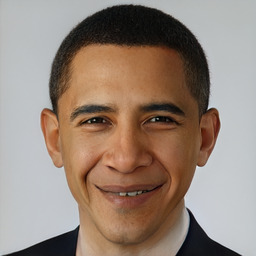} &
        \includegraphics[align=c,width=\lsize\textwidth]{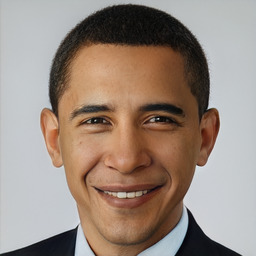} &
        \includegraphics[align=c,width=\lsize\textwidth]{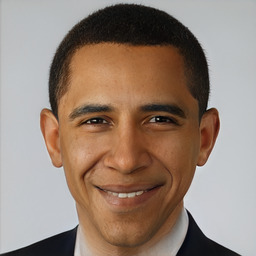} \\
        
        \rotatebox[]{90}{\parbox[c]{\lsize\textwidth}{\centering + Glasses }}\; &
        \includegraphics[align=c,width=\lsize\textwidth]{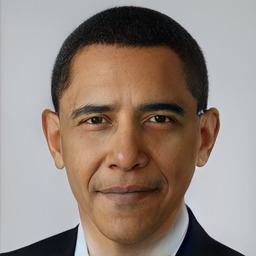} &
        \includegraphics[align=c,width=\lsize\textwidth]{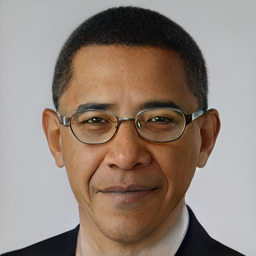} &
        \includegraphics[align=c,width=\lsize\textwidth]{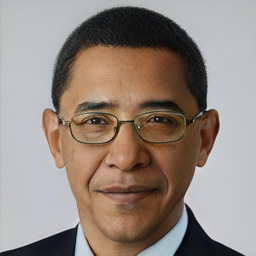} &
        \includegraphics[align=c,width=\lsize\textwidth]{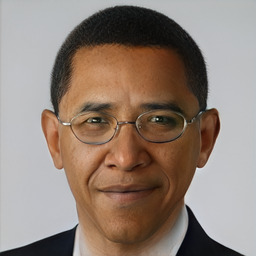} \\
    \end{tabular}
    
    \caption{
        We project images into $\NWp$ with additional loss described in Equation~\ref{eq:dist-reg} on $\Wp$ using coefficient $\lambda_{\text{dist}}$ (different coefficient in different columns).
        We obtain latent representations with distribution of $\Wp$ features very close to the generated images distribution.
        The Wasserstein distance (0.053 in the best case) is almost the same as for the expected generated case (0.047 $\pm$ 0.0284).
        The obtained latent representations have comparable quality and can be meaningfully edited by attribute vectors.
        For a given attribute, all images are edited with the same strength.
        Similarly to Figure~\ref{fig:Wp-equalization}, latent representations using distribution regularization are more sensitive to attribute editions, in particular in the latent representation optimized without distribution regularization, glasses aren't added at all, because of too low strength of the attribute vector.
    }
    \label{fig:Wp-disteq}
\end{figure}

\begin{figure}
    \setlength{\tabcolsep}{0pt}
    \renewcommand{\arraystretch}{0}
    \centering
    
    \def\ssize{0.2}
    \def\lsize{0.2}
    \begin{tabular}{ccccccc}
        & $\lambda_{\text{dist}} = 0$ & $\lambda_{\text{dist}} = 0.1$ & $\lambda_{\text{dist}} = 0.2$ & $\lambda_{\text{dist}} = 0.5$ \\
        \vspace{0.5em} \\
        
        &
        \includegraphics[align=c,width=\ssize\textwidth]{plots/Wp_trump_NWp.pdf} &
        \includegraphics[align=c,width=\ssize\textwidth]{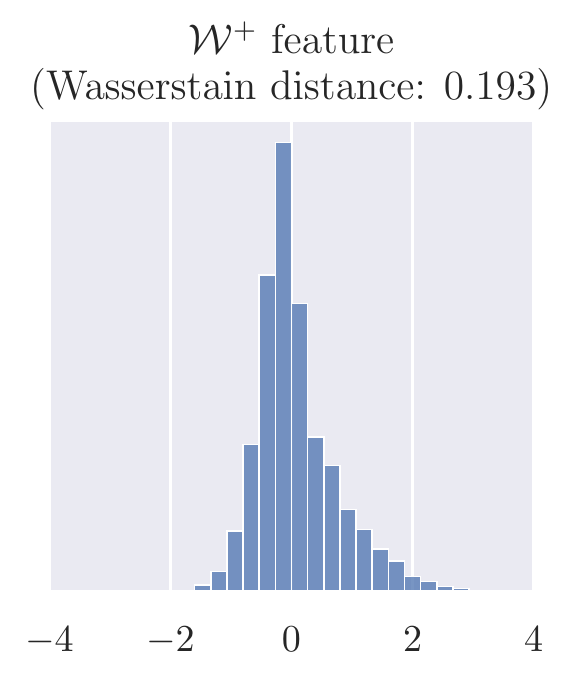} &
        \includegraphics[align=c,width=\ssize\textwidth]{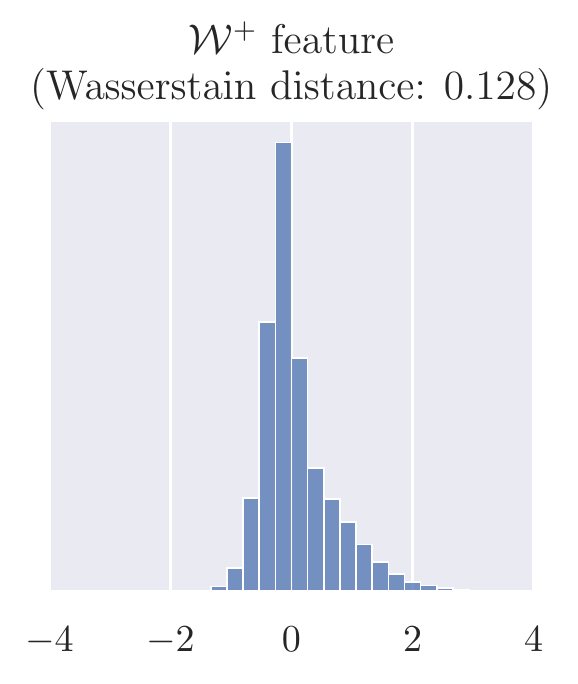} &
        \includegraphics[align=c,width=\ssize\textwidth]{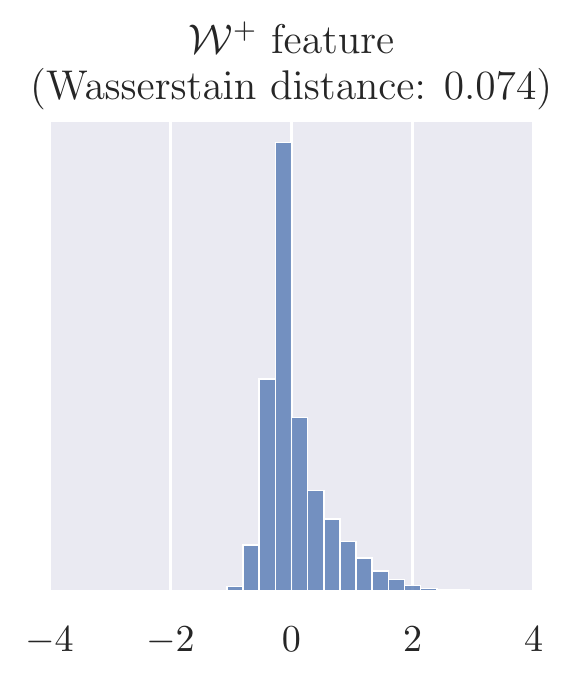} \\
        
        \vspace{0.25em} \\
        & \scriptsize LPIPS: 0.3330 & \scriptsize LPIPS: 0.3433 & \scriptsize LPIPS: 0.3415 & \scriptsize LPIPS: 0.3501
        \vspace{0.25em} \\

        \rotatebox[]{90}{\parbox[c]{\lsize\textwidth}{\centering Projected image}}\; &
        \includegraphics[align=c,width=\lsize\textwidth]{samples/basic_fit__trump_NWp.jpg} &
        \includegraphics[align=c,width=\lsize\textwidth]{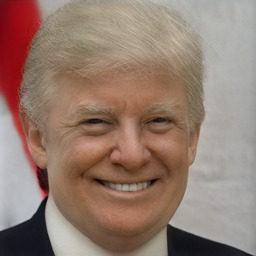} &
        \includegraphics[align=c,width=\lsize\textwidth]{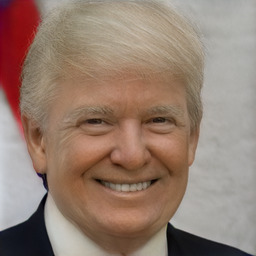} &
        \includegraphics[align=c,width=\lsize\textwidth]{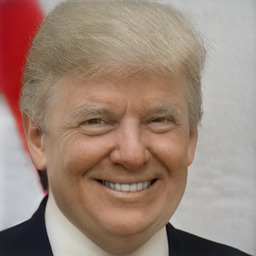} \\
        
        \rotatebox[]{90}{\parbox[c]{\lsize\textwidth}{\centering + Age }}\; &
        \includegraphics[align=c,width=\lsize\textwidth]{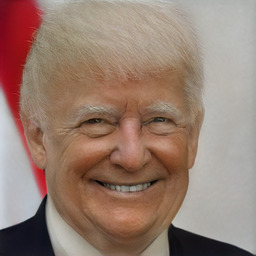} &
        \includegraphics[align=c,width=\lsize\textwidth]{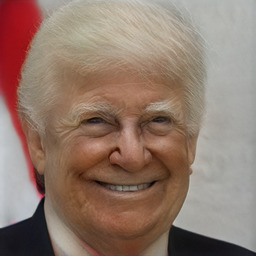} &
        \includegraphics[align=c,width=\lsize\textwidth]{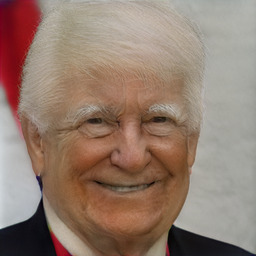} &
        \includegraphics[align=c,width=\lsize\textwidth]{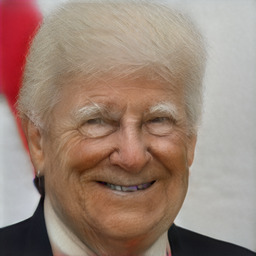} \\
        
        \rotatebox[]{90}{\parbox[c]{\lsize\textwidth}{\centering - Smile }}\; &
        \includegraphics[align=c,width=\lsize\textwidth]{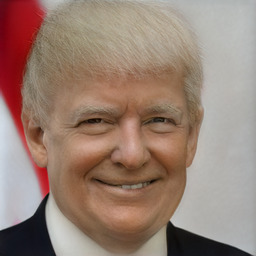} &
        \includegraphics[align=c,width=\lsize\textwidth]{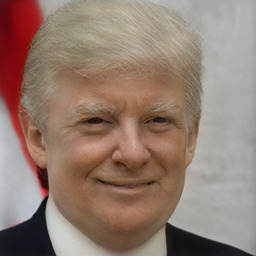} &
        \includegraphics[align=c,width=\lsize\textwidth]{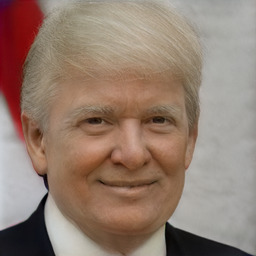} &
        \includegraphics[align=c,width=\lsize\textwidth]{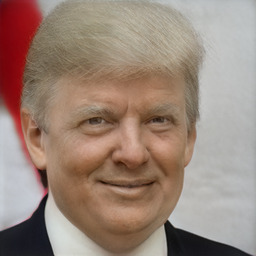} \\
        
        \rotatebox[]{90}{\parbox[c]{\lsize\textwidth}{\centering + Eyes open }}\; &
        \includegraphics[align=c,width=\lsize\textwidth]{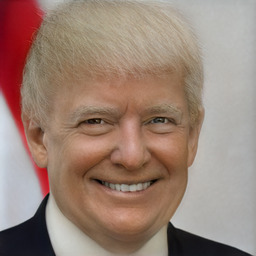} &
        \includegraphics[align=c,width=\lsize\textwidth]{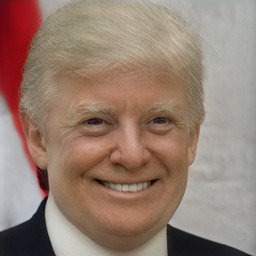} &
        \includegraphics[align=c,width=\lsize\textwidth]{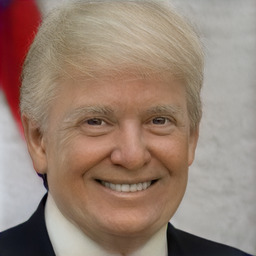} &
        \includegraphics[align=c,width=\lsize\textwidth]{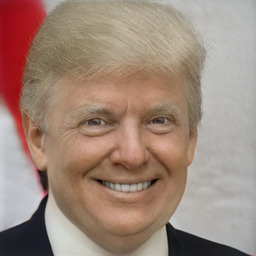} \\
    \end{tabular}
    
    \caption{
        The same as Figure~\ref{fig:Wp-disteq}, but with different image.
    }
    \label{fig:Wp-disteq2}
\end{figure}

\begin{figure}
    \setlength{\tabcolsep}{0pt}
    \renewcommand{\arraystretch}{0}
    \centering
    
    \def\ssize{0.133}
    \def\lsize{0.133}
    \begin{tabular}{cccccc}
        & \parbox[c]{\lsize\textwidth}{\centering \small $\lambda_{\text{dist}} = 0$   \\ \scriptsize 250 iterations}
        & \parbox[c]{\lsize\textwidth}{\centering \small $\lambda_{\text{dist}} = 0$   \\ \scriptsize till convergence}
        & \parbox[c]{\lsize\textwidth}{\centering \small $\lambda_{\text{dist}} = 0.2$ \\ \scriptsize 250 iterations}
        & \parbox[c]{\lsize\textwidth}{\centering \small $\lambda_{\text{dist}} = 0.2$ \\ \scriptsize 1000 iterations}
        & \parbox[c]{\lsize\textwidth}{\centering \small $\lambda_{\text{dist}} = 0.2$ \\ \scriptsize till convergence}
        \\
        \vspace{0.5em} \\
        
        &
        \includegraphics[align=c,width=\ssize\textwidth]{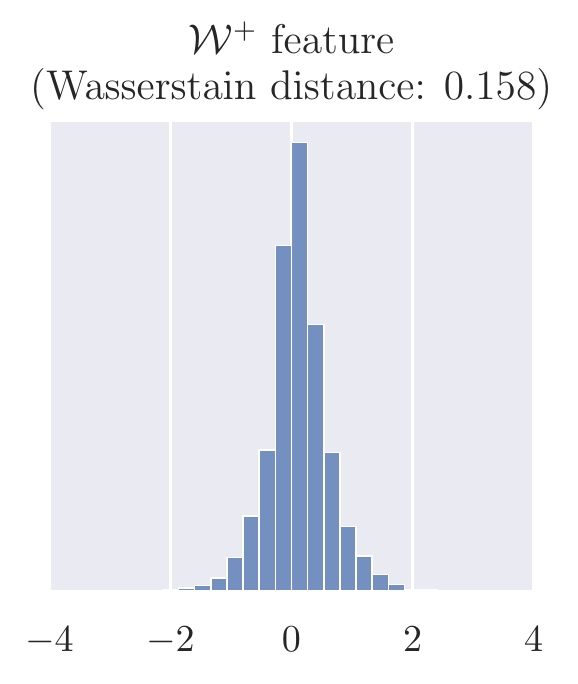} &
        \includegraphics[align=c,width=\ssize\textwidth]{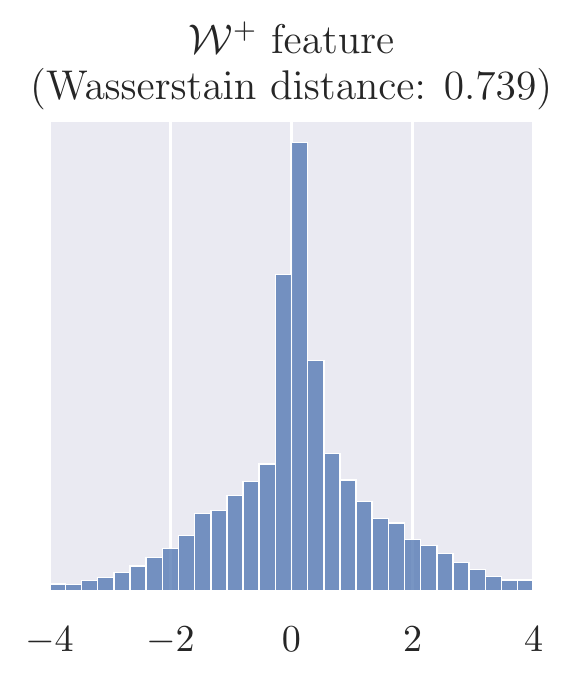} &
        \includegraphics[align=c,width=\ssize\textwidth]{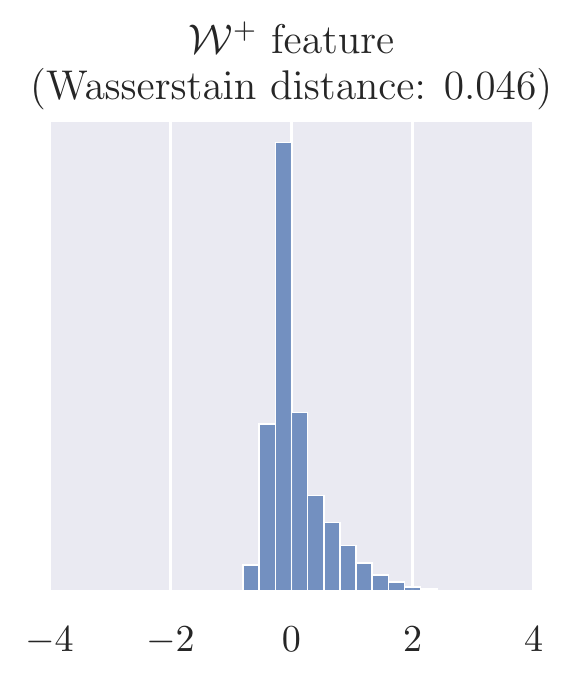} &
        \includegraphics[align=c,width=\ssize\textwidth]{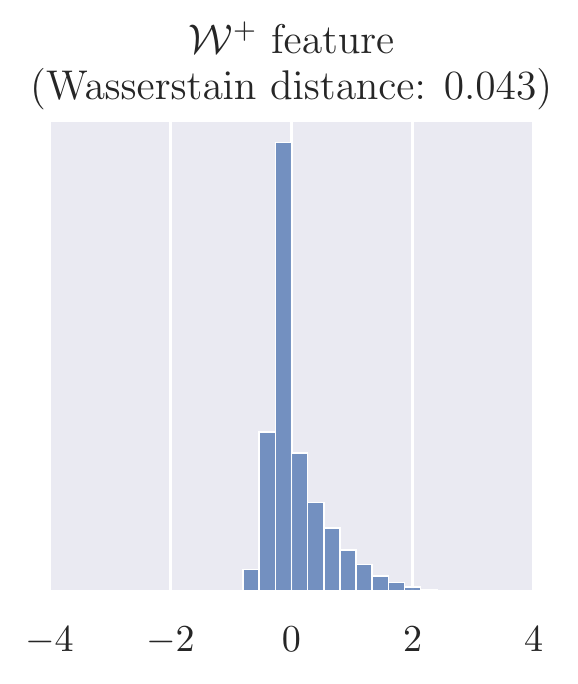} &
        \includegraphics[align=c,width=\ssize\textwidth]{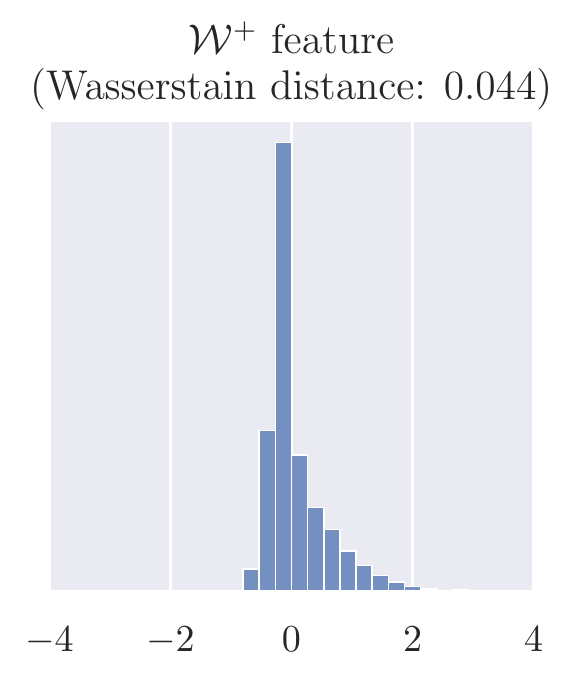} \\
        
        &
        \includegraphics[align=c,width=\ssize\textwidth]{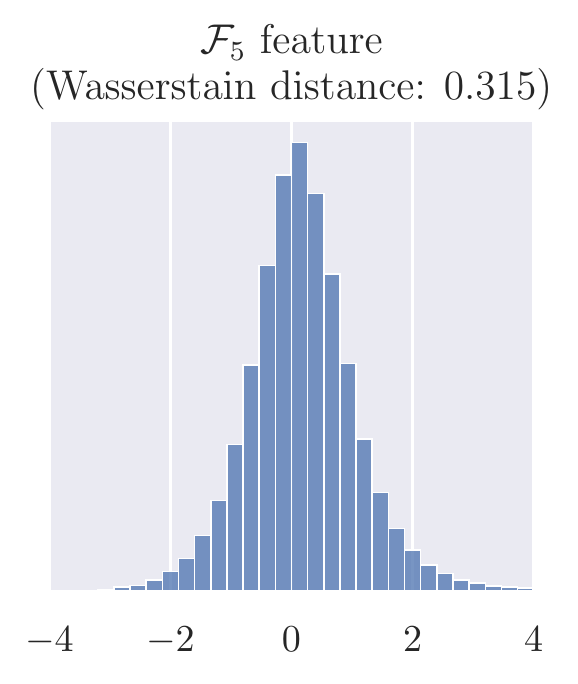} &
        \includegraphics[align=c,width=\ssize\textwidth]{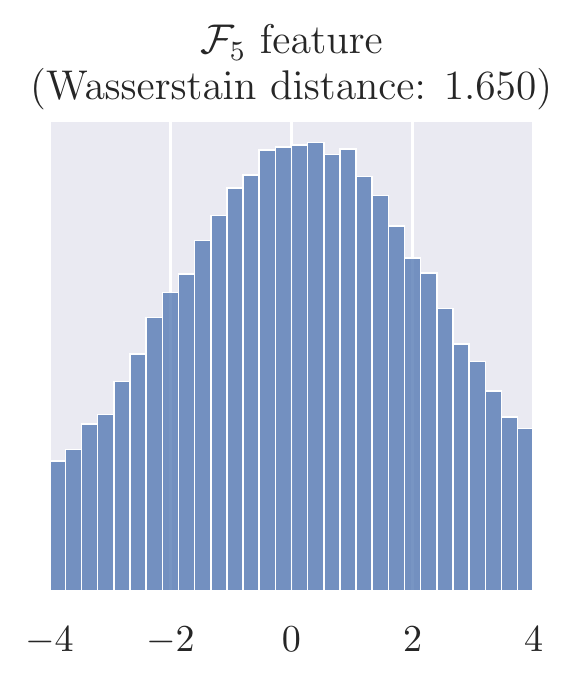} &
        \includegraphics[align=c,width=\ssize\textwidth]{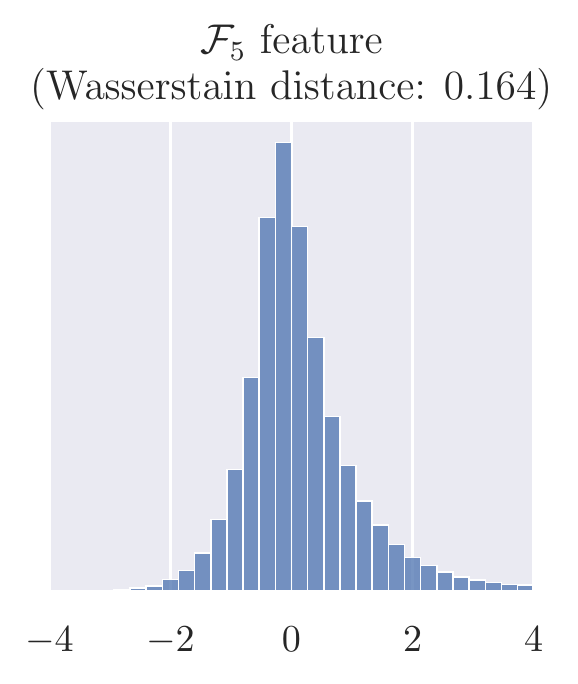} &
        \includegraphics[align=c,width=\ssize\textwidth]{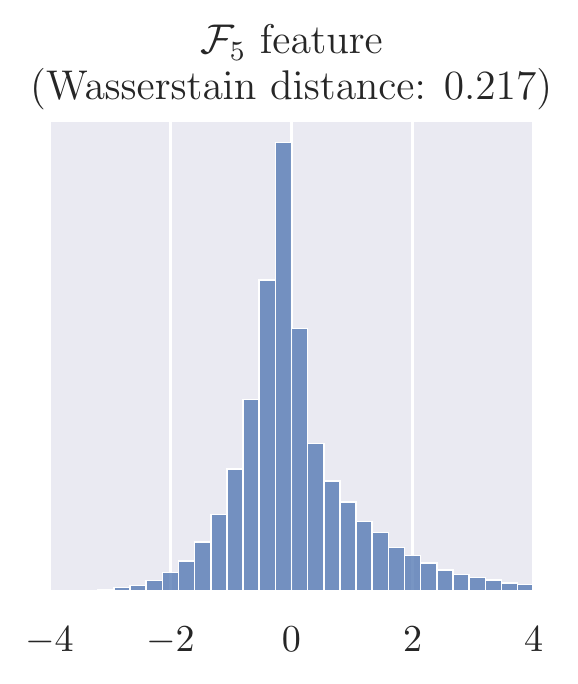} &
        \includegraphics[align=c,width=\ssize\textwidth]{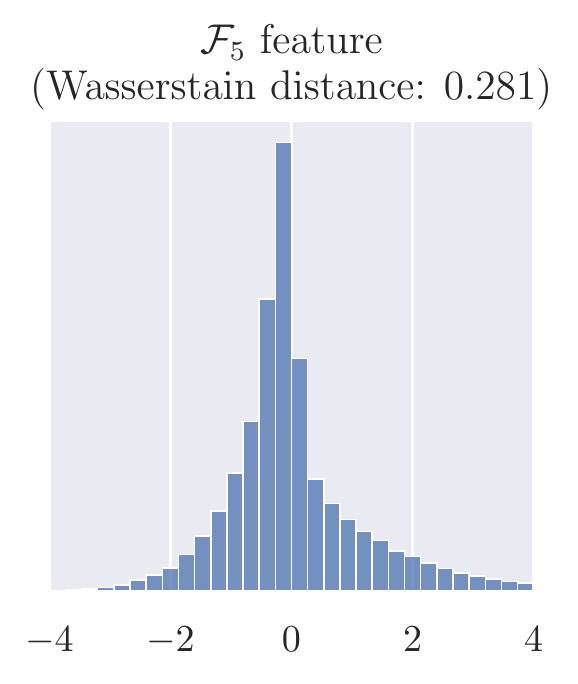} \\
        
        &
        \includegraphics[align=c,width=\ssize\textwidth]{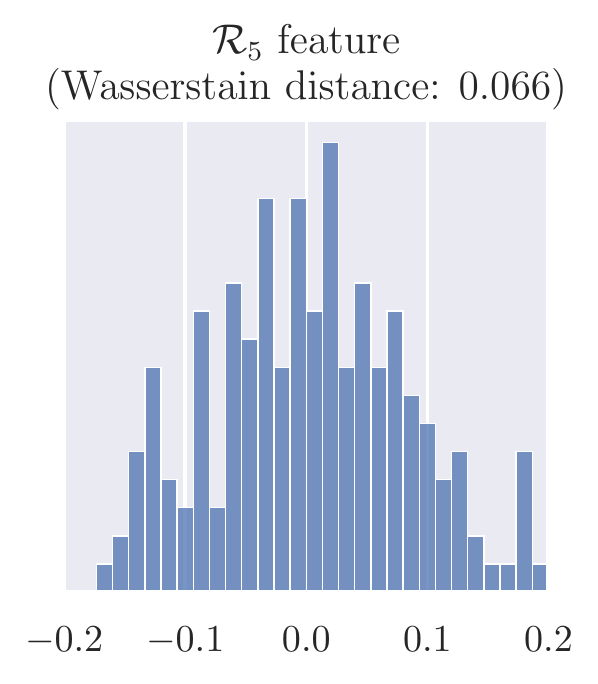} &
        \includegraphics[align=c,width=\ssize\textwidth]{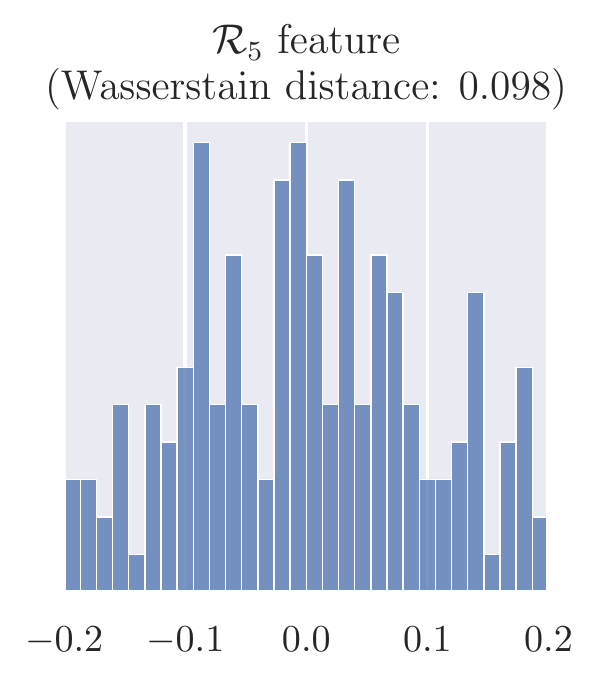} &
        \includegraphics[align=c,width=\ssize\textwidth]{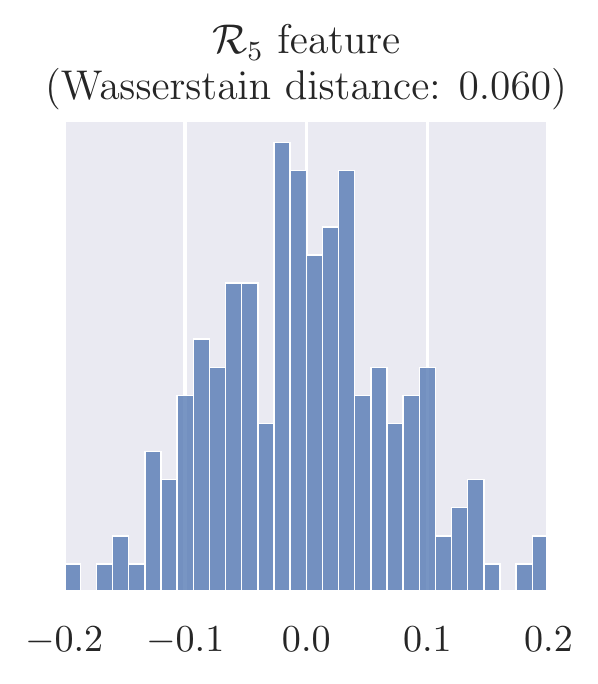} &
        \includegraphics[align=c,width=\ssize\textwidth]{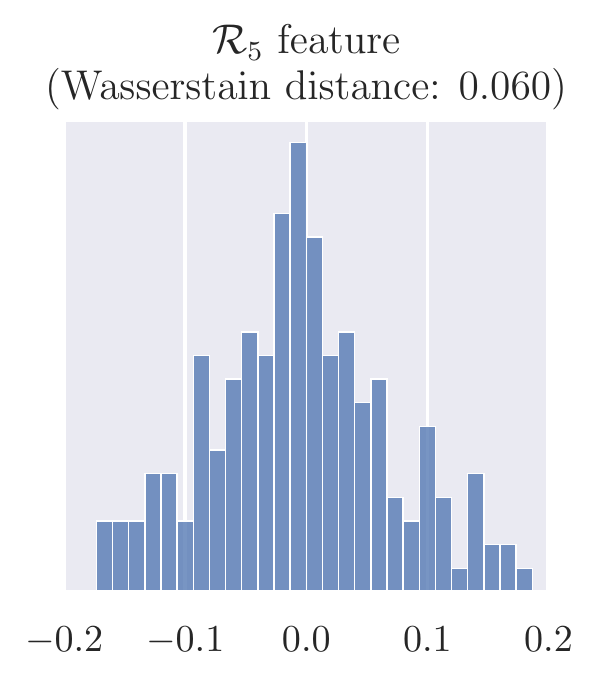} &
        \includegraphics[align=c,width=\ssize\textwidth]{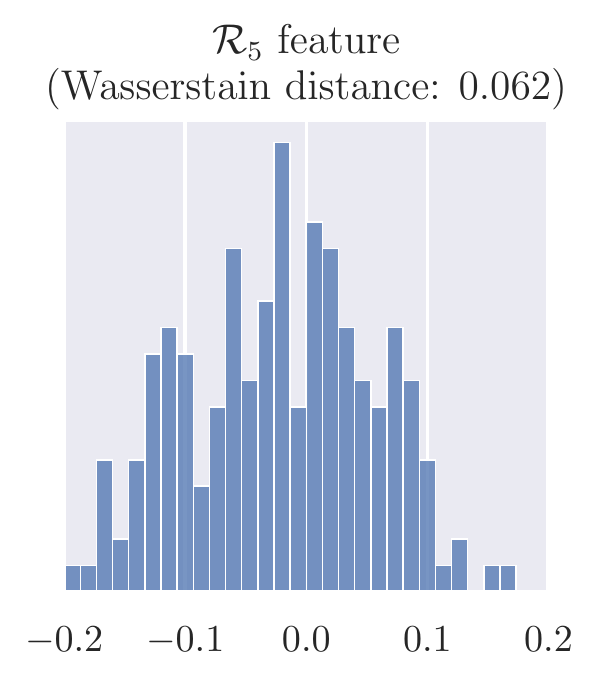} \\
        
        \vspace{0.25em} \\
        & \scriptsize LPIPS: 0.2076 & \scriptsize LPIPS: 0.1614 & \scriptsize LPIPS: 0.2182 & \scriptsize LPIPS: 0.1887 & \scriptsize LPIPS: 0.1712 \\
        \vspace{0.25em} \\
        
        \rotatebox[]{90}{\parbox[c]{\lsize\textwidth}{\centering Projected image}}\; &
        \includegraphics[align=c,width=\lsize\textwidth]{samples/basic_fit__obama_F5NWp.jpg} &
        \includegraphics[align=c,width=\lsize\textwidth]{samples/basic_fit__obama_F5NWp_convergence.jpg} &
        \includegraphics[align=c,width=\lsize\textwidth]{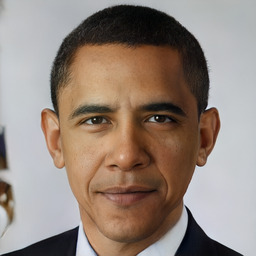} &
        \includegraphics[align=c,width=\lsize\textwidth]{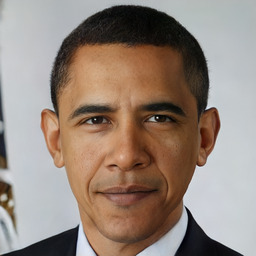} &
        \includegraphics[align=c,width=\lsize\textwidth]{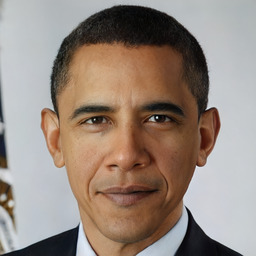} \\
        
        \rotatebox[]{90}{\parbox[c]{\lsize\textwidth}{\centering + Age }}\; &
        \includegraphics[align=c,width=\lsize\textwidth]{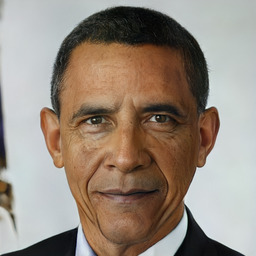} &
        \includegraphics[align=c,width=\lsize\textwidth]{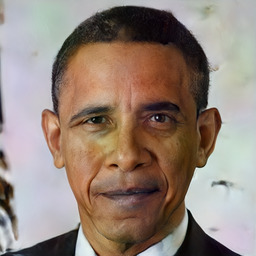} &
        \includegraphics[align=c,width=\lsize\textwidth]{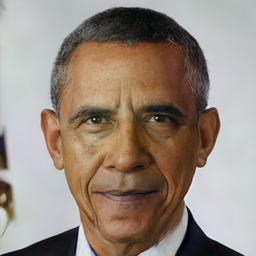} &
        \includegraphics[align=c,width=\lsize\textwidth]{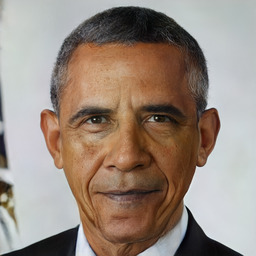} &
        \includegraphics[align=c,width=\lsize\textwidth]{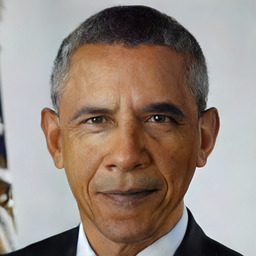} \\
        
        \rotatebox[]{90}{\parbox[c]{\lsize\textwidth}{\centering - Age }}\; &
        \includegraphics[align=c,width=\lsize\textwidth]{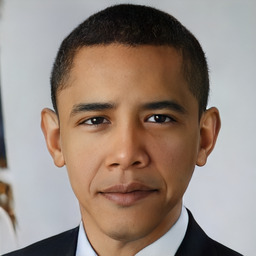} &
        \includegraphics[align=c,width=\lsize\textwidth]{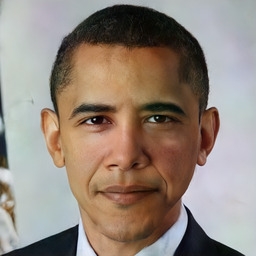} &
        \includegraphics[align=c,width=\lsize\textwidth]{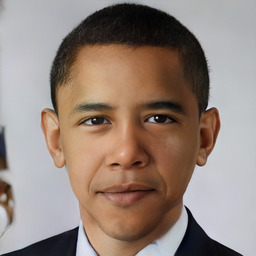} &
        \includegraphics[align=c,width=\lsize\textwidth]{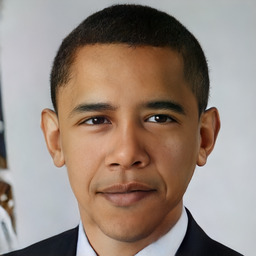} &
        \includegraphics[align=c,width=\lsize\textwidth]{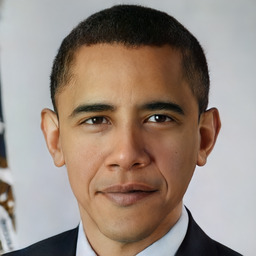} \\
        
        \rotatebox[]{90}{\parbox[c]{\lsize\textwidth}{\centering + Smile }}\; &
        \includegraphics[align=c,width=\lsize\textwidth]{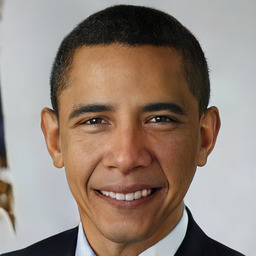} &
        \includegraphics[align=c,width=\lsize\textwidth]{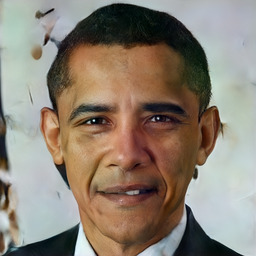} &
        \includegraphics[align=c,width=\lsize\textwidth]{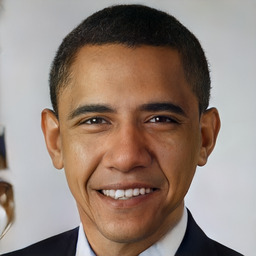} &
        \includegraphics[align=c,width=\lsize\textwidth]{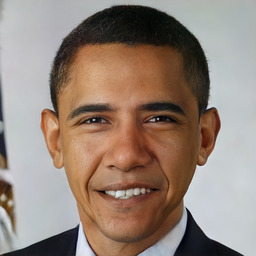} &
        \includegraphics[align=c,width=\lsize\textwidth]{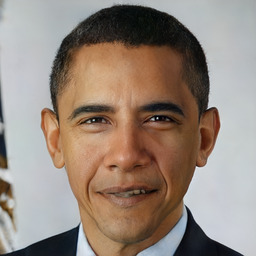} \\
        
        \rotatebox[]{90}{\parbox[c]{\lsize\textwidth}{\centering + Eyes closed }}\; &
        \includegraphics[align=c,width=\lsize\textwidth]{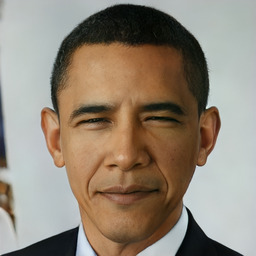} &
        \includegraphics[align=c,width=\lsize\textwidth]{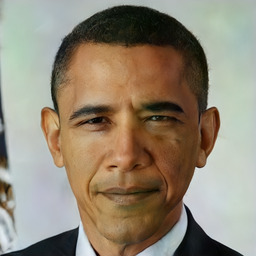} &
        \includegraphics[align=c,width=\lsize\textwidth]{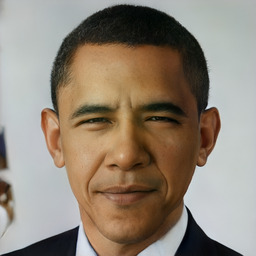} &
        \includegraphics[align=c,width=\lsize\textwidth]{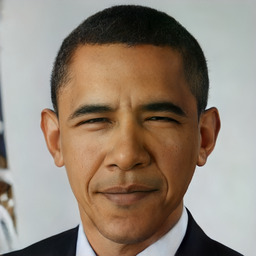} &
        \includegraphics[align=c,width=\lsize\textwidth]{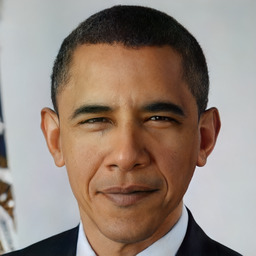} \\
    \end{tabular}
    
    \caption{
        We project images into $\FNWp{5}$ with distribution regularization on $\Wp$, $\mathcal{F}_5$ and $\mathcal{R}_5$ using coefficient $\lambda_{\text{dist}}$.
        Distribution preservation is especially visible for $\mathcal{F}_5$.
        $\mathcal{F}_5$ distribution discrepancy for the latent representation optimized till convergence without distribution regularization (2nd column) is a presumed culprit for not capturing semantic information and problems with attribute editing. Such features are out-of-sample for attribute model, and may not generalize well for changes in style layers.
        We can train longer without losing ability for attribute editing.
        Obtained latent representations optimized with distribution regularization have better quality (up to 17.5\% LPIPS improvement).
        Compared to projections into the $\NWp$ space (LPIPS: 0.2458), LPIPS score improves by 30\%.
    }
    \label{fig:FWp-disteq}
\end{figure}

\begin{figure}
    \setlength{\tabcolsep}{0pt}
    \renewcommand{\arraystretch}{0}
    \centering
    
    \def\ssize{0.14}
    \def\lsize{0.14}
    \begin{tabular}{cccccc}
        & \parbox[c]{\lsize\textwidth}{\centering \small $\lambda_{\text{dist}} = 0$   \\ \scriptsize 250 iterations}
        & \parbox[c]{\lsize\textwidth}{\centering \small $\lambda_{\text{dist}} = 0$   \\ \scriptsize till convergence}
        & \parbox[c]{\lsize\textwidth}{\centering \small $\lambda_{\text{dist}} = 0.2$ \\ \scriptsize 250 iterations}
        & \parbox[c]{\lsize\textwidth}{\centering \small $\lambda_{\text{dist}} = 0.2$ \\ \scriptsize 1000 iterations}
        & \parbox[c]{\lsize\textwidth}{\centering \small $\lambda_{\text{dist}} = 0.2$ \\ \scriptsize till convergence}
        \\
        \vspace{0.5em} \\
        
        &
        \includegraphics[align=c,width=\ssize\textwidth]{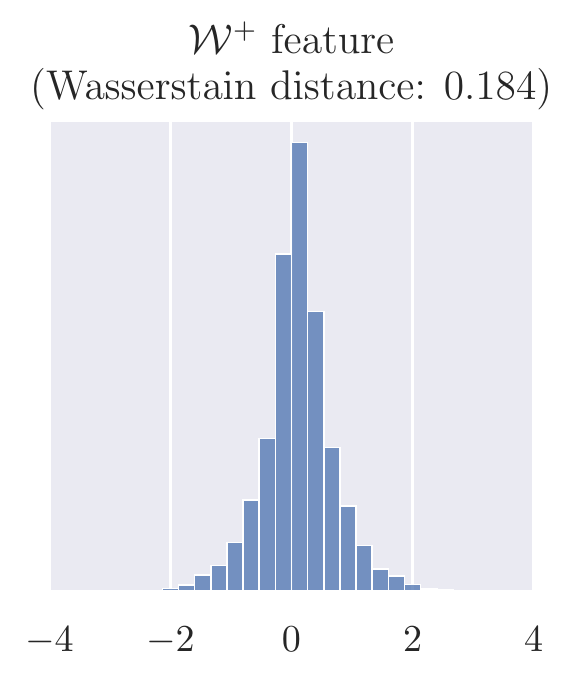} &
        \includegraphics[align=c,width=\ssize\textwidth]{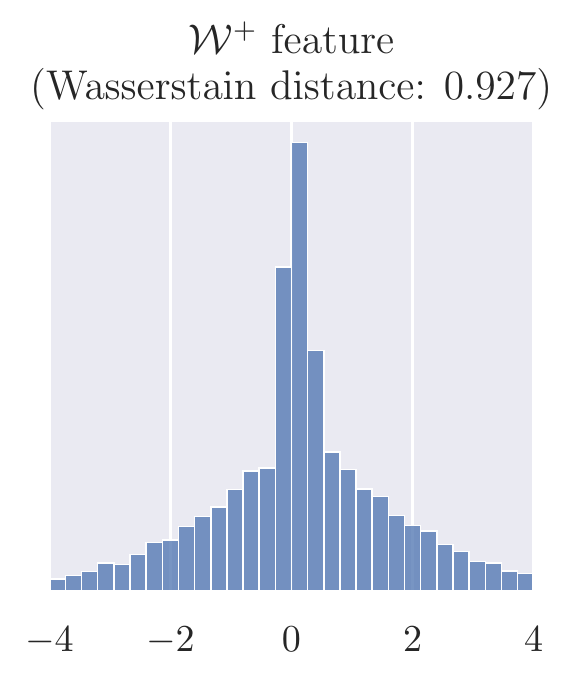} &
        \includegraphics[align=c,width=\ssize\textwidth]{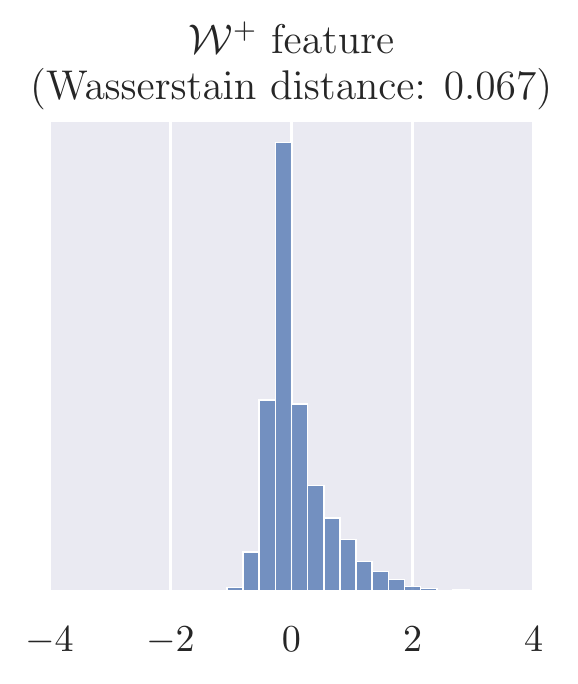} &
        \includegraphics[align=c,width=\ssize\textwidth]{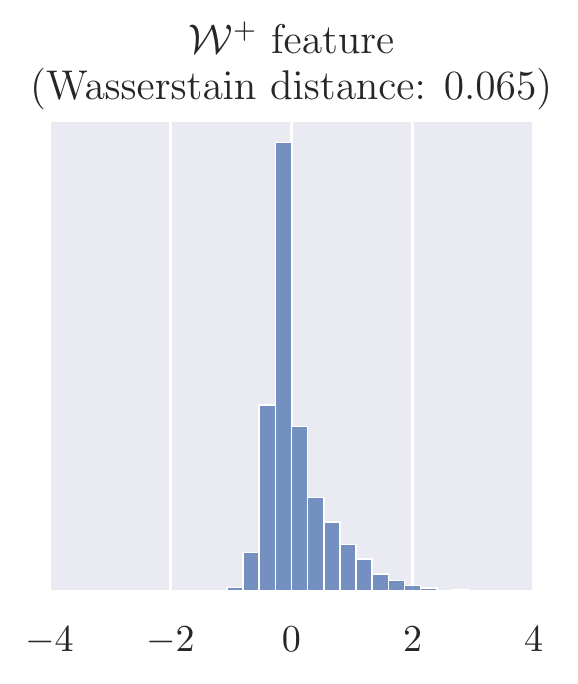} &
        \includegraphics[align=c,width=\ssize\textwidth]{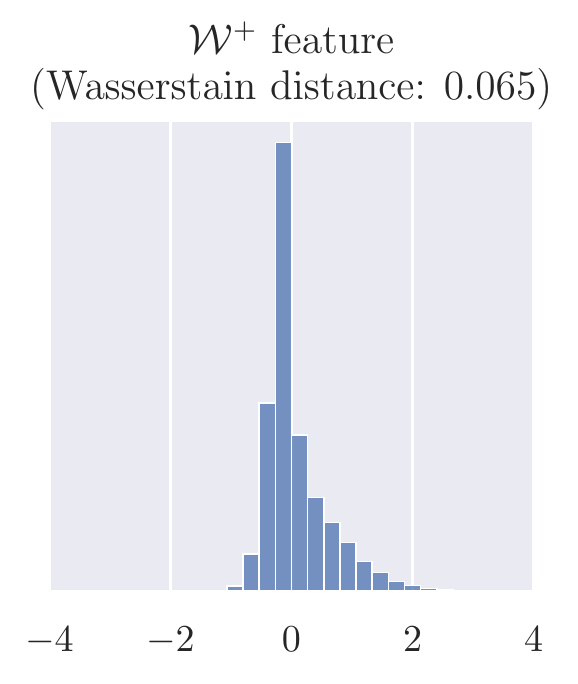} \\
        
        &
        \includegraphics[align=c,width=\ssize\textwidth]{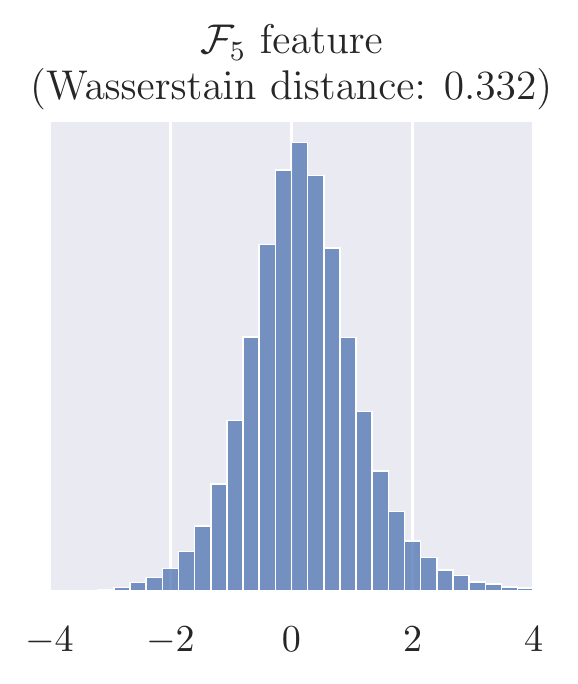} &
        \includegraphics[align=c,width=\ssize\textwidth]{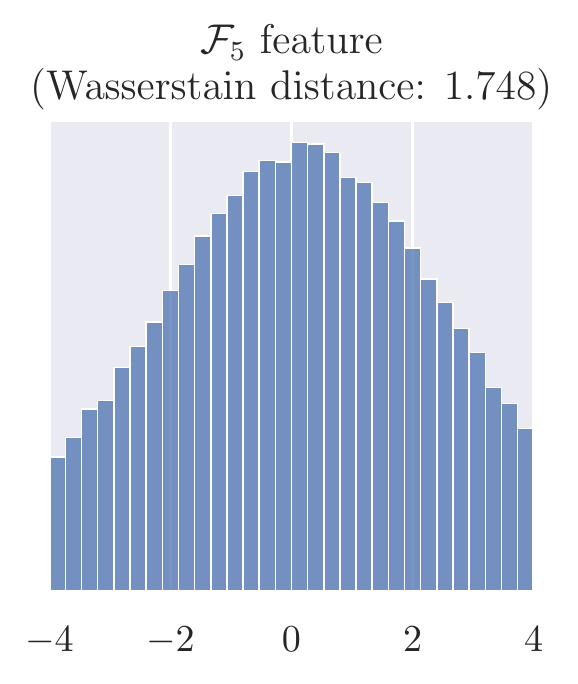} &
        \includegraphics[align=c,width=\ssize\textwidth]{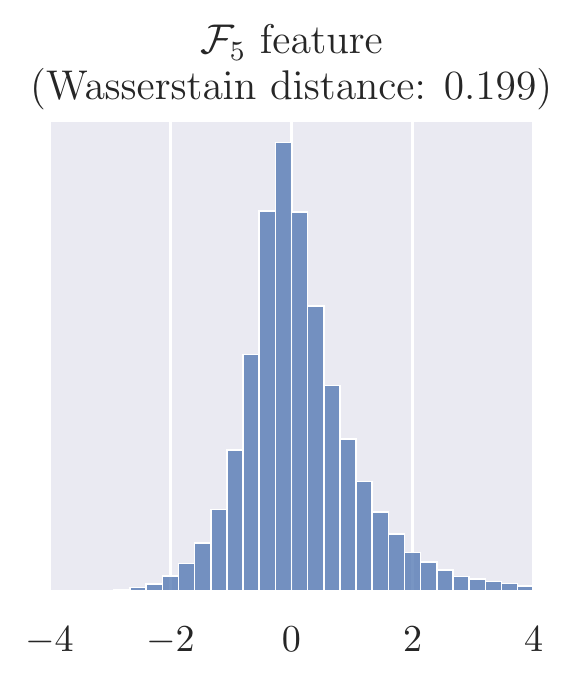} &
        \includegraphics[align=c,width=\ssize\textwidth]{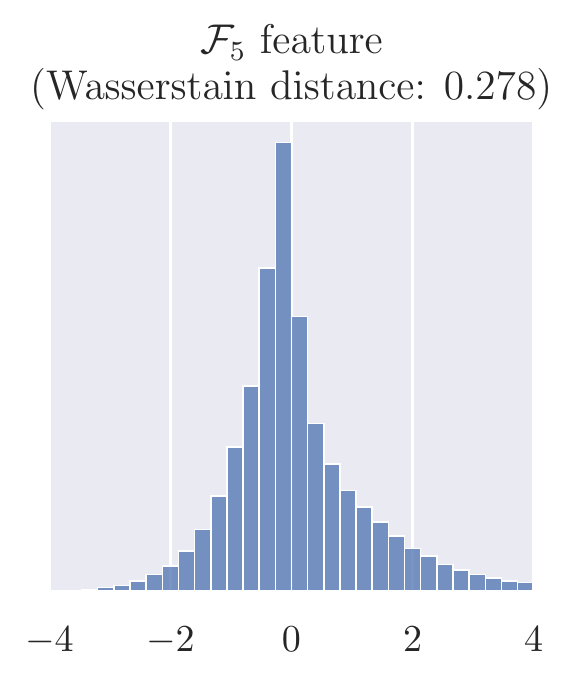} &
        \includegraphics[align=c,width=\ssize\textwidth]{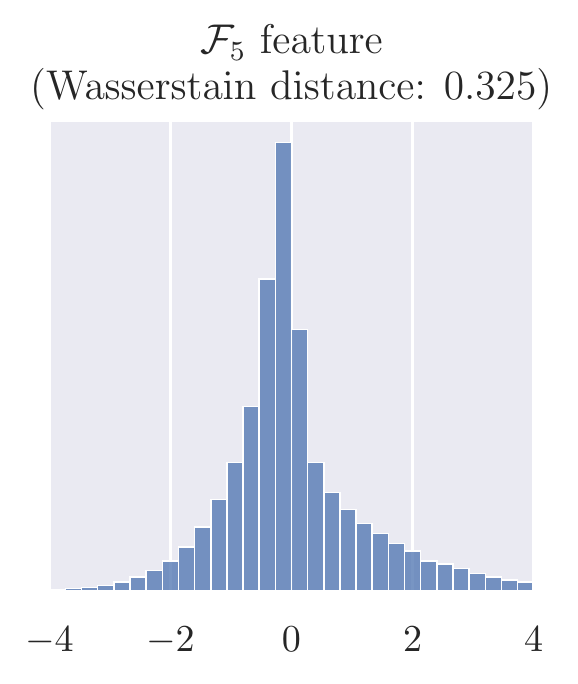} \\
        
        &
        \includegraphics[align=c,width=\ssize\textwidth]{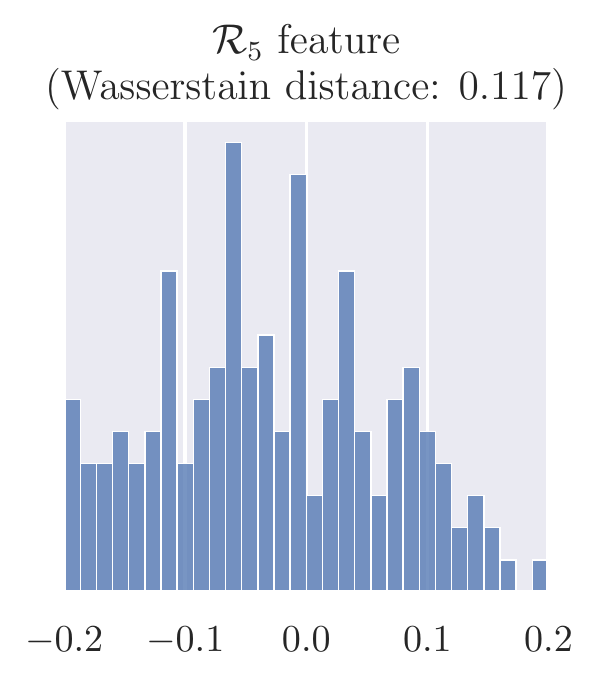} &
        \includegraphics[align=c,width=\ssize\textwidth]{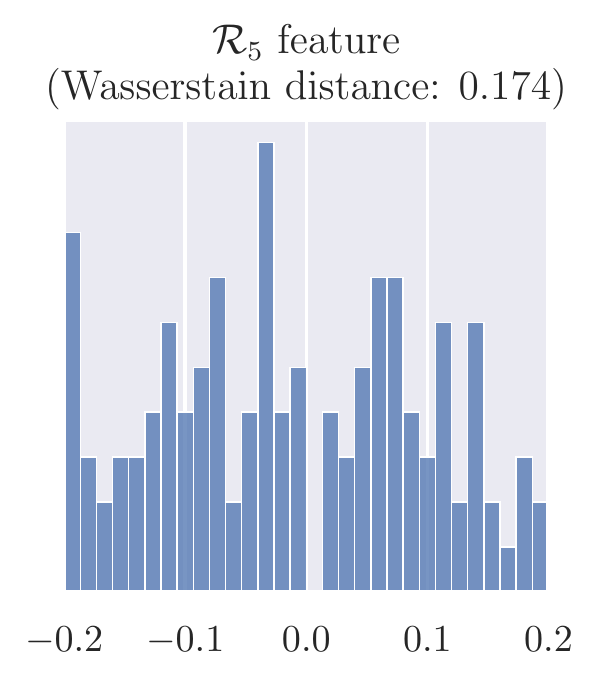} &
        \includegraphics[align=c,width=\ssize\textwidth]{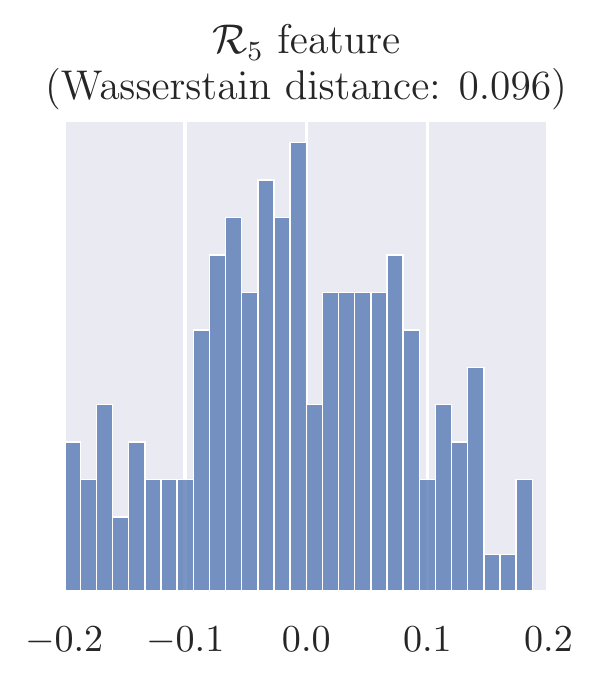} &
        \includegraphics[align=c,width=\ssize\textwidth]{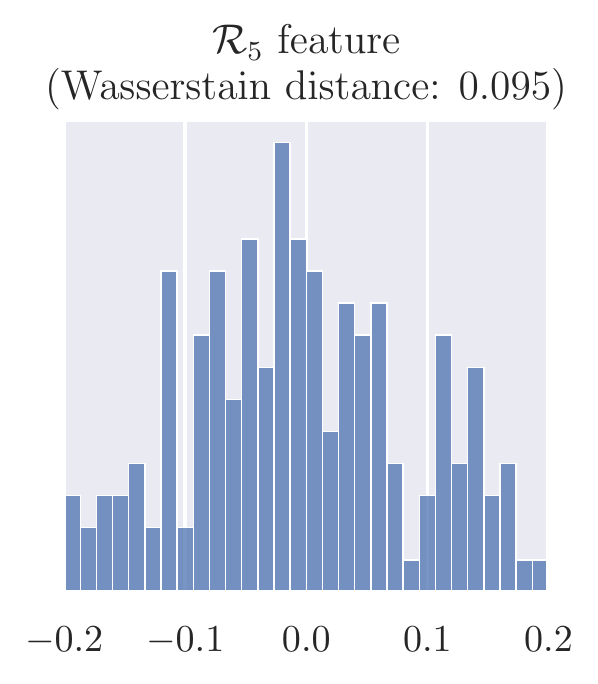} &
        \includegraphics[align=c,width=\ssize\textwidth]{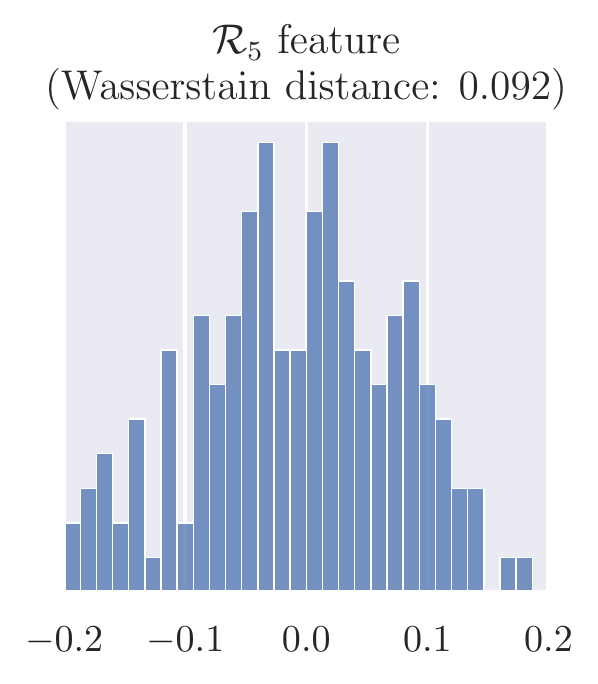} \\
        
        \vspace{0.25em} \\
        & \scriptsize LPIPS: 0.2985 & \scriptsize LPIPS: 0.2430 & \scriptsize LPIPS: 0.3132 & \scriptsize LPIPS: 0.2767 & \scriptsize LPIPS: 0.2496 \\
        \vspace{0.25em} \\
        
        \rotatebox[]{90}{\parbox[c]{\lsize\textwidth}{\centering Projected image}}\; &
        \includegraphics[align=c,width=\lsize\textwidth]{samples/basic_fit__trump_F5NWp.jpg} &
        \includegraphics[align=c,width=\lsize\textwidth]{samples/basic_fit__trump_F5NWp_convergence.jpg} &
        \includegraphics[align=c,width=\lsize\textwidth]{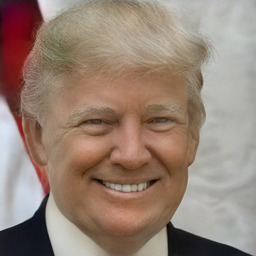} &
        \includegraphics[align=c,width=\lsize\textwidth]{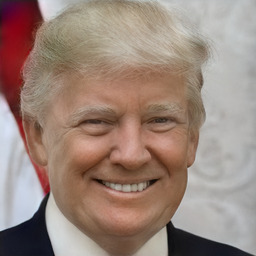} &
        \includegraphics[align=c,width=\lsize\textwidth]{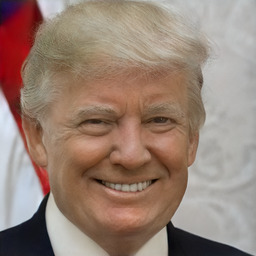} \\
        
        \rotatebox[]{90}{\parbox[c]{\lsize\textwidth}{\centering + Age }}\; &
        \includegraphics[align=c,width=\lsize\textwidth]{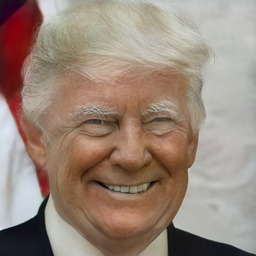} &
        \includegraphics[align=c,width=\lsize\textwidth]{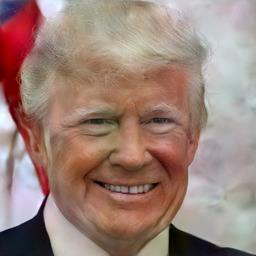} &
        \includegraphics[align=c,width=\lsize\textwidth]{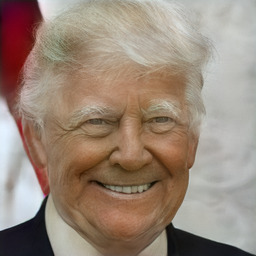} &
        \includegraphics[align=c,width=\lsize\textwidth]{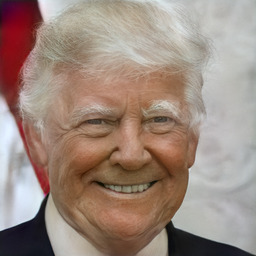} &
        \includegraphics[align=c,width=\lsize\textwidth]{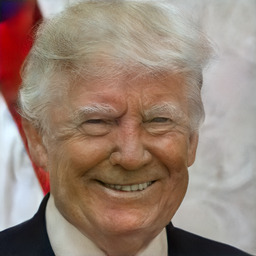} \\
        
        \rotatebox[]{90}{\parbox[c]{\lsize\textwidth}{\centering - Age }}\; &
        \includegraphics[align=c,width=\lsize\textwidth]{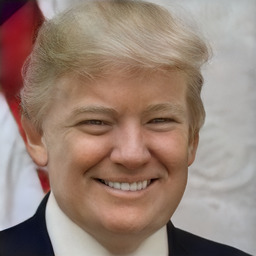} &
        \includegraphics[align=c,width=\lsize\textwidth]{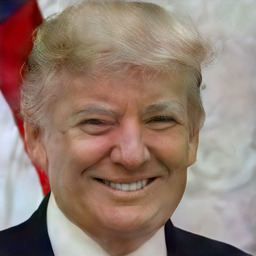} &
        \includegraphics[align=c,width=\lsize\textwidth]{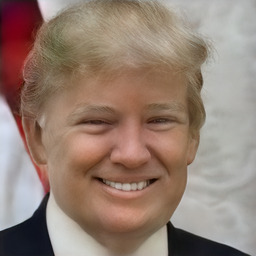} &
        \includegraphics[align=c,width=\lsize\textwidth]{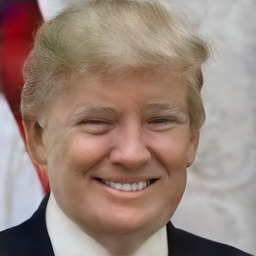} &
        \includegraphics[align=c,width=\lsize\textwidth]{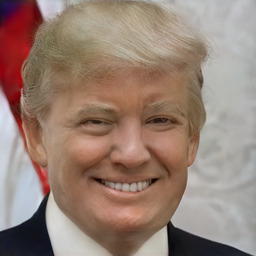} \\
        
        \rotatebox[]{90}{\parbox[c]{\lsize\textwidth}{\centering - Smile }}\; &
        \includegraphics[align=c,width=\lsize\textwidth]{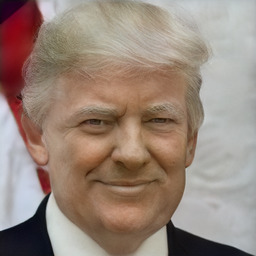} &
        \includegraphics[align=c,width=\lsize\textwidth]{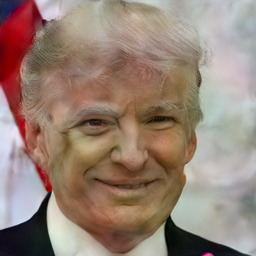} &
        \includegraphics[align=c,width=\lsize\textwidth]{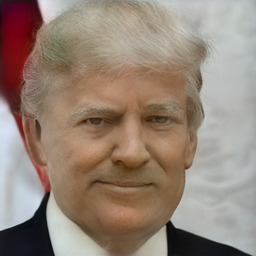} &
        \includegraphics[align=c,width=\lsize\textwidth]{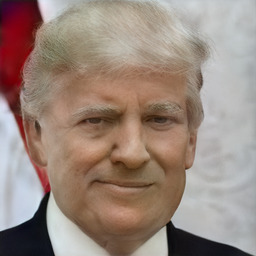} &
        \includegraphics[align=c,width=\lsize\textwidth]{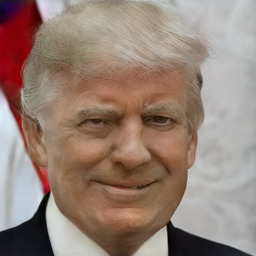} \\
        
        \rotatebox[]{90}{\parbox[c]{\lsize\textwidth}{\centering + Eyes opened }}\; &
        \includegraphics[align=c,width=\lsize\textwidth]{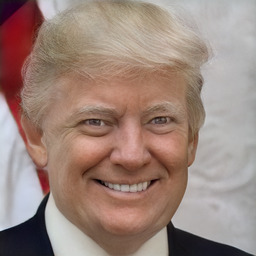} &
        \includegraphics[align=c,width=\lsize\textwidth]{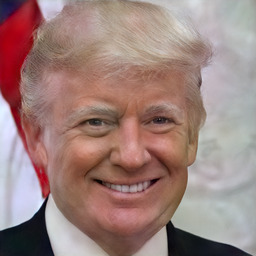} &
        \includegraphics[align=c,width=\lsize\textwidth]{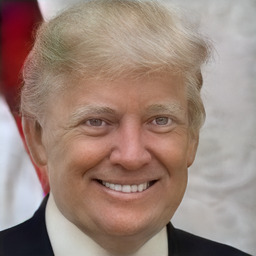} &
        \includegraphics[align=c,width=\lsize\textwidth]{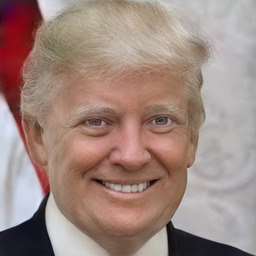} &
        \includegraphics[align=c,width=\lsize\textwidth]{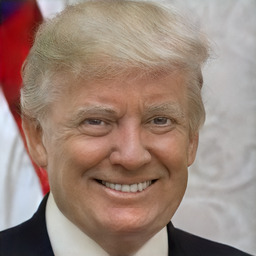} \\
    \end{tabular}
    
    \caption{
        The same as Figure~\ref{fig:FWp-disteq}, but with different image.
    }
    \label{fig:FWp-disteq2}
\end{figure}

%
%
%
%
%
%
%
%
%

\chapter{GAN Training on Spatial Representations}
\label{chap:train}
In this chapter, we train StyleGAN2 models with spatial latent spaces.
We propose to use a training procedure from StyleGAN2-ADA \cite{DBLP:journals/corr/abs-2006-06676},
extending latent representations by $\mathcal{F}_i$ and $\mathcal{R}_i$, \ie we train in the $\mathcal{FNZ}_i$ space.
Our idea is to sample feature maps from a simple distribution, \eg $\mathcal{N}\funarg{0, I}$.
Because the modified generator is a fully convolutional neural network,
and the input feature map is sampled in a spatially independent way,
distant objects on the output image also should be independent.

\section{Approach}

As we showed in previous sections, even early feature maps keep semantically meaningful signal.  
For face generation task, independent sampling in feature maps prevents the model from forming and arranging parts on their expected spatial positions,
\eg eyes should be at the same height, nose should be between eyes, \etc.
However for some dataset, such properties are not needed. For example satellite image datasets, which can be translated and rotated freely, and distant objects on images are usually visually not dependent on each other.

There are three main reasons why we suspect that such representation may be more suitable for such datasets:
\begin{itemize}
    \item In original StyleGAN2, the first feature map is learned, therefore it may contain some
        unwanted semantic information, \eg center of images having different properties than edge of images.
        Representations containing $\mathcal{F}_i$ and $\mathcal{R}_i$ are equivariant to translation.
        We show an example of a related artifact for the non-spatial model in Figure~\ref{fig:train-vanilla-background}.
    \item Spatial representations have more degrees of freedom, having ability to cover the input distribution better.
        For example the original StyleGAN2 training latent space $\mathcal{Z}$ has only 512 parameters whereas
        $\mathcal{F}_3$ has $512 \times 8 \times 8 = 32\,768$ and $\mathcal{F}_5$ has $512 \times 16 \times 16 = 131\,072$.
    \item Input latent maps can be generated with different sizes, allowing to generate consistent images of arbitrary sizes.
\end{itemize}


One possibility to sample input feature tensors is to use $\mathcal{N}\funarg{0, I}$.
However, we suspect that full spatial independence in the distribution may result in different output image distribution,
for example seas, lakes or forests take often a significant part of the image, and independence in input feature maps prevents from learning common information across wider area.

\def\gaussiandist{blurred normal noise distribution}
\def\Gaussiandist{Blurred normal noise distribution}

Therefore, we propose a \textit{\gaussiandist}.
To sample a variable, we randomize a feature map from $\mathcal{N}\funarg{0, I}$, and
apply gaussian blur for all channels with linearly increasing blurring coefficients.
Algorithm~\ref{alg:gaussiandist-sampling} shows a procedure with details.

We visualize a sample element from that distribution in Figure~\ref{fig:gaussiandist-sample}.
We want to give both spatially dependent and independent information on various scale,
in a way that the closer to each other points are, the more common signal they have.
For example, the generator may use later channels to decide water and land placement,
middle channels for city roads, and early channels for stochastic variance.

\begin{algorithm}
    \caption{\Gaussiandist{} sampling}
    \label{alg:gaussiandist-sampling}
    \begin{algorithmic}
        \Require{ \\
            \;\; $C, H, W$ -- number of channels and spatial dimensions of a $\mathcal{F}_i$ feature map \\
            \;\; $k$ -- padding for gaussian blur \\
            \;\; $\sigma_{\mathrm{max}}$ -- max standard deviation for gaussian blur \\
        }
        \vspace{0.5em}
        \State $x \gets \text{sample from $\mathcal{N}\funarg{0, I}$ of size $C \times (H + 2k) \times (W + 2k)$}$
        \State $y \gets \text{an empty $C \times H \times W$ tensor}$
        \For{$i \text{ in } [0, \ldots, C - 1]$}
        \State $\sigma \gets \frac{i}{C - 1} \sigma_{\mathrm{max}}$
        \State $K \gets \text{$(2k+1) \times (2k+1)$ gaussian blur kernel with standard deviation $\sigma$}$
        \State $z \gets \text{convolve $x_i$ using kernel $K$}$
        \State $y_i \gets z \cdot \sqrt{\sum_{k \in K}k^2}$
        \Comment{\algparbox{0.475\textwidth}{This multiplication is needed to obtain a unit standard deviation in the expected case. See \ref{preserve-std-for-blur} for proof and details.}}
        \EndFor
        \State \Return $y$
    \end{algorithmic}
\end{algorithm}

\begin{figure}
    \setlength{\tabcolsep}{0.15em}
    \renewcommand{\arraystretch}{0}
    \centering
    
    \def\ssize{0.2}
    \def\shsize{0.41}
    \begin{tabular}{ccccc}
        \vspace{0.5em} \\
        & $n = 0$ & $n = 32$ & $n = 64$ & $n = 96$ \\
        \vspace{0.3em} \\
        
        \multirow{2}{*}{\includegraphics[height=\shsize\textwidth]{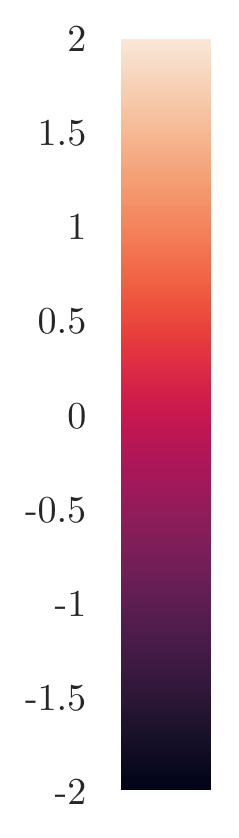}} &
        \includegraphics[align=t,width=\ssize\textwidth]{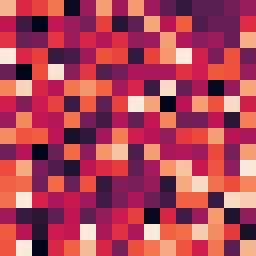} &
        \includegraphics[align=t,width=\ssize\textwidth]{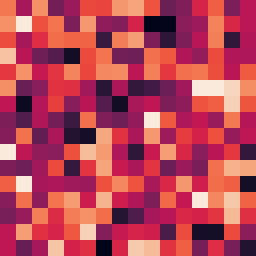} &
        \includegraphics[align=t,width=\ssize\textwidth]{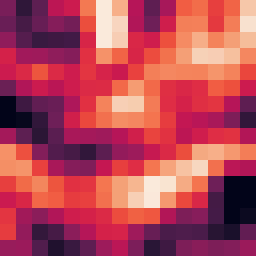} &
        \includegraphics[align=t,width=\ssize\textwidth]{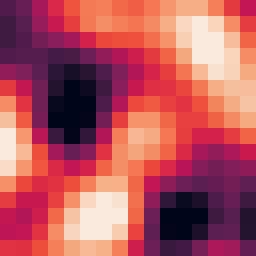} \\
        
        \vspace{0.5em} \\
        & $n = 128$ & $n = 192$ & $n = 256$ & $n = 511$ \\
        \vspace{0.3em} \\
        & 
        \includegraphics[align=t,width=\ssize\textwidth]{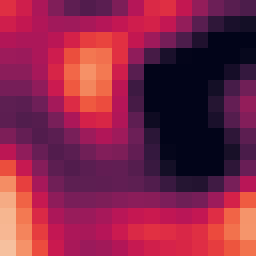} &
        \includegraphics[align=t,width=\ssize\textwidth]{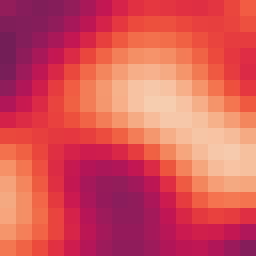} &
        \includegraphics[align=t,width=\ssize\textwidth]{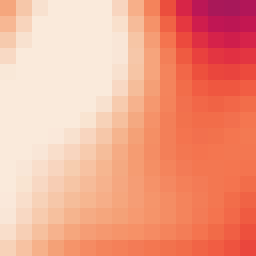} &
        \includegraphics[align=t,width=\ssize\textwidth]{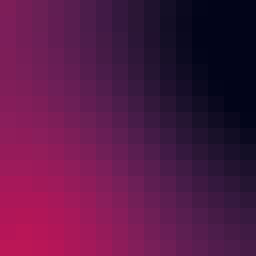} \\
    \end{tabular}
    
    \caption{
        Visualization of the $512 \times 16 \times 16$ ($\mathcal{F}_5$) tensor sampled from \gaussiandist.
        $n$ denotes the number of the visualized channel (zero indexed).
        Further channels are more smooth, giving a similar signal to distant points.
    }
    \label{fig:gaussiandist-sample}
\end{figure}

\section{Experiments}

In our experiments, we use SpaceNet \cite{spacenet} dataset in RGB format. We train five models: a baseline (using original StyleGAN2 space), and models with latent space built upon $\mathcal{F}_3$ and $\mathcal{F}_5$, using both $\mathcal{N}\funarg{0, I}$ and \gaussiandist{} sampling methods. For both spatial models, we always set the input RGB map to 0.

We use exactly the same hyperparameters and resolution (1024px $\times$ 1024px) as in FFHQ config F in StyleGAN2-ADA \cite{DBLP:journals/corr/abs-2006-06676},
with the exception that instead of training for 25\,000 iterations,
we first train for 20\,000 iterations using vanilla architecture, and then train further every of 5 models (including vanilla model) for additional 5\,000 iterations.
That way, we spent a total of 45\,000 iterations, instead of 125\,000 which would be the case when training all models from scratch, running out of our computing budget and taking significantly longer.

\begin{table}
    \centering
    \begin{tabular}{|l|c|}
    \hline
    \textbf{Model} & \textbf{FID} \\ \hline
    Vanilla StyleGAN2 & 14.3 \\
    $\mathcal{F}_5$, sampled from $\mathcal{N}\funarg{0, I}$ & 61.1 \\
    $\mathcal{F}_5$, sampled from \gaussiandist & 40.7 \\
    $\mathcal{F}_3$, sampled from $\mathcal{N}\funarg{0, I}$ & 12.7 \\
    $\mathcal{F}_3$, sampled from \gaussiandist & \textbf{10.1} \\ \hline
    \end{tabular}
    \caption{
        Fréchet Inception Distance (FID) (lower is better) for trained generator models.
        As in StyleGAN \cite{DBLP:journals/corr/abs-1812-04948}, we calculate
        the FIDs using 50\,000 random images from the training set and 50\,000 generated images.
        In particular, we use the same pretrained Inception-V3 model as in StyleGAN2-ADA \cite{DBLP:journals/corr/abs-2006-06676}.
        $\mathcal{F}_3$ sampled from \gaussiandist{} outperforms the original StyleGAN2 model by 29\% in terms of FID score.
        Using \gaussiandist{} instead of $\mathcal{N}\funarg{0, I}$ performs 20--30\% better.
    }
    \label{table:train-results}
\end{table}

\begin{figure}
    \setlength{\tabcolsep}{0.05em}
    \renewcommand{\arraystretch}{0}
    \centering
    
    \def\ssize{0.24}
    \begin{tabular}{cccc}
        \includegraphics[width=\ssize\textwidth]{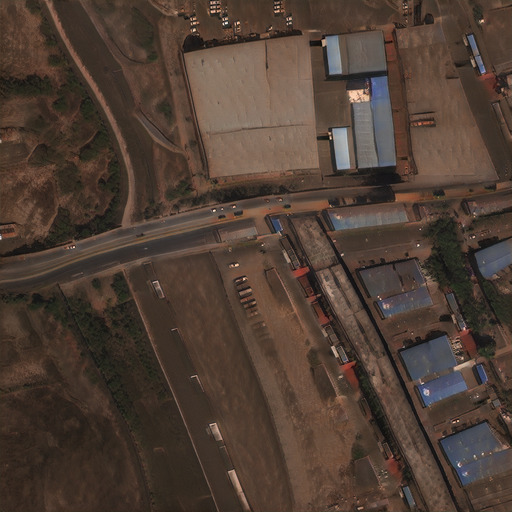} &
        \includegraphics[width=\ssize\textwidth]{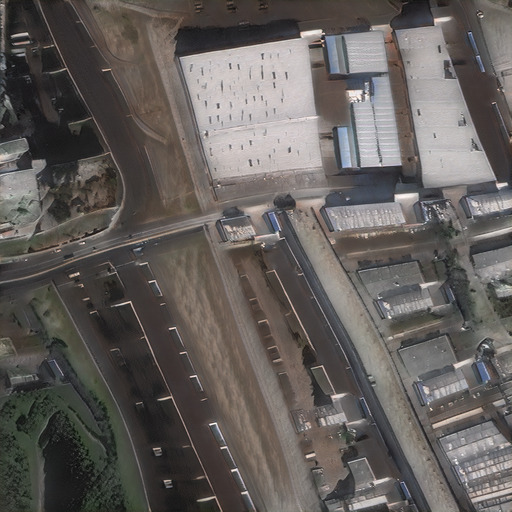} &
        \includegraphics[width=\ssize\textwidth]{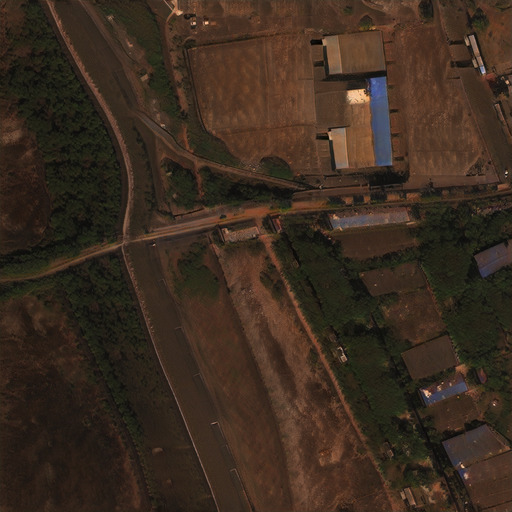} &
        \includegraphics[width=\ssize\textwidth]{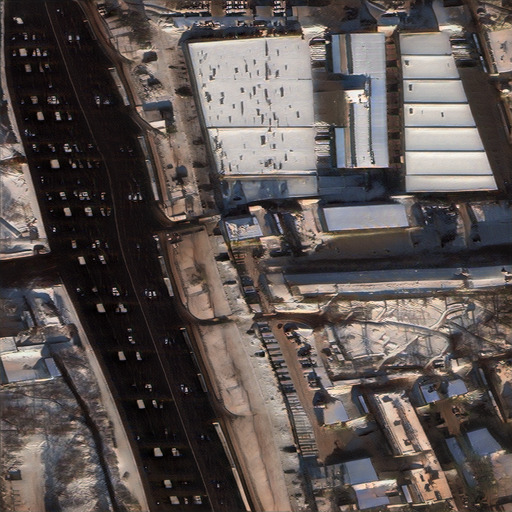} \\
    \end{tabular}
    \caption{
        We took trained vanilla StyleGAN2 generator and randomly generated 64 samples.
        From such small amount, we are able to select 4 images that are spatially very similar to each other (\eg buildings and roads are in the same location, and have the same shape).
        We suspect that building locations and shapes are encoded in the trained $\mathcal{F}_1$ tensor, as it's the only place in StyleGAN2 where spatial information can be encoded.
        Note that we don't see similar behavior for models trained in spatial spaces nor in the training dataset as ground truth images are obtained by random cropping very large images.
    }
    \label{fig:train-vanilla-background}
\end{figure}

SpaceNet consists of very large images, sometimes even one image per the entire area of interest, often having more than a few thousands megapixels.
Instead of preparing $1024 \times 1024$ patches beforehand, in our implementation we use images converted into Zarr \cite{github:zarr} data format, which allows us to crop patches during training from large images in a performant way.
We sample data proportionally to total image sizes, cropping patches uniformly, \ie every $1024 \times 1024$ patch is equally probable. We aim to avoid overfitting, especially to spatial elements.

We use the Fréchet Inception Distance (FID) \cite{DBLP:journals/corr/HeuselRUNKH17} score for evaluation.
We show results of our method in Table~\ref{table:train-results}. Samples of generated images are in Figure~\ref{fig:train-images}.
We obtain up to 29\% improvement in FID score using $\mathcal{F}_3$ sampled with \gaussiandist{} compared to the original StyleGAN2 model and training procedure. 

In Figures~\ref{fig:train-high-res-images}, \ref{fig:train-high-res-images2}, \ref{fig:train-high-res-images3}, and \ref{fig:train-high-res-images4},
we utilize spatiality of the latent space to sample a spatially large input maps
to generate high resolution images.

\begin{figure}
    \setlength{\tabcolsep}{0.15em}
    \renewcommand{\arraystretch}{0}
    \centering
    
    \def\ssize{0.28}
    \begin{tabular}{cccc}
        \rotatebox[]{90}{Vanilla}\; &
        \includegraphics[align=c,width=\ssize\textwidth]{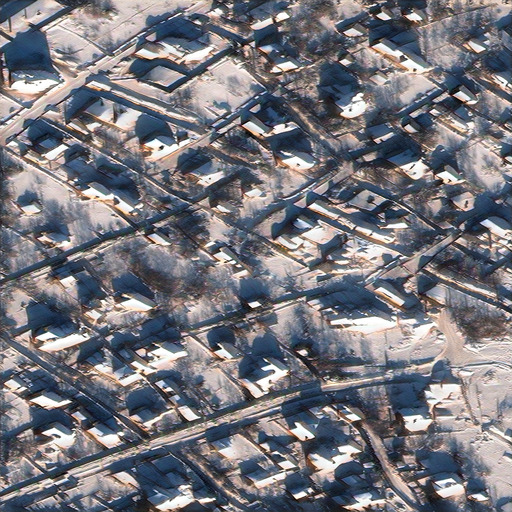} &
        \includegraphics[align=c,width=\ssize\textwidth]{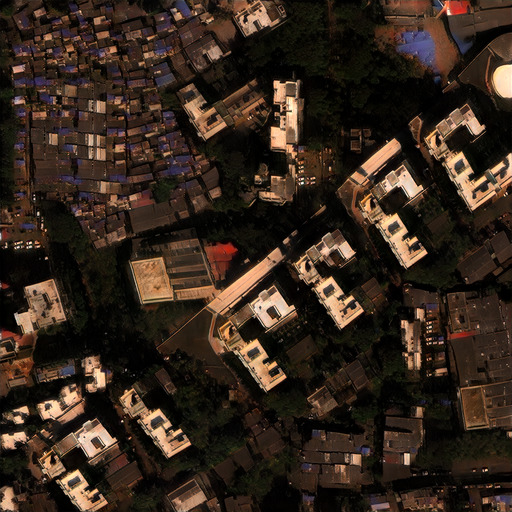} &
        \includegraphics[align=c,width=\ssize\textwidth]{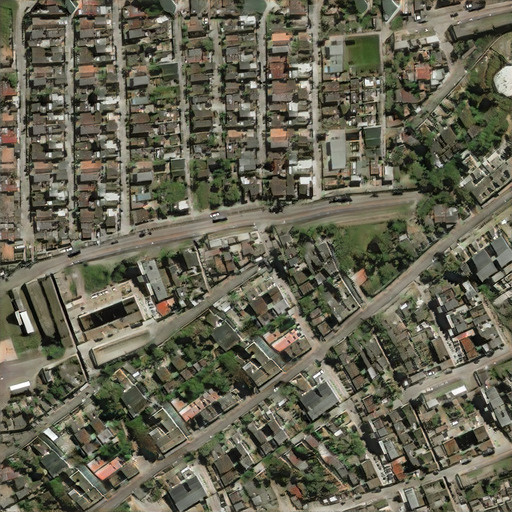} \\
        \vspace{0.15em} \\
        
        \rotatebox[]{90}{\parbox{\ssize\textwidth}{\centering $\mathcal{F}_3$ sampled from \gaussiandist}}\; &
        \includegraphics[align=c,width=\ssize\textwidth]{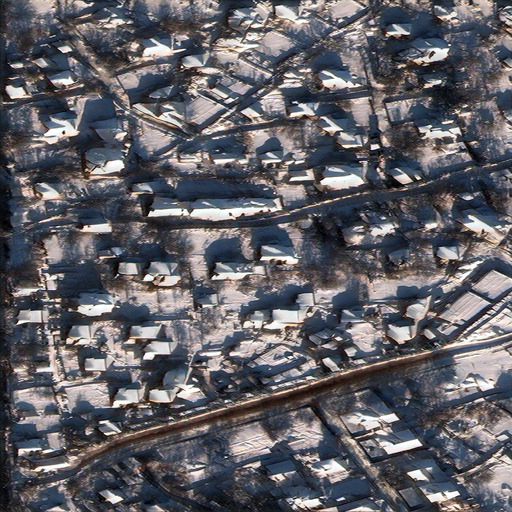} &
        \includegraphics[align=c,width=\ssize\textwidth]{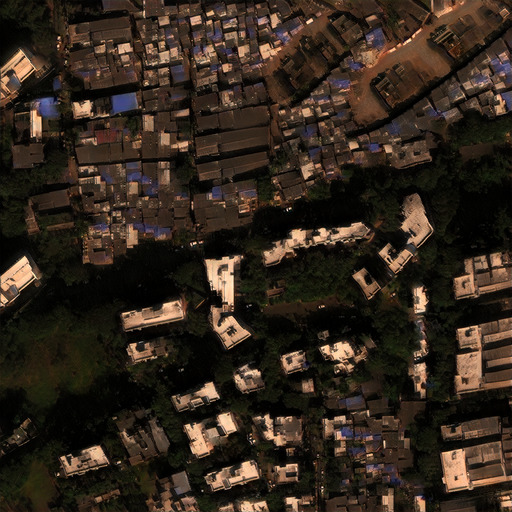} &
        \includegraphics[align=c,width=\ssize\textwidth]{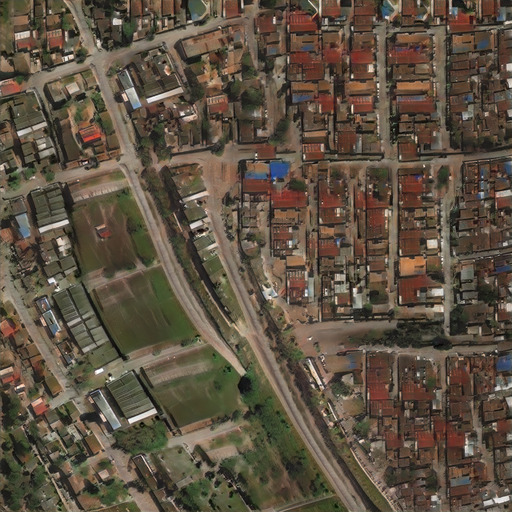} \\
        \vspace{0.15em} \\
        
        \rotatebox[]{90}{\parbox{\ssize\textwidth}{\centering $\mathcal{F}_5$ sampled from $\mathcal{N}\funarg{0, I}$}}\; &
        \includegraphics[align=c,width=\ssize\textwidth]{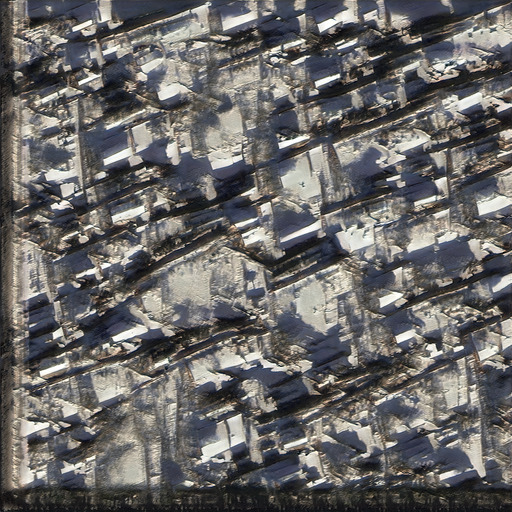} &
        \includegraphics[align=c,width=\ssize\textwidth]{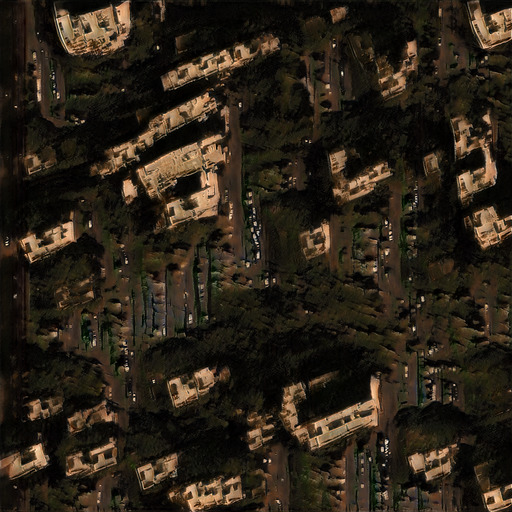} &
        \includegraphics[align=c,width=\ssize\textwidth]{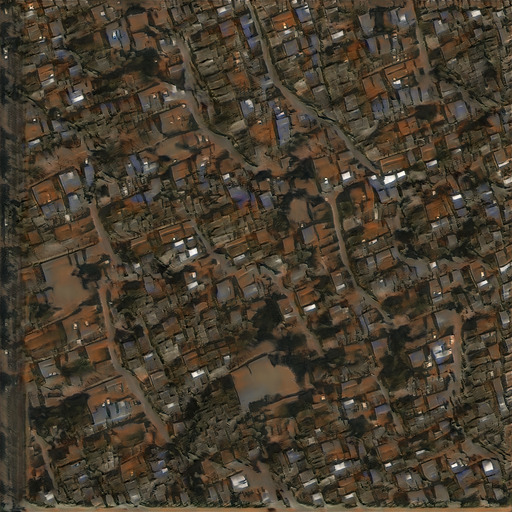} \\
        \vspace{0.15em} \\
        
        \rotatebox[]{90}{\parbox{\ssize\textwidth}{\centering $\mathcal{F}_5$ sampled from \gaussiandist}}\; &
        \includegraphics[align=c,width=\ssize\textwidth]{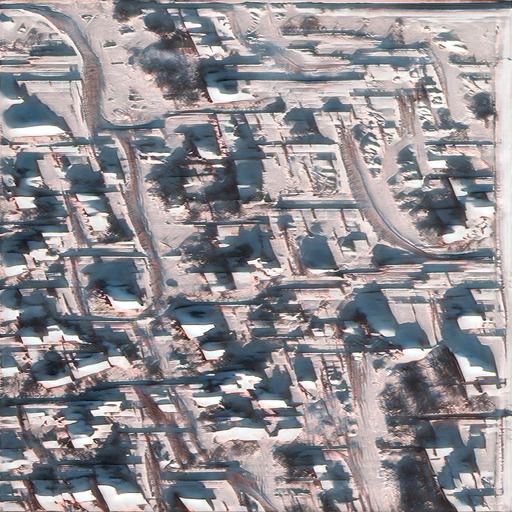} &
        \includegraphics[align=c,width=\ssize\textwidth]{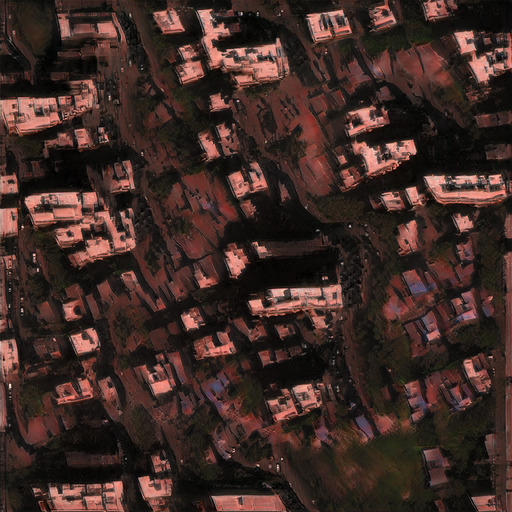} &
        \includegraphics[align=c,width=\ssize\textwidth]{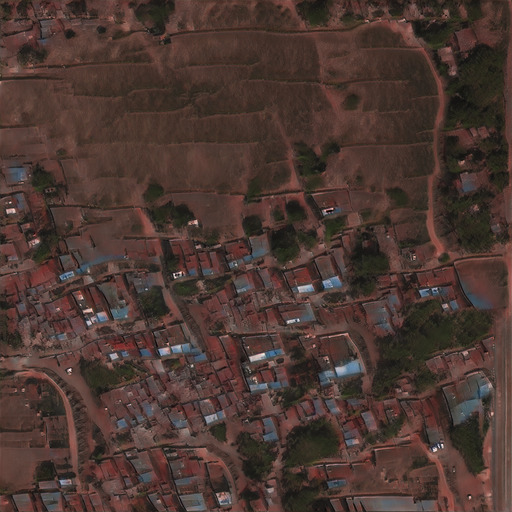} \\
    \end{tabular}
    \caption{
        Generated image examples using different models.
        Input tensors sampled from $\mathcal{N}\funarg{0, I}$ (3rd row) cannot capture spatial variance well.
    }
    \label{fig:train-images}
\end{figure}

\begin{figure}
    \setlength{\tabcolsep}{0.15em}
    \renewcommand{\arraystretch}{0}
    \centering
    
    \def\ssize{0.99}
    \includegraphics[width=\ssize\textwidth]{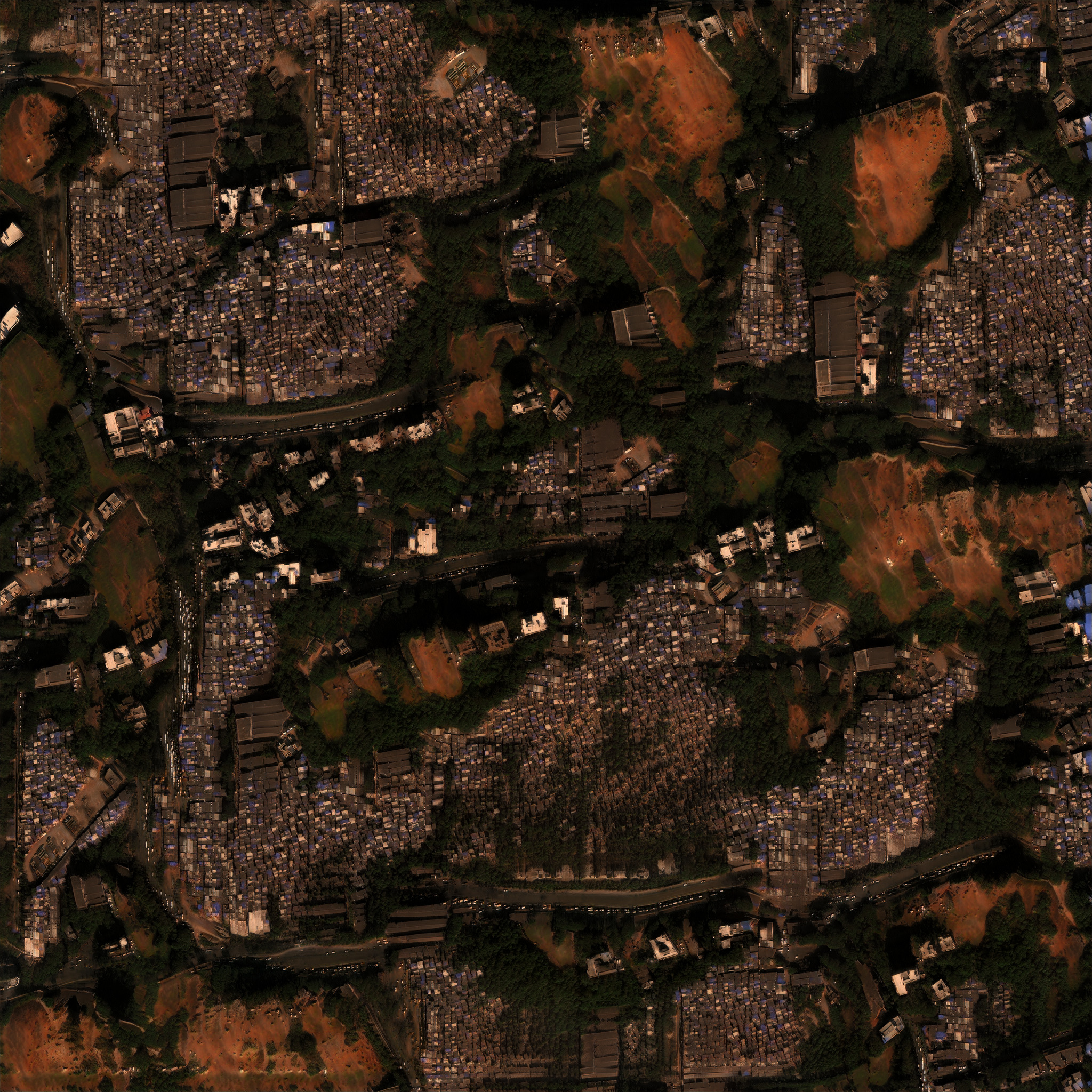}
    \caption{
        We take a model trained on $\mathcal{F}_3$ sampled from \gaussiandist.
        We generate 4 times larger (on edge) input map than during training, \ie $512 \times 32 \times 32$ instead of $512 \times 8 \times 8$.
        It results in $4096 \times 4096$ image depicting correspondingly wider area.
    }
    \label{fig:train-high-res-images}
\end{figure}

\begin{figure}
    \setlength{\tabcolsep}{0.15em}
    \renewcommand{\arraystretch}{0}
    \centering
    
    \def\ssize{0.99}
    \includegraphics[width=\ssize\textwidth]{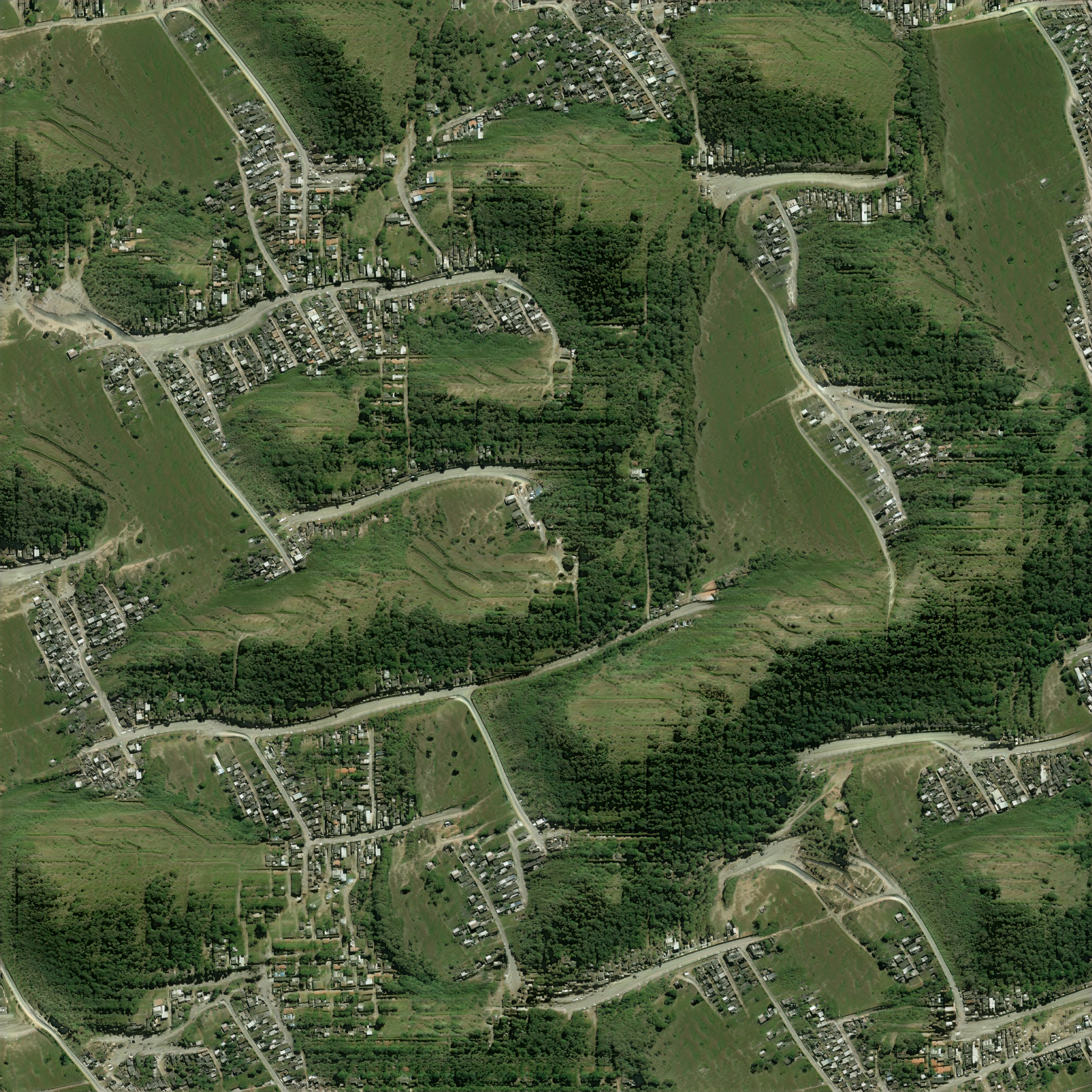} \\
    \caption{
        The same as Figure~\ref{fig:train-high-res-images}, but with the different latent representation.
    }
    \label{fig:train-high-res-images2}
\end{figure}

\begin{figure}
    \setlength{\tabcolsep}{0.15em}
    \renewcommand{\arraystretch}{0}
    \centering
    
    \def\ssize{0.99}
    \includegraphics[width=\ssize\textwidth]{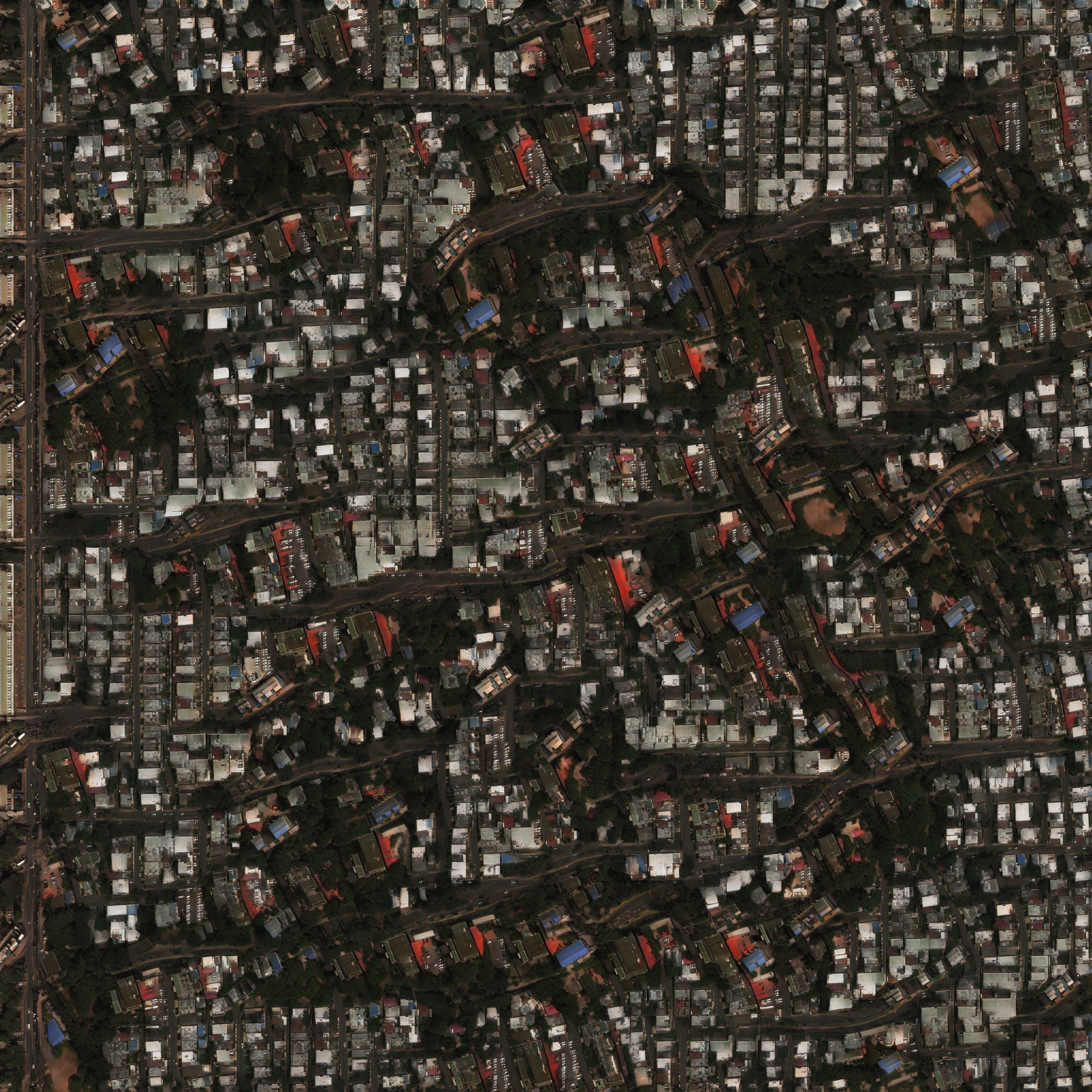} \\
    \caption{
        We follow the procedure from Figure~\ref{fig:train-high-res-images}, however we use the model trained on $\mathcal{F}_3$ sampled from $\mathcal{N}\funarg{0, I}$.
        The obtained image is much more spatially homogeneous.
    }
    \label{fig:train-high-res-images3}
\end{figure}

\begin{figure}
    \setlength{\tabcolsep}{0.15em}
    \renewcommand{\arraystretch}{0}
    \centering
    
    \def\ssize{0.99}
    \includegraphics[width=\ssize\textwidth]{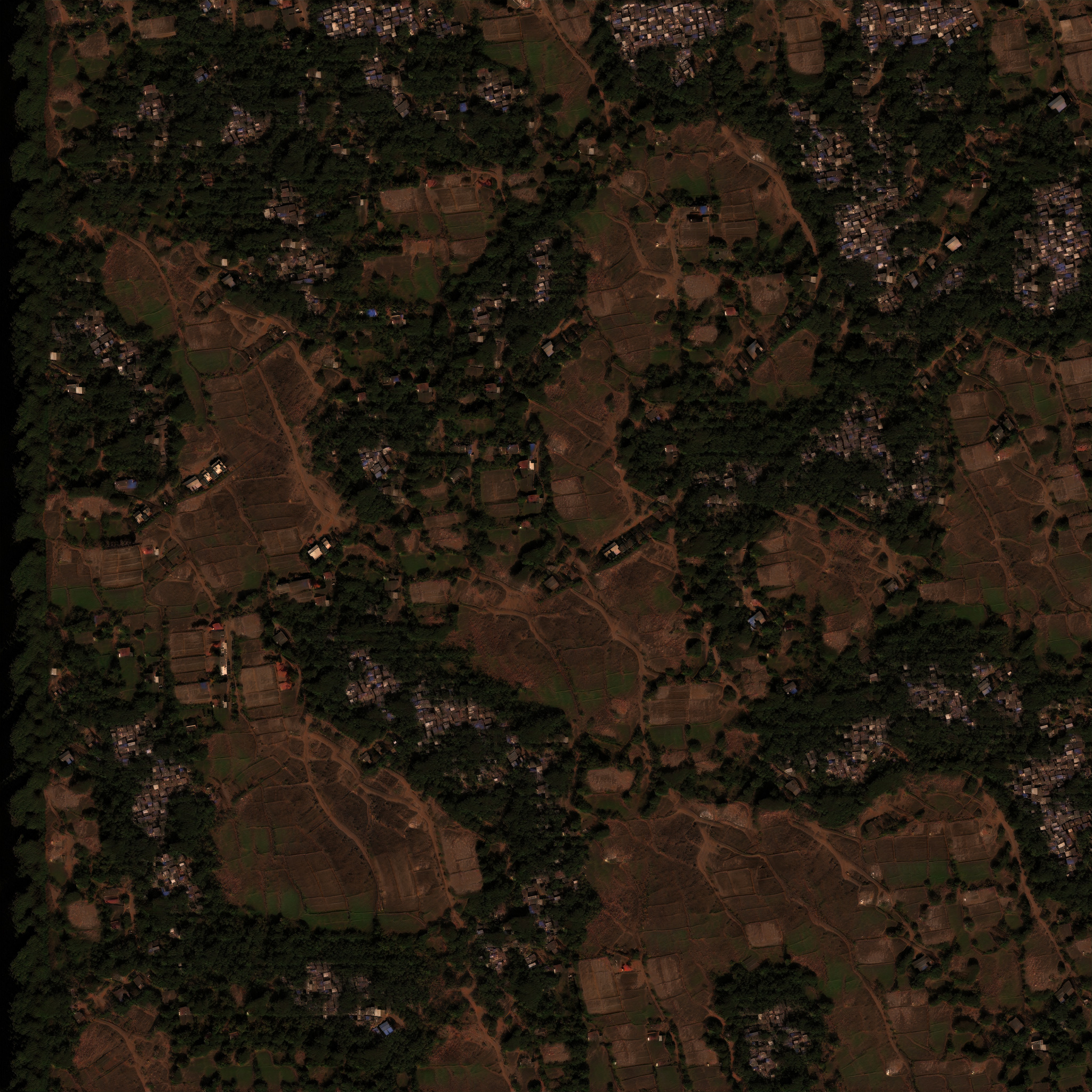} \\
    \caption{
        The same as Figure~\ref{fig:train-high-res-images3}, but with the different latent representation.
        Despite of sampling from $\mathcal{N}\funarg{0, I}$, a small fraction of images still manage to contain visually realistic spatial variance.
    }
    \label{fig:train-high-res-images4}
\end{figure}

\chapter{Summary}

We define the $\SWp$ latent space by extending style vectors into style maps and families of $\FWp{}$ and $\FSWp{}$ latent spaces, which add feature and RGB maps in StyleGAN2.
We call these spaces spatial, because they contain 2D structure for which a given spatial coordinate
affects the same coordinate and its surrounding in the generated image.
Additionally $\FWp{}$ and $\FSWp{}$ spaces are equivariant to translation, \ie translating them spatially will translate the generated image in the same way.

We validate them for quality, mixing, attribute edition and manipulation capabilities.
We propose an attribute model, which is capable of editing our spatial latent spaces
and is trained using an attribute direction vector that can be obtained using other methods.
Our attribute model learns a transformation for feature and RGB maps in representations
to apply a targeted attribute effect.

We propose a method for projecting images into $\FNWp{}$.
The method uses a projection encoder for initial latent representation projection and iterative optimization.
The projection encoder captures useful spatial semantic of the images \eg face and face parts positions and background segmentations.
Latent representation from the encoder is a great initial representation for optimization algorithm, and prevents from falling into meaningless local minima.

We show that our image projection method can obtain better results than projection into non-spatial latent spaces,
shows equivariance to translation and can deal with out-of-sample images like an image with multiple people,
or faces with non-standard amount of face parts.
Our attribute model can be used for such representations and obtain realistic effect.

Further, we identify a potential general issue that projected latent representations differ greatly from generated representations in $\mathcal{W}$ feature distribution.
We argue that such projected representations may perform worse for methods that use generated representations for training.
For example methods for obtaining attribute vectors can be trained on generated representations,
therefore they may consider other representations (\eg projected representations) out-of-sample.

We show a simple method for interpolating a sample distribution into the target distribution,
and observe that besides of lowering difference in feature histograms, it retains some visual similarity.
We extend that idea and define it as a regularization loss for latent representation optimization step for image projection algorithm.
We observe that for non-spatial space $\NWp$, we can almost entirely equalize features, without decreasing image quality.
With our regularization, attribute direction vector editions get more sensitive and presumably more on-target.
Similarly, we use our distribution regularization on style vectors, feature maps and RGB maps for $\FNWp{}$ latent representation optimization.
We notice that the regularization helps greatly and allows us to optimize representations much longer while keeping semantics,
which improves the image quality significantly and resulting representations are still suitable for image editing purposes.

Our projection method into $\FNWp{5}$ using distribution regularization can improve the image quality even as much as 30\%
compared to projection into commonly used $\NWp$ on selected samples, and still keep semantics for image manipulation.

We propose a training procedure in spatial space for StyleGAN2.
Such models can outperform vanilla StyleGAN2 on spatially homogeneous datasets like satellite imagery.
We propose a \gaussiandist{} for sampling input tensors for our spatial representations.
It blurs $\mathcal{N}\funarg{0, I}$ noise maps, using different coefficient for every channel, allowing to obtain greater
similarity between spatially close elements and less similarity between spatially far elements, possibly better capturing geometric shapes.
Our models improve the FID score on SpaceNet dataset by 29\%,
in addition to desired properties such as equivariance to translation, and ability to generate consistent images of arbitrary size.

\section{Future Works}
\subsection{Image Projection into $\FNSWp{}$}
One potentially promising research direction that we started, and still requires more investigation to achieve satisfying results,
is image projection into $\FNSWp{}$.
Having spatial style maps in the representation, drastically increases the number of degrees of freedom, making the projection much harder.
We tried to come up with good regularization methods for style maps.
For instance, we tried to use lower resolution style maps and upscale them to a proper size when passing to the style layer.
The other regularization was to apply an additional loss defined as MSE loss between a style map and its blurred version,
which should result in increased similarity in styles between close points.
We noticed that selecting one blur coefficient for all layers is not enough,
as later layers capture more fine style, and are more prone to overfitting.
Note that we defined blur coefficient as a percentage of spatial size, not as a the number of features,
because later style maps have more features.
We additionally tried to lower learning rate with style map depth,
however that didn't change results much as feature maps $\mathcal{F}$ are more easily able to learn information overlapping with earlier style maps that occurs after them. 
At the end, we weren't able to get a stable projection procedure, that achieves noticeably better results than our other methods.

We believe that there's a lot of potential in researching this subject further.
It may result in much better projection quality.
An exciting use case, could be obtaining style maps in an unsupervised way for analyzing image content.
For example, it could be used for satellite or street-view images for detection or segmenting architectural building styles.
It would also make style unification in the image possible.


\subsection{Better Input Map Distribution for Models}
We showed that GAN models trained on spatial latent spaces using \gaussiandist{} perform better than models trained using
$\mathcal{N}\funarg{0, I}$.
We believe that input distribution selection has still a lot of room for improvement.
A simple idea is to include more kinds of noise maps, for example different kind of blurring of $\mathcal{N}\funarg{0, I}$ maps, \eg box blur, radial blur, or even motion blur to add a random direction.
All blurring methods make maps smoother, therefore we suspect that adding some sharp shapes can be beneficial, for example scattering random polygons.
The model could possibly also benefit from more complex maps, \eg using generators of so-called grunge textures.
Presumably, models should be able to learn an appropriate combinations of channels, extracting most useful information to generate images from spatially matching distribution.

Another idea is to parameterize input tensor generators, and make these parameters learnable.
For example, for gaussian blur, it can be $\sigma$ coefficient.
Leaving coefficient adjustment to the model, may results
in better quality, and easier tuning to different datasets.
Note that such method requires differentiable map generators.
This is easy to define for \eg blurring methods, however it may be more difficult for more complex generators.

\appendix

\chapter{}

\section{Spatial Style Layer Implementation Details}
\label{style-layer-impl}

Authors of StyleGANs implement the style layer by convolution weights modulation using styles followed by demodulation.
However such approach cannot be done in spatial setup,
as we would have to modulate convolution weights at different spatial coordinate differently. It gets even more problematic because convolutions have usually $3 \times 3$ kernels, so for every spatial coordinate we would need to keep the copy all convolution weights.

Instead, we propose to modulate and demodulate activation tensors instead of convolution weights.
Modulation can be done easily by multiplying input activation tensors by the style vector, which is equivalent of multiplying/modulating convolution weights. Such multiplication can be done on every spatial element independently.

Computing demodulation weights for output activation tensors is more problematic as it requires to calculate them directly from modulated convolution weights.

First we need to expand a convolution weights into a 7D tensor [batch, out-channels, in-channels, conv-kernel-height, conv-kernel-width, spatial-height, spatial-width].
Now, we can multiple this tensor by similarly expanded style map, and use the same formula as for convolution weights to find demodulation coefficients for every sample, output channel, spatial height, and spatial width.

However this general idea cannot be directly used because a single forward pass requires a lot of memory,
and such implementation runs out of memory even on high-end GPUs, even with only one element in a mini-batch.

To tackle that problem we split large tensors into multiple smaller tensors on one dimension, process them independently, and then concatenate.
We split along spatial-width dimension in our implementation.
Note that it's possible to do this only on some dimensions to ensure that the concatenated output tensor won't change when using this method.
Such memory optimization makes it feasible to run, however it still doesn't solve memory problems for backpropagation, where
a computational graph still has to keep all tensors in the memory to calculate gradients.
We implement a custom method for a backward pass, that similarly splits tensors (on the input channel dimension).

Figure~\ref{code:demod} shows our implementation of a function returning demodulation coefficients for the output activation tensor.
The layer takes convolution weights and style maps.

\begin{figure}[h]
\begin{python}
class Demodulation(torch.autograd.Function):
    max_tensor_size = 4e8  # change to control memory consumption
    
    @staticmethod
    def forward(ctx, weight, style):
        # weight : B x Cout x Cinp x Hconv x Wconv
        # style  : B        x Cinp                 x H x W
        
        batch, in_channel, height, width = style.shape
        size = batch * in_channel * height * width *
               weight.shape[1] * weight.shape[3] * weight.shape[4]
        step = max(1, round(width * Demodulation.max_tensor_size / size))
        demod = torch.cat([
            torch.rsqrt(
                (weight.view(*weight.size(), 1, 1) * 
                 style.view(batch, 1, in_channel, 1, 1, height, width
                           )[..., i : i + step]
                ).pow(2).sum([2, 3, 4]) + 1e-8)
            for i in range(0, width, step)], -1)
        ctx.save_for_backward(demod, weight, style)
        return demod
    
    @staticmethod
    def backward(ctx, demod_grad):
        demod, weight, style = ctx.saved_tensors
        # demod  : B x Cout                        x H x W
        # weight : B x Cout x Cinp x Hconv x Wconv
        # style  : B        x Cinp                 x H x W
        
        batch, in_channel, height, width = style.shape
        size = batch * in_channel * height * width * 
               weight.shape[1] * weight.shape[3] * weight.shape[4]
        step = max(1,
                   round(in_channel * Demodulation.max_tensor_size / size))
        
        mdemod3 = -demod.pow(3)
        weight_grad = torch.cat([
            (weight.view(*weight.size(), 1, 1)[:, :, i : i + step] *
             style.view(batch, 1, in_channel, 1, 1, height, width
                       )[:, :, i : i + step].pow(2) *
             mdemod3.view(batch, demod.shape[1], 1, 1, 1, height, width)
            ).mean([5, 6])
            for i in range(0, in_channel, step)], 2)
        
        style_grad = torch.cat([
            (weight.view(*weight.size(), 1, 1)[:, :, i : i + step].pow(2) *
             style.view(batch, 1, in_channel, 1, 1, height, width
                       )[:, :, i : i + step] *
             mdemod3.view(batch, demod.shape[1], 1, 1, 1, height, width)
            ).sum([3, 4]).mean(1)
            for i in range(0, in_channel, step)], 1) * width * height
        
        return (
            torch.einsum('abfg,abcde->abcde', demod_grad, weight_grad),
            torch.einsum('abfg,acfg->acfg', demod_grad, style_grad),
        )
\end{python}
\caption{
    Spatial demodulation weight layer implementation in PyTorch \cite{NEURIPS2019_9015} using \pyth{torch.autograd.Function} API.
}
\label{code:demod}
\end{figure}

\clearpage

\section{Preserving Noise Map Distribution for Mixing}
\label{preserve-std}
\begin{theorem}
    If $x, y \sim \mathcal{N}\funarg{0, I}$, $x, y \in \mathbb{R}^{H \times W}$ are independent and $m \in [0, 1]^{H \times W}$ then
    \begin{equation}
        \frac{\mix\funarg{x, y, m}}{\sqrt{2m^2 - 2m + 1}} \sim \mathcal{N}\funarg{0, I}.
    \end{equation}
\end{theorem}
\begin{proof}
    Because all elements/scalars in tensors $x$ and $y$ are independent,
    it's sufficient to prove it, assuming that $x$, $y$ and $m$ are scalars, \ie $H = W = 1$.
    Let $x, y \sim \mathcal{N}\funarg{0, 1}$, and $m \in [0, 1]$.
    In such case $\mix\funarg{x, y, m} = x + m \parens{y - x} = (1-m)x + m y$. Further
    $\Var\funarg{(1-m)x + my} = \Var\funarg{(1-m)x} + \Var\funarg{my} = (1-m)^2\Var\funarg{x} + m^2\Var\funarg{y} = 2m^2 - 2m + 1$.
    Therefore $\Var\funarg{\frac{\mix\funarg{x, y, m}}{\sqrt{2m^2 - 2m + 1}}} = \frac{2m^2 - 2m + 1}{(\sqrt{2m^2 - 2m + 1})^2} = 1$.
    Similarly $\mathbb{E}\funarg{\frac{\mix\funarg{x, y, m}}{\sqrt{2m^2 - 2m + 1}}} = \frac{(1-m) \cdot 0 + m \cdot 0}{\sqrt{2m^2 - 2m + 1}} = 0$.
    Because linear combination of independent normal random variables is also a normal random variable $\frac{\mix\funarg{x, y, m}}{\sqrt{2m^2 - 2m + 1}} \sim \mathcal{N}\funarg{0, 1}$.
\end{proof}

\section{Preserving Noise Input Standard Deviation for Convolutions}
\label{preserve-std-for-blur}
\begin{theorem}
    Let $x : \mathbb{Z} \times \mathbb{Z} \rightarrow \mathbb{R}$ where $x_{a, b} \sim \mathcal{N}\funarg{0, 1}$ for any $a, b \in \mathbb{Z}$, $x_{a, b}$ is independent from $x_{c, d}$ for any $a,b,c,d \in \mathbb{Z}$, $(a,b) \neq (c,d)$, and $K \in \mathbb{R}^{H \times W}$, $H, W \in \mathbb{N}_+$ then
    \begin{equation}
        \Var\funarg{(x * K)_{i,j}} = \sum_{a \in [1, \ldots, H], b \in [1, \ldots, W]} K_{a,b}^2,
    \end{equation}
    for any $i, j \in \mathbb{Z}$, where $*$ is the convolution operator.
\end{theorem}
\begin{proof}
    \begin{align}
        \Var\funarg{(x * K)_{i,j}} &=
        \Var\funarg{\sum_{a \in [1, \ldots, H], b \in [1, \ldots, W]} x_{i+a,j+b} \cdot K_{a,b}} = \notag\\
        &= \sum_{a \in [1, \ldots, H], b \in [1, \ldots, W]} K_{a,b}^2 \cdot \Var\funarg{x_{i+a,j+b}} =
        \sum_{a \in [1, \ldots, H], b \in [1, \ldots, W]} K_{a,b}^2.
    \end{align}
\end{proof}
Similarly to \ref{preserve-std}, we could prove that
$\Var\funarg{\frac{(x * K)_{i,j}}{\sqrt{\sum_{k \in K} k^2}}} \sim \mathcal{N}\funarg{0, 1}$,
for all $(i,j)$. However $(x * K)_{a,b}$ and $(x * K)_{c,d}$ are not independent from each other if $|a-c| < H$ or $|b-d| < W$.

%
%

\begin{sloppypar}
\addcontentsline{toc}{chapter}{Bibliography}
\printbibliography
\end{sloppypar}

\end{document}